\documentclass[11pt]{article}
\pdfoutput=1

\usepackage{booktabs} 
\usepackage{etoolbox}
\usepackage{natbib,xfrac}
\usepackage{algorithm}
\usepackage[noend]{algorithmic}
\usepackage{amsmath,amsthm,amsfonts,amssymb}
\usepackage{amsmath}
\usepackage{color}
\usepackage{mathrsfs}
\usepackage{verbatim}
\usepackage{enumitem}
\usepackage{bm}
\usepackage{hyperref}
\usepackage{xr}
\usepackage{tikz}
\usepackage{float}
\usepackage{listings}

\usepackage{fancyhdr, lastpage}
\usepackage[vmargin=1.00in,hmargin=1.00in,centering,letterpaper]{geometry}

\newtoggle{sigconf}
\togglefalse{sigconf}

\newtoggle{icml}
\togglefalse{icml}

\newcommand{\sphereone}{\calS^1}
\newcommand{\samplen}{S^n}
\newcommand{\wA}{w}
\newcommand{\Awa}{A_{\wA}}

\newcommand{\wst}{\theta}
\newcommand{\wls}{\widehat{\theta}_{\mathrm{LS}}}

\newcommand{\sufgap}{\mathbf{suf}}
\newcommand{\calgap}{\mathbf{cal}}
\newcommand{\sepgap}{\mathbf{sep}}
\newcommand{\pargap}{\mathbf{ind}}
\newcommand{\calW}{\mathcal{W}}

\newcommand{\Delweak}{\sufgap}
\newcommand{\Delstrong}{\calgap}
\newcommand{\Delsep}{\sepgap}

\newcommand{\Zbar}{\overline{Z}}
\newcommand{\qbar}{\overline{q}}

\newcommand{\clspace}{\calX}
\newcommand{\attspace}{\calA}

\newcommand{\Ftil}{\widetilde{\calF}}

\newcommand{\classx}{X}
\newcommand{\attx}{A}
\newcommand{\calL}{\mathcal{L}}

\newcommand{\mat}[1]{\ensuremath{\mathbf{#1}}}

\newcommand{\grad}{\nabla}

\newcommand{\argmin}{\mathop{\rm argmin}}

\newcommand{\Ind}[1]{\mathbf{1}\{#1\}}

\newcommand{\R}{\mathbb{R}}
\newcommand{\E}{\mathbb{E}}
\newcommand{\F}{\mathcal{F}}
\newcommand{\Var}{\mathrm{Var}}

\newcommand{\A}{\mathcal{A}}

\newcommand{\X}{\mat{X}}

\renewcommand{\L}{\mathcal{L}}

\newcommand{\calF}{\mathcal{F}}
\newcommand{\fhat}{\widehat{f}}
\newcommand{\calS}{\mathcal{S}}

\newcommand{\calX}{\mathcal{X}}

\newcommand{\calD}{\mathcal{D}}

\newcommand{\calA}{\mathcal{A}}
\newcommand{\fbayes}{f^B}
\newcommand{\func}{f^U}

\newcommand{\bayscore}{\text{calibrated Bayes score}}
\newcommand{\bayrisk}{\text{calibrated Bayes risk}}

\newtheorem{example}{Example}[section]

\newtheorem{theorem}{Theorem}[section]
\newtheorem{definition}{Definition}
\newtheorem{proposition}[theorem]{Proposition}
\newtheorem{corollary}[theorem]{Corollary}

\newtheorem{lemma}[theorem]{Lemma}

\newtheorem{assumption}{Assumption}

\newcommand{\unifsim}{\overset{\mathrm{unif}}{\sim}}
\newcommand{\sign}{\mathrm{sign}}
\newcommand{\wbar}{\overline{w}}

\newcommand{\thetahat}{\widehat{\theta}}

\newcommand{\thetaone}{{\theta_1}}
\newcommand{\thetatwo}{{\theta_2}}
\newcommand{\thetai}{{\theta_i}}

\newcommand{\KL}{\mathrm{KL}}
\newcommand{\Bern}{\mathrm{Bernoulli}}
\newcommand{\ihat}{\widehat{i}}
\newcommand{\Dwst}{\calD_{\theta}}
\newcommand{\fls}{\widehat{f}_{n}}
\newcommand{\ferm}{\widehat{f}_{n}}

\definecolor{DarkBlue}{rgb}{0.1,0.1,0.5}

\hypersetup{
	colorlinks=true,       
	linkcolor=DarkBlue,          
	citecolor=DarkBlue,        
	filecolor=DarkBlue,      
	urlcolor=DarkBlue,          
	pdftitle={},
	pdfauthor={},
}

\begin{document}


\title{The implicit fairness criterion of unconstrained learning}

\author{
 Lydia T.~Liu\thanks{Department of Electrical Engineering and Computer Sciences, University of California, Berkeley}~\thanks{Equal contribution}
	\and Max Simchowitz\footnotemark[1]~\footnotemark[2]
	\and Moritz Hardt\footnotemark[1]}

\date{}

\maketitle

\begin{abstract}

We clarify what fairness guarantees we can and cannot expect to
follow from unconstrained machine learning. Specifically, we
characterize when unconstrained learning on its own implies
\emph{group calibration}, that is, the outcome variable is
conditionally independent of group membership given the score. We show that
under reasonable conditions, the deviation from satisfying group calibration is
upper bounded by the excess risk of the learned score relative to the
Bayes optimal score function. A lower bound confirms the optimality of our upper
bound. Moreover, we prove that as the excess risk of the learned score
decreases, the more strongly it violates \emph{separation}
and \emph{independence}, two other standard fairness criteria.

Our results show that group calibration is the fairness criterion that
unconstrained learning implicitly favors. On the one hand, this means that 
calibration is often satisfied on its own without the need for active
intervention, albeit at the cost of violating other criteria that are at odds
with calibration. On the other hand, it suggests that we should be satisfied
with calibration as a fairness criterion only if we are at ease with the
use of unconstrained machine learning in a given application.

%
%
%

\end{abstract}


\section{Introduction}

Although many fairness-promoting interventions have been proposed in the machine learning literature, unconstrained learning remains the dominant paradigm among practitioners for learning risk scores from data. Given a prespecified class of models, unconstrained learning simply seeks to minimize the average prediction loss over a labeled dataset, without explicitly correcting for disparity with respect to sensitive attributes, such as race or gender. Many criticize the practice of unconstrained machine learning for propagating harmful biases~\citep{crawford2013hidden,barocas2016big,crawford17trouble}. Others see merit in unconstrained learning for reducing bias in consequential decisions \citep{Corbett-Davies:2017, corbett2017imperfect,Mullainathan2018}.

In this work, we show that defaulting to unconstrained learning does not neglect fairness considerations entirely. Instead, it prioritizes one notion of ``fairness" over others: unconstrained learning achieves 
%
\emph{calibration} with respect to one or more sensitive attributes, as well as a related criterion called \emph{sufficiency} \citep[e.g.,][]{barocas-hardt-narayanan},
at the cost of 
violating other widely used fairness criteria, \emph{separation} and \emph{independence} (see Section \ref{sec:related} for references therein).

A risk score is \emph{calibrated} for a group if the risk score obviates the need to solicit group membership for the purpose of predicting an outcome variable of interest. The concept of calibration has a venerable history in statistics and machine learning~\citep{Cox58,murphy77,dawid82,degroot83,platt99probabilisticoutputs, Zadrozny01,nicu05}. 
The appearance of calibration as a widely adopted and discussed ``fairness criterion'' largely resulted from a recent debate around fairness in recidivism prediction and pre-trial detention. After journalists at ProPublica pointed out that a popular recidivism risk score had a disparity in false positive rates between white defendants and black defendants~\citep{angwin2016machine}, the organization that produced these scores countered that this disparity was a consequence of the fact that their scores were calibrated by race~\citep{dieterich16compas}. Formal trade-offs dating back the 1970s confirm the observed tension between calibration and other classification criteria, including the aforementioned criterion of \emph{separation}, which is related to the disparity in false positive rates~\citep{darlington1971another,chould16fair,kleinberg16inherent,barocas-hardt-narayanan}.

Implicit in this debate is the view that calibration is a constraint that needs to be actively enforced as a means of promoting fairness. Consequently, recent literature has proposed new learning algorithms which ensure approximate calibration in different settings~\citep{hebert-johnson18multi,kearns2017preventing}.


The goal of this work is to understand when approximate calibration can in fact be achieved by unconstrained machine learning alone. We define several relaxations of the exact calibration criterion, and show that approximate group calibration is often a routine consequence of unconstrained learning. 
Such guarantees apply even when the sensitive attributes in question are not available to the learning algorithm.
On the other hand, we demonstrate that under similar conditions, unconstrained learning strongly violates the \emph{separation} and \emph{independence} criteria. We also prove novel lower bounds which demonstrate that in the worst case, no other algorithm can produce score functions that are substantially better-calibrated than unconstrained learning.  
%
Finally, we verify our theoretical findings with experiments on two well-known datasets, demonstrating the effectiveness of unconstrained learning in achieving approximate calibration with respect to multiple group attributes simultaneously.



\subsection{Our results\label{sec:results_summary}}

We begin with a simplified presentation of our results. As is common in
supervised learning, consider a pair of random variables $(X,Y)$ where $X$
models available features, and $Y$ is a binary target variable that we try to predict
from~$X.$ We choose a discrete random variable~$A$ in the same probability space
to model group membership. For example, $A$ could represent gender, or race. In particular, our results \emph{do not} require that $X$ perfectly encodes the attribute $A$.

A score function~$f$ maps the random variable $X$ to a real number. We say that
the score function $f$ is \emph{sufficient} with respect to attribute~$A$ if we have $\E[Y\mid
f(X)]=\E[Y\mid f(X), A]$ almost surely.\footnote{This notion has also been referred to as ``calibration'' in previous work \citep[e.g.,][]{chould16fair}. In this work we refer to it as ``sufficiency'', hence distinguishing it from $\E[Y\mid f(X), A] = f(X)$, which has also been called ``calibration'' in previous work \citep[e.g.,][]{weinberger2017calibration}. These two notions are not identical, but closely related; we present analagous theoretical results for both. 
	} In words, conditioning on~$A$ provides no additional
information about~$Y$ beyond what was revealed by~$f(X).$ 
This definition leads to a natural notion of the \emph{sufficiency gap}:
\begin{align}\label{eq:sufgap_def}
\sufgap_f(A) = \E[\left|\E[Y\mid f(X)]-\E[Y\mid f(X), A]\right|]\,, 
\end{align}
which measures the expected deviation from satisfying sufficiency over a random
draw of $(X, A).$

We say that
the score function $f$ is \emph{calibrated} with respect to group~$A$ if we have $\E[Y\mid f(X), A] = f(X)$. Note that calibration implies sufficiency. We define the \emph{calibration gap} \citep[see also][]{weinberger2017calibration} as 
\begin{align}\label{eq:calgap_def}
\calgap_f(A)=\E\left[\left| f(X) - \E[Y\mid f(X),A ]\right| \right].
\end{align}

Denote by $\calL(f)=\E[\ell(f, Y)]$ the \emph{population risk} (\emph{risk}, for short) of the score
function~$f$. Think of the loss function~$\ell$ as either the square loss or
the logistic loss, although our results apply more generally. Our first result relates the sufficiency and calibration gaps of a score to its risk.

\begin{theorem}[Informal]
For a broad class of loss functions that includes the square loss and logistic
loss, we have
\begin{align*}
\max\{\Delweak_f(A),\Delstrong_f(A) \} \le O\left(\sqrt{\calL(f)-\calL^\ast}\right)\,.
\end{align*}
Here, $\calL^\ast$ is the \emph{\bayrisk}, i.e., the risk of the
score function $\fbayes(x,a)=\E[Y\mid X=x,A=a].$
\end{theorem}

The theorem shows that if we manage to find a score function with small
excess risk over the \bayrisk, then the score function will also be reasonably sufficient and well-calibrated with respect to the group attribute~$A$. We also provide analogous results for the calibration error restricted to a particular group $A = a$.

In particular, the above theorem suggests that computing the unconstrained \emph{empirical risk minimizer}~\citep{vapnik1992principles}, or \emph{ERM}, is a natural strategy for achieving group calibration and sufficiency. For a given loss $\ell: [0,1] \times \{0,1\} \to \R$, finite set of examples $\samplen:=\{(X_i,Y_i)\}_{i \in [n]}$, and class of possible scores $\calF$, the ERM is the score function
\begin{equation}
\ferm \in \arg\min_{f\in \calF} \frac{1}{n}\sum_{i=1}^n \ell(f(X_i),Y_i)~.
\end{equation}
It is well known that, under very general conditions, 
$\calL(\ferm) \overset{\mathrm{prob}}{\to} \min_{f \in \calF} \calL(f)$; that is, the risk of $\ferm$ converges in probability to the least expected loss of any score function $f \in \calF$.



In general, the ERM may not achieve small excess risk, $\calL(f)-\calL^\ast$. Indeed, we have defined the \bayscore~$\fbayes$ as one that
has access to both $X$ and~$A.$ 
In cases
where the available features~$X$ do not encode~$A,$ but~$A$ is relevant to the
prediction task, the excess risk may be large. In other cases, the excess risk
may be large simply because the function class over which we can feasibly
optimize provides only poor approximations to the \bayscore. In example~\ref{ex:decomposition}, we provide scenarios when the excess risk is indeed small.

The constant in front of the square root in our theorem depends on properties of
the loss function, and is typically small, e.g., bounded by~$4$ for 
both the squared loss and the logistic loss. The more significant question is if
the square root is necessary. We answer this question in the affirmative.

\begin{theorem}[Informal]\label{thm:lower_bound_informal}
There is a triple of random variables $(X,A,Y)$ such that the empirical risk
minimizer $\fls$ trained on $n$ samples drawn i.i.d. from $(X, Y)$ satisfies 
$\min\{\calgap_{\fls}(A),\sufgap_{\fls}(A)\}\ge \Omega(1/\sqrt{n})$ and $\calL(\fls)-\calL*\le O(1/n)$ with
probability~$\Omega(1).$
\end{theorem}

In other words, our upper bound sharply characterizes the worst-case
relationship between excess risk, sufficiency and calibration. Moreover, our
lower bound applies not only to the empirical risk minimizer $\fls$, but to any score learned from data which is a linear function of the features $X$. Although group calibration and sufficiency is a natural consequence of unconstrained learning, it is in general untrue that they imply a good predictor. For example, predicting the group average, $f=\E[Y\mid A]$ is a pathological score function that nevertheless satisfies calibration and sufficiency.

Although unconstrained learning leads to well-calibrated scores, it violates other notions of group fairness. We show that the ERM typically violates independence---the criterion that scores are independent of group attribute $A$---as long as the base rate $\Pr[Y=1]$ differs by group. 
Moreover, we show that the ERM violates \emph{separation}, which asks for scores $f(X)$ to be conditionally independent of the attribute $A$ given the target $Y$ \citep[see][Chapter 2]{barocas-hardt-narayanan}. In this work, we define the \emph{separation gap}:
\begin{align*} \Delsep_f(A) := \E_{Y,A}[|\E[f(X)\mid Y,A] - \E[f(X)\mid Y] |],
\end{align*} and show that any score with small excess risk must in general have a large separation gap. Similarly, we show that unconstrained learning violates $\pargap_f(A) := \E_{A}[|\E[f(X)\mid A] - \E[f(X)] |]$, a quantitative version of the \emph{independence} criterion \citep[see][Chapter 2]{barocas-hardt-narayanan}.

\begin{theorem}[Informal]\label{lower_bound_sep}
	For a broad class of loss functions that includes the square loss and logistic
	loss, we have
	\begin{align*}
	\Delsep_{f}(A)\geq C_{f_B} \cdot Q_A - O(\sqrt{\calL(f) - \calL^*}),
	\end{align*}
	where $C_{f_B}$ and $Q_A$ are problem-specific constants independent of $f$. $C_{f_B}$~represents the inherent noise level of the prediction task, and $Q_A$ is the variation in group base rates. Moreover, $\pargap_{f}(A)\geq Q_A - O(\sqrt{\calL(f) - \calL^*})$ for the same constant $Q_A$.
\end{theorem}
The lower bound for $\Delsep_f$ is explained in Section~\ref{sec:body_sep_lb}; the lower bound for $\pargap_f$ is deferred to Appendix~\ref{sec:sep_lb_sec}.

\paragraph{Experimental evaluation.} We explore the extent to which the result of empirical risk minimization satisfies sufficiency, calibration and separation, via comprehensive experiments on the UCI Adult dataset \citep{Dua2017uci} and pretrial defendants dataset from Broward County, Florida \citep{angwin2016machine,Dresseleaao5580}. For various choices of group attributes, including those defined using arbitrary combinations of features, we observe that the empirical risk minimizing score is fairly close to being calibrated and sufficient. Notably, this holds even when the score is not a function of the group attribute in question. 

%

\subsection{Related work}\label{sec:related}

Calibration was first introduced as a fairness criterion by the education testing literature in the 1960s. It was formalized by the \emph{Cleary criterion} \citep{cleary66,cleary68}, which compares the slope of regression lines between the test score and the outcome in different groups. More recently, machine learning and data mining communities have rediscovered calibration, and examined the inherent tradeoffs between calibration and other fairness constraints. \citet{chould16fair} and \citet{kleinberg16inherent} independently demonstrate that exact group calibration is incompatible with \emph{separation} (equal true positive and false positive rates), except under highly restrictive situations such as perfect prediction or equal group base rates. Such impossibility results have been further generalized by \citet{weinberger2017calibration}.

There are multiple post-processing procedures which achieve calibration, \citep[see e.g.][and references therein]{nicu05}. Notably, Platt scaling \citep{platt99probabilisticoutputs} learns calibrated probabilities for a given score function by logistic regression. Recently, \citet{hebert-johnson18multi} proposed a polynomial time agnostic learning algorithm that achieves both low prediction error, and \emph{multi-calibration}, or simultaneous calibration with respect to all, possibly overlapping, groups that can be described by a concept class of a given complexity.
Complementary to this finding, our work shows that low prediction error often implies calibration with no additional computational cost, under very general conditions. Unlike \citet{hebert-johnson18multi}, we do not aim to guarantee calibration with respect to arbitrarily complex group structure; instead we study when usual empirical risk minimization already achieves calibration with respect to a given group attribute~$A$.



A variety of other fairness criteria have been proposed to address concerns of fairness with respect to a sensitive attribute. 
These are typically group parity constraints on the score function, including, among others, \emph{demographic parity} (also known as \emph{independence} and \emph{statistical parity}), \emph{equalized odds} (also known as \emph{error-rate balance} and \emph{separation}), as well as \emph{calibration} and \emph{sufficiency} \citep[see e.g.][]{feldman2015,hardt2016,chould16fair, kleinberg16inherent,weinberger2017calibration,barocas-hardt-narayanan}.
Beyond parity constraints, recent works have also studied dynamic aspects of fairness, such as the impact of model predictions on future welfare \citep{liu18delayed} and demographics \citep{hashimoto2018repeated}.



\section{Formal setup and results}\label{sec:setup_results}

We consider the problem of finding a \emph{score function} $\fhat$ which encodes the probability of a binary outcome $Y \in \{0,1\}$, given access to features $X \in \calX$. We consider functions $f: \calX \to [0,1]$ which lie in a prespecified function class $\calF$. 
We assume that individuals' features and outcomes $(X,Y)$ are random variables whose law is governed by a probability measure $\calD$ over a space $\Omega$, and will view functions $f$ as maps $\Omega \to [0,1]$ via $f = f(X)$. We use $ \Pr_\calD[\cdot], \Pr[\cdot]$ to denote the probability of events under $\calD$, and $\E_\calD[\cdot], \E[\cdot]$ to denote expectation taken with respect to $\calD$.


We also consider a $\calD$-measurable protected attribute $A \in \calA$, with respect to which we would like to ensure \emph{sufficiency} or \emph{calibration}, as defined in Section~\ref{sec:results_summary} above. While assume that $f = f(X)$ for all $f \in \calF$, we compare the performance of $f$ to the benchmark that we call the \emph{\bayscore}\footnote{Note that this is \emph{not} the perfect predictor unless $Y$ is deterministic given $A$ and $X$.}
\begin{align}\label{eq:Bayes_score}
\fbayes(x,a) := \E\left[Y \mid \classx = x, \attx = a\right],
\end{align}
which is a function of both the feature $x$ and the attribute $a$. As a consequence, $\fbayes \notin \calF$, except possibly whenever $Y$ is conditionally independent of $A$ given $X$. Nevertheless, $\fbayes$ is well defined as a map $\Omega \to[0,1]$ and it always satisfies sufficiency and calibration:
\begin{proposition}\label{prop:bayes_calibrated} $\fbayes$ is sufficient and calibrated, that is $\E[Y\mid
f^B(X)]=\E[Y\mid f^B(X), A]$ and $f^B=\E[Y\mid f(X), A]$, almost surely. 
Moreover, if $\Phi: \calX \to \calX'$ is any map, then the classifier $f_{\Phi}(X):= \E[Y\mid \Phi(X),A]$ is sufficient and calibrated. 
\end{proposition}
Proposition~\ref{prop:bayes_calibrated} is a direct consequence of the tower property (proof in Appendix~\ref{proof:prop_bayes_calib}). 
In general, there are many challenges to learning perfectly calibrated scores. As mentioned above, $\fbayes$ depends on information about $A$ which is not necessarily accessible to scores $f \in \calF$. Moreover, even in the setting where $A = A(X)$, it may still be the case that $\calF$ is a restricted class of scores, and $\fbayes \notin \calF$. Lastly, if $\fhat$ is estimated from data, it may require infinitely many samples to achieve perfect calibration. To this end, we introduce the following approximate notion of sufficiency and calibration:


\begin{definition} Given a $\calD$-measurable attribute $A \in \calA$ and value $a \in \calA$, we define the sufficiency gap of $f$ with respect to $A$ for group $a$ as 
\begin{align}
\sufgap_f(a;A):= \E_{\calD}\left[\left| \E[Y\mid f(X)] - \E[Y\mid f(X),A ]\right|\mid A = a \right]~.\label{eq:delweak}
\end{align}
and the calibration gap for group $a$ as
\begin{align}
\calgap_f(a;A):=\E_{\calD}\left[\left| f - \E[Y\mid f(X),A ]\right|\mid A = a \right]~.\label{eq:delstrong}
\end{align}
We shall let $\sufgap_f(A) $ and $\calgap_f(A)$ be as defined above in~\eqref{eq:sufgap_def} and~\eqref{eq:calgap_def}, respectively.
\end{definition}

\subsection{Sufficiency and calibration\label{sec:upper_bound_calib}}
We now state our main results, which show that the sufficiency and calibration gaps of a function $f$ can be controlled by its loss, relative to the $\bayscore$ $\fbayes$. All proofs are deferred to the supplementary material. Throughout, we let $\calF$ denote a class of score functions $f:\clspace \to [0,1]$ . For a loss function $\ell:[0,1] \times \{0,1\} \to \R$ and any $\calD$-measurable $f:\Omega \to [0,1]$, recall the \emph{population risk} $\calL(f) := \E[\ell(f,Y)]$. 
Note that for $f \in \F$, $\calL(f) = \E[\ell(f(X),Y)]$, whereas for the \bayscore~$\fbayes$, we denote its population risk as $\calL^*:=\calL(\fbayes) = \E[\ell(\fbayes(X,A),Y)]$. We further assume that our losses satisfy the following regularity condition:
\begin{assumption}\label{asm} Given a probability measure $\calD$, we assume that $\ell(\cdot,\cdot)$ is (a) $\kappa$-strongly convex: $\ell(z,y) \ge \kappa (z - y)^2$, (b) there exists a differentiable map $g: \R \to \R$ such that $\ell(z,y) = g(z) - g(z) - g'(z)(z - y)$ (that is, $\ell$ is a \emph{Bregman Divergence}), and (c) the \bayscore~is a critical point of the population risk, that is 
\begin{align*}
\E\left[ \dfrac{\partial}{\partial z} \ell(z,Y)\big{|}_{z = f^B}\right]= 0~.
\end{align*}
\end{assumption} 

Assumption \ref{asm} is satisfied by common choices for the loss function, such as the \emph{square loss} $\ell(z,y) = (z-y)^2$ with $\kappa = 1$, and the logistic loss, as shown by the following lemma, proved in Appendix~\ref{proof:lem:logistic}.
\iftoggle{sigconf}
{
\begin{lemma}[Logistic Loss]\label{lem:logistic} The logistic loss, defined as $\ell(f, Y) = -( Y \log f + (1-Y) \log (1-f))$, satisfies Assumption~\ref{asm} with $\kappa = 2/\log 2$. 
\end{lemma}
}
{
\begin{lemma}[Logistic Loss]\label{lem:logistic} The logistic loss $\ell(f, Y) = -( Y \log f + (1-Y) \log (1-f))$ satisfies Assumption~\ref{asm} with $\kappa = 2/\log 2$. 
\end{lemma}
}


 We are now ready to state our main theorem (proved in Appendix \ref{sec:main_ub_proofs}), which provides a simple bound on the sufficiency and calibration gaps, $\sufgap_f$ and $\calgap_f$, in terms of the excess risk $\calL(f) - \calL^*$:
\begin{theorem}[Sufficiency and Calibration are Upper Bounded by Excess Risk]\label{thm:main_upper} Suppose the loss function $\ell(\cdot,\cdot)$ satisfies Assumption~\ref{asm} with parameter $\kappa > 0$. Then, for any score $f \in \F$ and any attribute $A$,
\begin{align}
	\max\{\Delstrong_f(A),\Delweak_{f}(A)\}  &\le 4\sqrt{\frac{\calL(f) - \calL^*}{\kappa}}. \label{eq:average_calib}
\end{align} 
Moreover, it holds that for $ a \in \calA$, 
\begin{align}
\max\{\Delstrong_f(a;A),\Delweak_{f}(a;A)\} \le 2\sqrt{\frac{\calL(f) - \calL^*}{\Pr[A=a] \cdot \kappa}}. \label{eq:worst_case_calib}
\end{align}
\end{theorem}
Theorem~\ref{thm:main_upper} applies to any $f\in \calF$, regardless of how $f$ is obtained. 
As a consequence of Theorem~\ref{thm:main_upper}, we immediately conclude the following corollary for the empirical risk minimizer:
\begin{corollary}[Calibration of the ERM]\label{cor:erm_calib} Let $\fhat$ be the output of any learning algorithm (e.g. ERM) trained on a sample $\samplen \sim \calD^n$, and let $\calL(f)$ be as in Theorem~\ref{thm:main_upper}. Then, if $\fhat$ satisfies the guarantee
\begin{align*}
\Pr_{\samplen \sim \calD^n}\left[\calL(\fhat) - \min_{f \in \calF} \calL(f) \ge \epsilon\right] \le \delta,
\end{align*}
and if $\ell$ satisfies Assumption~\ref{asm} with parameter $\kappa > 0$, then with probability at least $1-\delta$ over $\samplen \sim \calD^n$, it holds that
\begin{align*}
 \max\{\Delstrong_f(A),\Delweak_{f}(A)\}  \le 4\sqrt{\frac{\epsilon +  \min_{f \in \calF}\calL(f) - \calL^*}{\kappa}}.
\end{align*}
\end{corollary}
The above corollary states that if there exists a score in the function class $\F$ whose population risk $\calL(f)$ is close to that of the calibrated Bayes optimal $\calL^*$, then empirical risk minimization succeeds in finding a well-calibrated score.

In order to apply Corollary \ref{cor:erm_calib}, one must know when the gap between the best-in-class risk and calibrated Bayes risk, $\min_{f\in\F}\calL(f) - \calL^*$, is small. In the full information setting where $\attx = \attx(\classx)$ (that is, the group attribute is available to the score function),  $\min_{f\in\F}\calL(f) -\calL^*$ corresponds to the approximation error for the class $\F$~\citep{bartlett2006convexity}. When $\classx$ may not contain all the information about $A$, $\min_{f\in\F}\calL(f) -\calL^*$ depends not only on the class $\F$ but also on how well $\attx$ can be encoded by $\classx$ given the class $\F$, and possibly additional regularity conditions. We now present a guiding example under which one can meaningfully bound the excess risk in the incomplete information setting. In Appendix~\ref{app:proof_examples}, we provide two further examples to guide the readers' intuition. For our present example, we introduce as a benchmark the \emph{uncalibrated Bayes optimal score} 
\begin{align*}
\func(x):= \E[Y | X = x],
\end{align*}
which minimizes empirical risk over all $X$ measurable functions, and  is necessarily in $\F$. Our first example gives a decomposition of $\calL(f) - \calL^*$ when $\ell$ is the square loss.
\begin{example}\label{ex:decomposition} Let $\ell(z,y) := (z-y)^2$ denote the squared loss. Then,

	\begin{align}
\calL(\fhat) - \calL^* &= \underset{(i)}{\left(\calL(\fhat) - \inf_{f \in \calF}(f)\right)} + \underset{(ii)}{\left(\inf_{f \in \calF} \calL(f)- \calL(\func)\right)} + \underset{(iii)}{\E_{X}\left[\Var_{A}\left[\fbayes \mid X\right]\right]},\label{eq:decomp}
\end{align}
where $\Var_{A}[\fbayes \mid X] = \E[(\fbayes - \E_{A}[\fbayes \mid X] )^2 \mid X]$ denotes the conditional variance of $\fbayes$ given $X$.

\end{example}
The decomposition in Example~\ref{ex:decomposition} follows immediately from the fact that the excess risk of $\func$ over $\fbayes$, $\calL(\func) - \calL^*$, is precisely $\Var_{A}[\fbayes \mid X]$ when $\ell$ is the square loss. Examining~\eqref{eq:decomp}, (i) represents the excess risk of $\fhat$ over the best score in $\F$, which tends to zero if $\fhat$ is the ERM. Term (ii) captures the richness of the function class, for as $\F$ contains a close approximation to $\func$. If $\fhat$ is obtained by a consistent non-parametric learning procedure, and $\func$ has small complexity, then both (i) and (ii) tend to zero in the limit of infinite samples. Lastly, (iii) captures the additional information about $A$ contained in $X$. Note that in the full information zero, this term is zero.




\subsection{Lower bounds for separation\label{sec:body_sep_lb}}
 

In this section, we show that empirical risk minimization robustly violates the \emph{separation} criterion that scores are conditionally independent of the group $A$ given the outcome $Y$. For a classifier that exactly satisfies separation, we have $\E[f(X)\mid Y,A] = \E[f(X)\mid Y] $ for any group $A$ and outcome $Y$. We define the \emph{separation gap} as the average margin by which this equality is violated:
 \begin{definition}[Separation gap]\label{def:seq_err} The \emph{separation gap} is 
\begin{align*}
\Delsep_f(A) := \E_{Y,A}[|\E[f(X)\mid Y,A] - \E[f(X)\mid Y] |].
\end{align*}
 \end{definition}
Our first result states that the calibrated Bayes score $f^B$, has a non-trivial separation gap.  The following lower bound is proved in Appendix~\ref{sec:sep_lb_sec}:
\begin{proposition}[Lower bound on separation gap]\label{proposition:sep_lb} Denote $\qbar := \Pr[Y = 1]$, and  $q_A := \Pr[Y = 1 | A]$ for a group attribute $A$. Let $\Var(\cdot)$ denote variance, and $\Var(\cdot\mid X)$ denote conditional variance given a random variable $X$. Then, $\Delsep_{\fbayes}(A) \ge C_{\fbayes} \cdot Q_A$, where
	\begin{align*}
	Q_A := \E_A |\qbar - q_A| \quad \text{and} \quad C_{\fbayes} := 
 	\frac{\E_{\calD}\Var[Y\mid X,A]}{\Var[Y]}.
	\end{align*}
 	\end{proposition}

 	Intuitively, the above bound says that the separation gap of the calibrated Bayes score is lower bounded by the product of two quantities: $Q_A = \E_A |q_A - \qbar|$ corresponds to the $L_1$-variation in base-rates among groups, and $C_{\fbayes}$ corresponds to the intrinsic noise level of the prediction problem. For example, consider the case where perfect prediction is possible (that is, $Y$ is deterministic given $X, A $). 
 	 Then, the lower bound is vacuous because $C_\fbayes = 0$, and indeed $\fbayes$ has zero separation gap. 

 	Proposition~\ref{proposition:sep_lb} readily implies that any score $f$ which has small risk with respect to $f^B$ also necessarily violates the separation criterion:
	\begin{corollary}[Separation of the ERM]\label{cor:erm_sep} Let $\calL$ be the risk associated with a loss function $\ell(\cdot,\cdot)$ satisfying Assumption~\ref{asm} with parameter $\kappa > 0$. Then, for any score $\hat{f} \in \F$, possibly the ERM, and any attribute $A$,
	\begin{align*}
	\Delsep_{\hat{f}}\geq C_{\fbayes} \cdot \E_A |q_A - \qbar| -2\sqrt{\frac{\calL(\hat{f}) - \calL^*}{\kappa}}.
	\end{align*}
	\end{corollary}

 	In prior work, \citet{kleinberg16inherent}'s impossibility result (Theorem 1.1, 1.2), as well as subsequent generalizations in \citet{weinberger2017calibration}, states that a score that satisfies both calibration and separation must be either a perfect predictor or the problem must have equal base rates across groups, that is, $\qbar = q_A$. 
In contrast, Proposition \ref{proposition:sep_lb} provides a quantitative lower bound on the separation gap of a calibrated score, for arbitrary configurations of base rates and closeness to perfect prediction. This is crucial for approximating the separation gap of the ERM in Corollary \ref{cor:erm_sep}.


\subsection{Lower bounds for sufficiency and calibration\label{sec:lower_bound}}

We now present two lower bounds which demonstrate that the behavior depicted in Theorem~\ref{thm:main_upper} is sharp in the worse case.  In Appendix~\ref{app:lb_proof}, we construct a family of distributions $\{\calD_{\theta}\}_{\theta \in \Theta}$ over pairs $(X,Y) \in \calX \times \{0,1\}$, and a family of attributes $\{A_w\}_{w \in \calW}$ which are measurable functions of~$X$. We choose the distribution parameter $\theta$ and attribute parameter $w$ to be drawn from specified priors $\pi_{\Theta}$ and $\pi_{\calW}$.
We also consider a class of score functions $\calF$ mapping $\calX \to [0,1]$, which contains the calibrated Bayes classifer for any $\theta \in \Theta$ and $w \in \calW$ (this is possible because the attributes are $X$-measurable). We choose $\calL$ to be the risk associated with the square loss, and consider classifiers trained on a sample $S^n = \{(X_i,Y_i)\}_{i=1}^n$ of $n$ i.i.d draws from $\calD_{\theta}$. In this setting, we have the following:
\begin{theorem}\label{thm:main_lb_body}
Let $\fhat \in \calF$ denote the output of any learning algorithm trained on a sample $\samplen \sim \calD^n$,  and let $\fls$ denote the empirical risk minimizer of $\calL$ trained on $S^n$. Then, with constant probability over $\theta \sim \pi_{\Theta}$, $w \sim \pi_{\calW}$, and $S^n \sim \calD_{\theta}$, $\min\{\calgap_{\fhat}(A_w),\sufgap_{\fhat}(A_w)\}\ge \Omega(1/\sqrt{ n})$ and $\calL(\fls)-\calL*\le O(1/n)$. 
\end{theorem}
In particular, taking $\fhat = \fls$, we see that the for any sample size $n$, we have that 
\begin{align*}
\min\{\calgap_{\fls}(A_w),\sufgap_{\fls}(A_w)\}/\sqrt{\calL(\fls)-\calL*} = \Omega(1).
\end{align*} with constant probability. In addition, Theorem~\ref{thm:main_lb_body} shows that in the worst case, the calibration and sufficiency gaps decay as $\Omega(1/\sqrt{n})$ with $n$ samples. 

We can further modify the construction to lower bound the per-group sufficency and calibration gaps in terms of $\Pr[A = a]$. Specifically, for each $p \in (0,1/4)$, we construct in Appendix~\ref{app:lower_bound_imbalance} a family of distributions $\{\calD_{\theta;p}\}_{\theta \in \Theta}$ and $X$-measurable attributes $\{A_w\}_{w \in \calW}$ such that, for all $(\theta,w)$, $\min_{a \in \calA} \Pr_{(X,Y) \sim \calD_{\theta;p}}$ $[A_w(X) = a] = p$, for all $\theta \in \Theta$ and $w \in \calW$.  The construction also entails modifying the class $\calF$; in this setting, our construction is as follows:
\begin{theorem}\label{thm:imbalance_lb_body} Fix $p \in (0,1/4)$. For any score $\fhat \in \calF$  trained on $S^n$, and the empirical risk mnimizer $\fls$,  it holds that $\min\{\calgap_{\fhat}(A_w),\sufgap_{\fhat}(A_w)\}\ge \Omega(1/\sqrt{ p n})$ and $\calL(\fls)-\calL*\le O(1/n)$, with constant probability over $\theta \sim \pi_{\Theta}$, $w \sim \pi_{\calW}$, and $S^n \sim \calD_{\theta;p}$.
\end{theorem}

\section{Experiments}


In this section, we present numerical experiments on two datasets to corroborate our theoretical findings
. These are the Adult dataset from the UCI Machine Learning Repository~\citep{Dua2017uci} and a dataset of pretrial defendants
from Broward County, Florida \citep{angwin2016machine,Dresseleaao5580} (henceforth referred to as the Broward dataset).

The Adult dataset contains 14 demographic features for 48842 individuals, for predicting whether one's annual income is greater than \$50,000. The Broward dataset contains 7 features of 7214 individuals arrested in Broward County, Florida between 2013 and 2014, with the goal of predicting recidivism within two years. It is derived by \citet{Dresseleaao5580} from the original dataset used by \citet{angwin2016machine} to evaluate a widely used criminal risk assessment tool. We present results for the Adult dataset in the current section, and those for the Broward dataset in Appendix \ref{sec:compas_expt}. 

Score functions are obtained by logistic regression on a training set that is 80\% of the original dataset, using all available features, unless otherwise stated. 

We first examine the sufficiency of the score with respect to two sensitive attributes, gender and race in Section~\ref{sec:expt_suff}. Then, in Section~\ref{sec:expt_multi} we show that the score obtained from empirical risk minimization is sufficient and calibrated with respect to multiple sensitive attributes 
simultaneously. Section~\ref{sec:expt_model} explores how sufficiency and separation are affected differently by the amount of training data, as well as the model class.


We use two descriptions of sufficiency. In Sections~\ref{sec:expt_suff} and~\ref{sec:expt_multi}, we present the so-called \emph{calibration plots} (e.g., Figure~\ref{fig:calib_w_group}), which plots observed positive outcome rates against score deciles for different groups.  The shaded regions indicate 95\% confidence intervals for the rate of positive outcomes under a binomial model. In Section~\ref{sec:expt_model}, we report empirical estimates of the sufficiency gap, $\sufgap_f(A)$, using a test set that is 20\% of the original dataset. More details on this estimator can be found in Appendix \ref{app:implementation_deets}. In general, models that are more sufficient and calibrated have smaller $\sufgap_f$ and their calibration plots show overlapping confidence intervals for different groups.




%
%

\subsection{Training with group information has modest effects on sufficiency}\label{sec:expt_suff}

\begin{figure}[thbp]
	\includegraphics[width=0.5\columnwidth, trim=0 10 0 10 , clip]{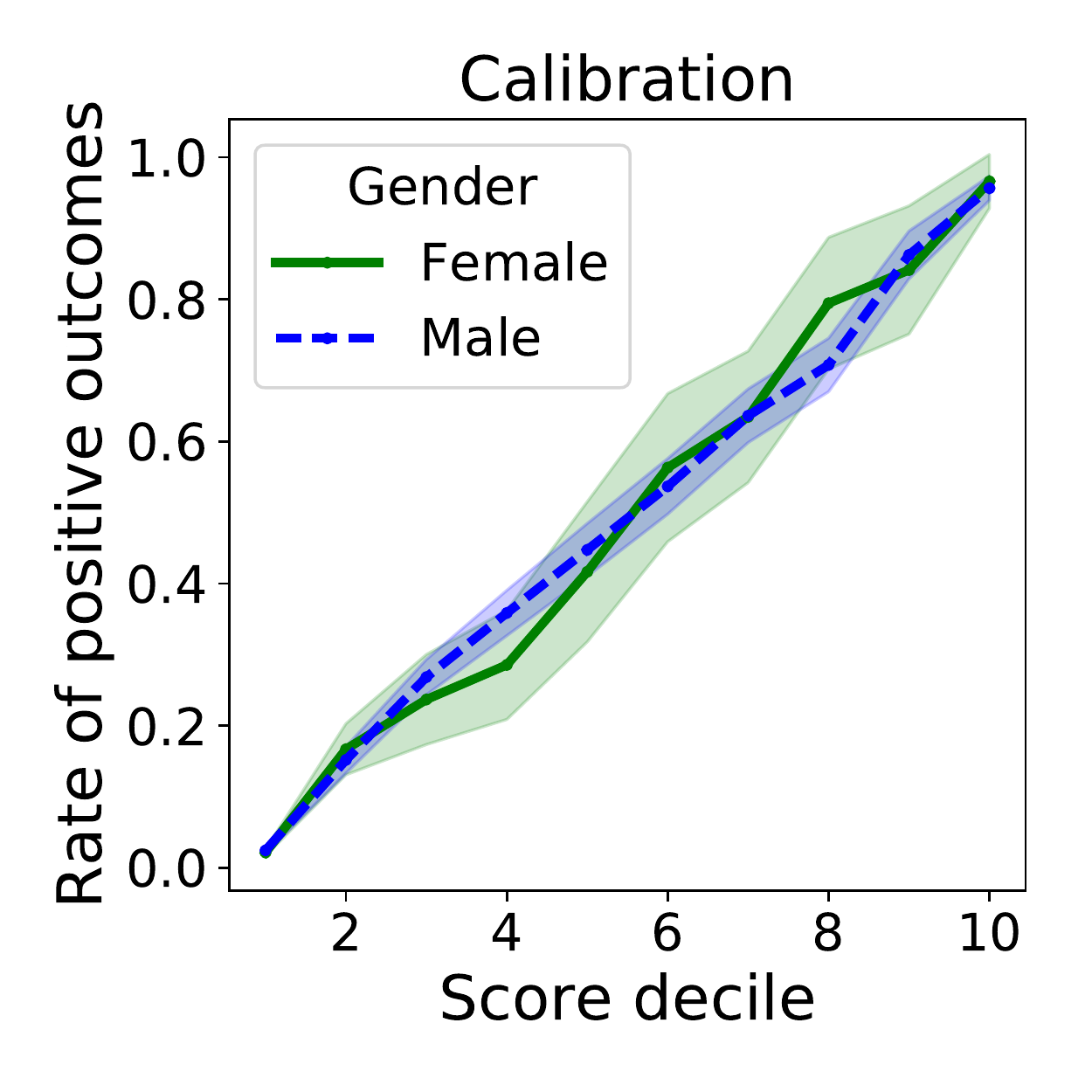}%
	\includegraphics[width=0.5\columnwidth, trim=0 10 0 10 , clip]{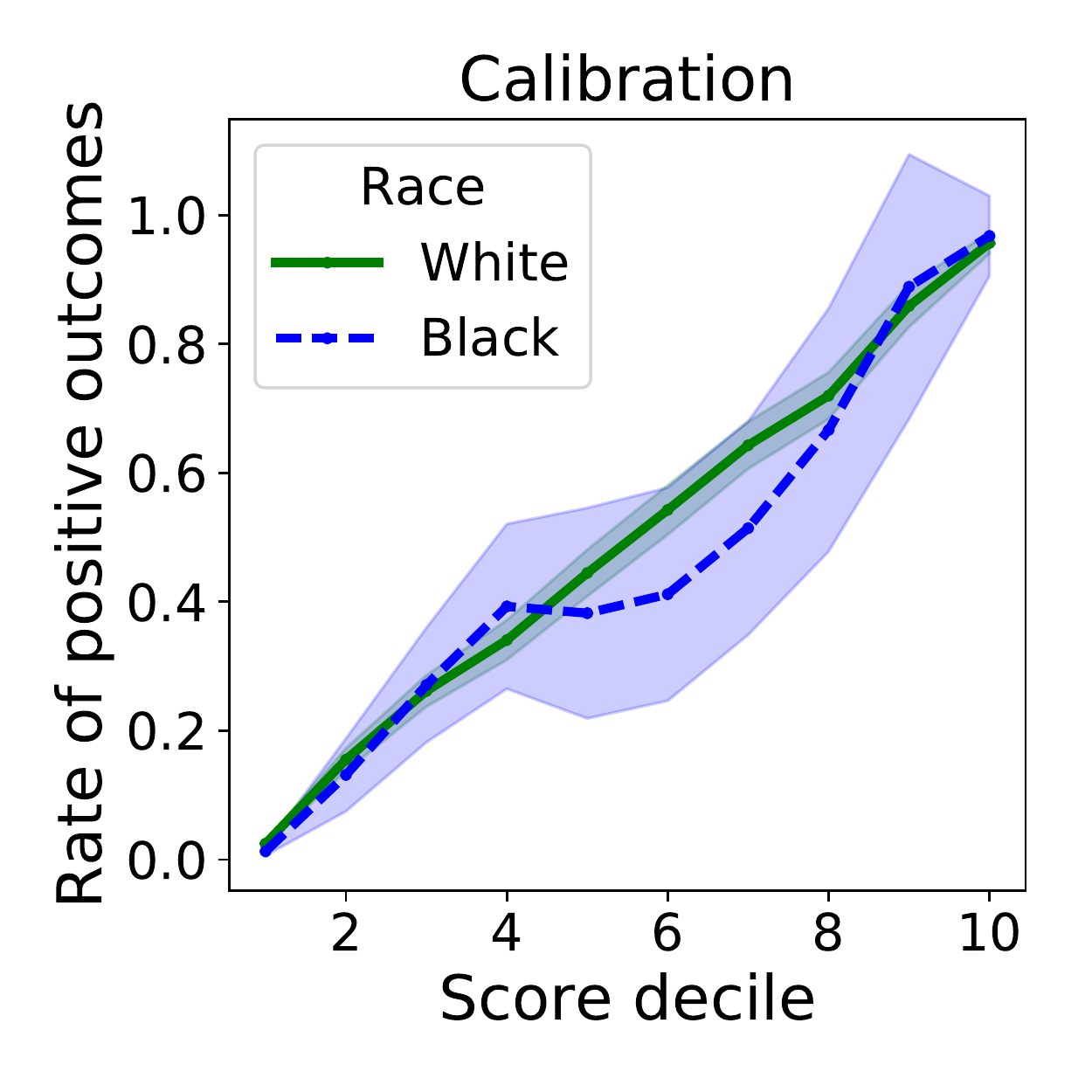}
	\caption{Calibration plot for score using group attribute}\label{fig:calib_w_group}
\end{figure}%
\begin{figure}[htbp]
	\centering
	\includegraphics[width=0.5\columnwidth, trim=0 10 0 10 , clip]{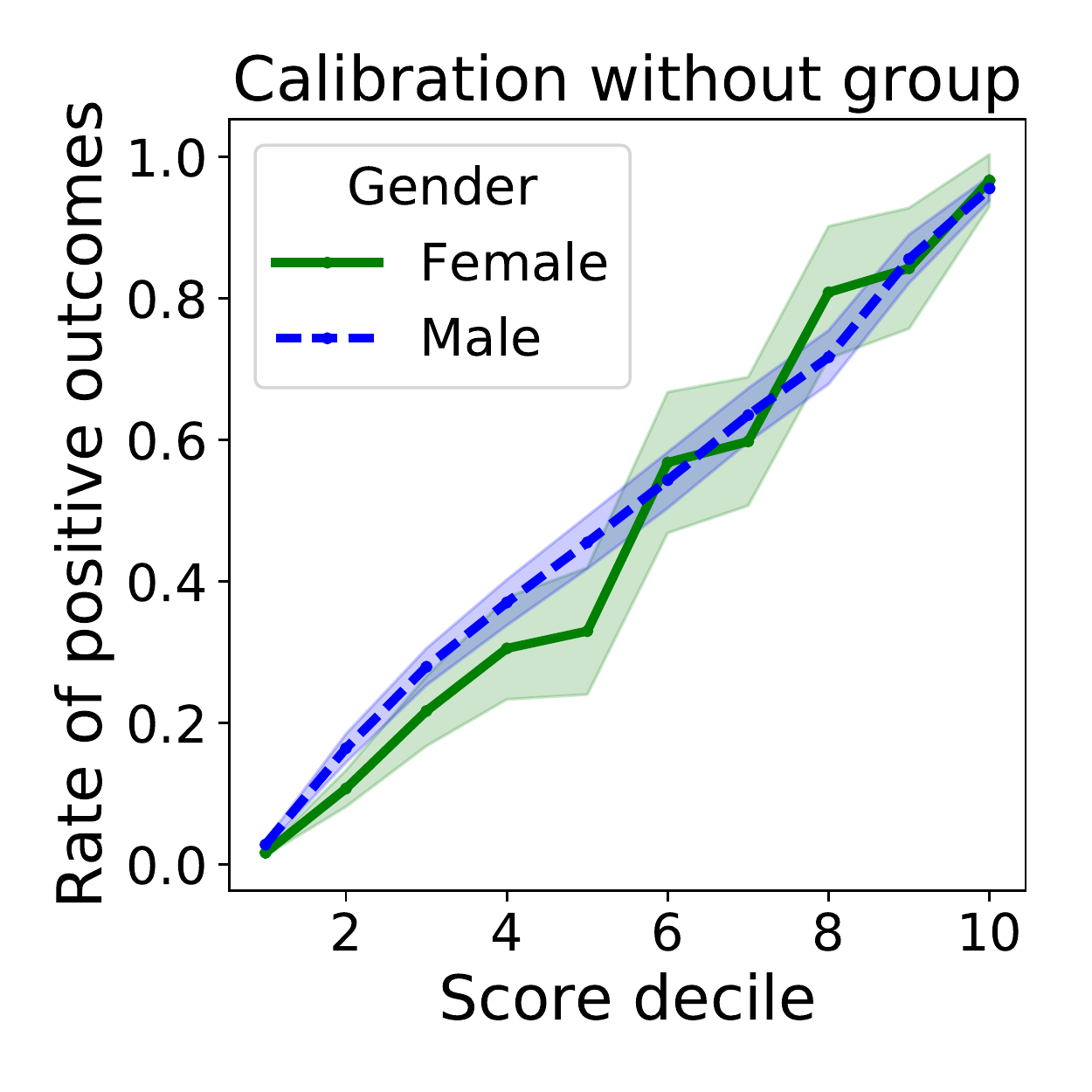}%
	\includegraphics[width=0.5\columnwidth, trim=0 10 0 10 , clip]{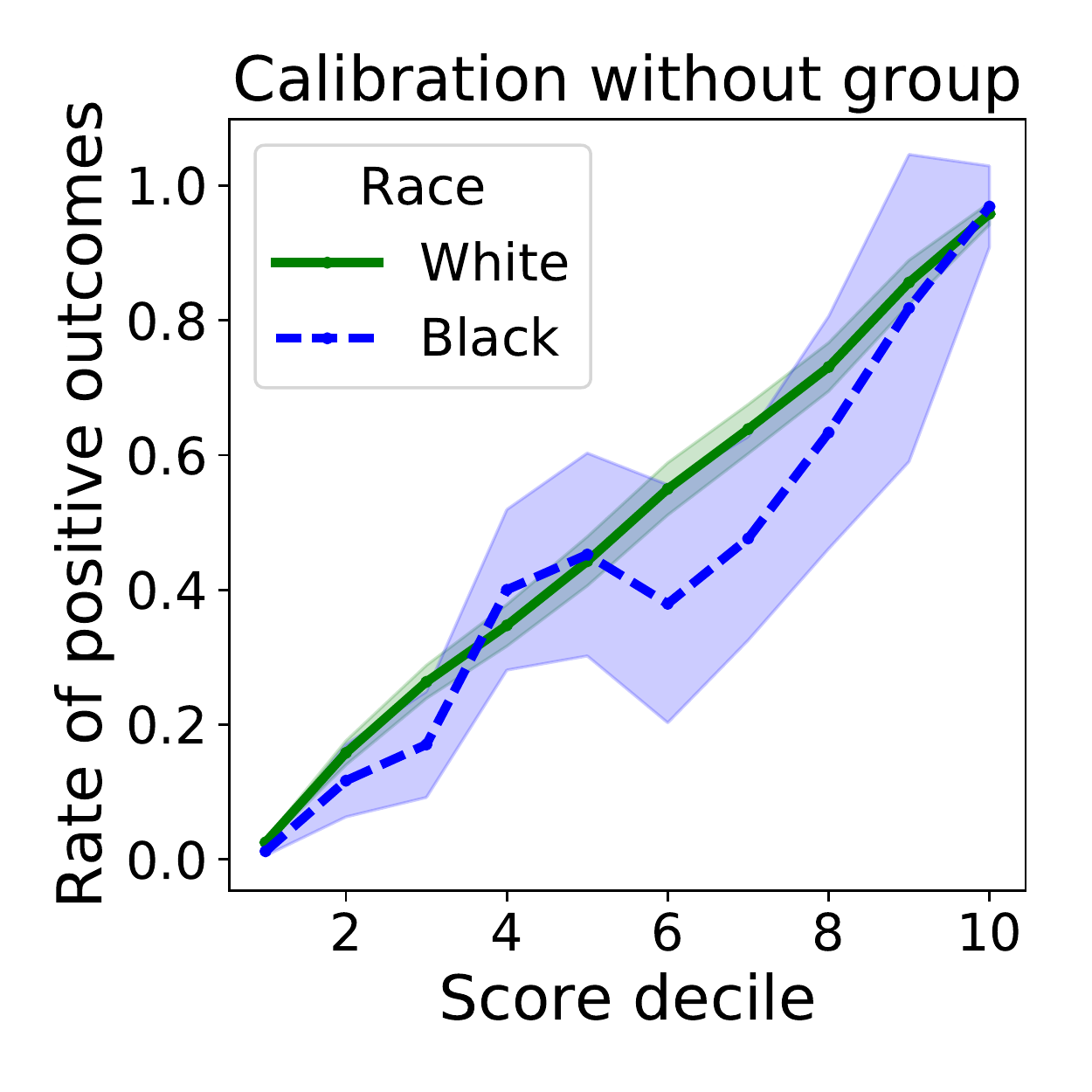}
	\caption{Calibration plot for score not using group attribute} \label{fig:calib_wo_group}
\end{figure}

 In this section, we examine the sufficiency of ERM scores, with respect to gender and race. When all available features were used in the regression, including sensitive attributes, the empirical risk minimizer of the logistic loss is sufficient and calibrated with respect to both gender and race, as seen in Figure~\ref{fig:calib_w_group}. However, sufficiency can hold approximately even when the score is not a function of the group attribute. Figure \ref{fig:calib_wo_group} shows that without the group variable, the ERM score is only slightly less calibrated; the confidence intervals for both groups still overlap at every score decile.

\subsection{Simultaneous sufficiency with respect to multiple group attributes}\label{sec:expt_multi}
\begin{figure}[htbp]
	\includegraphics[width=0.5\columnwidth, trim=0 10 0 10 , clip]{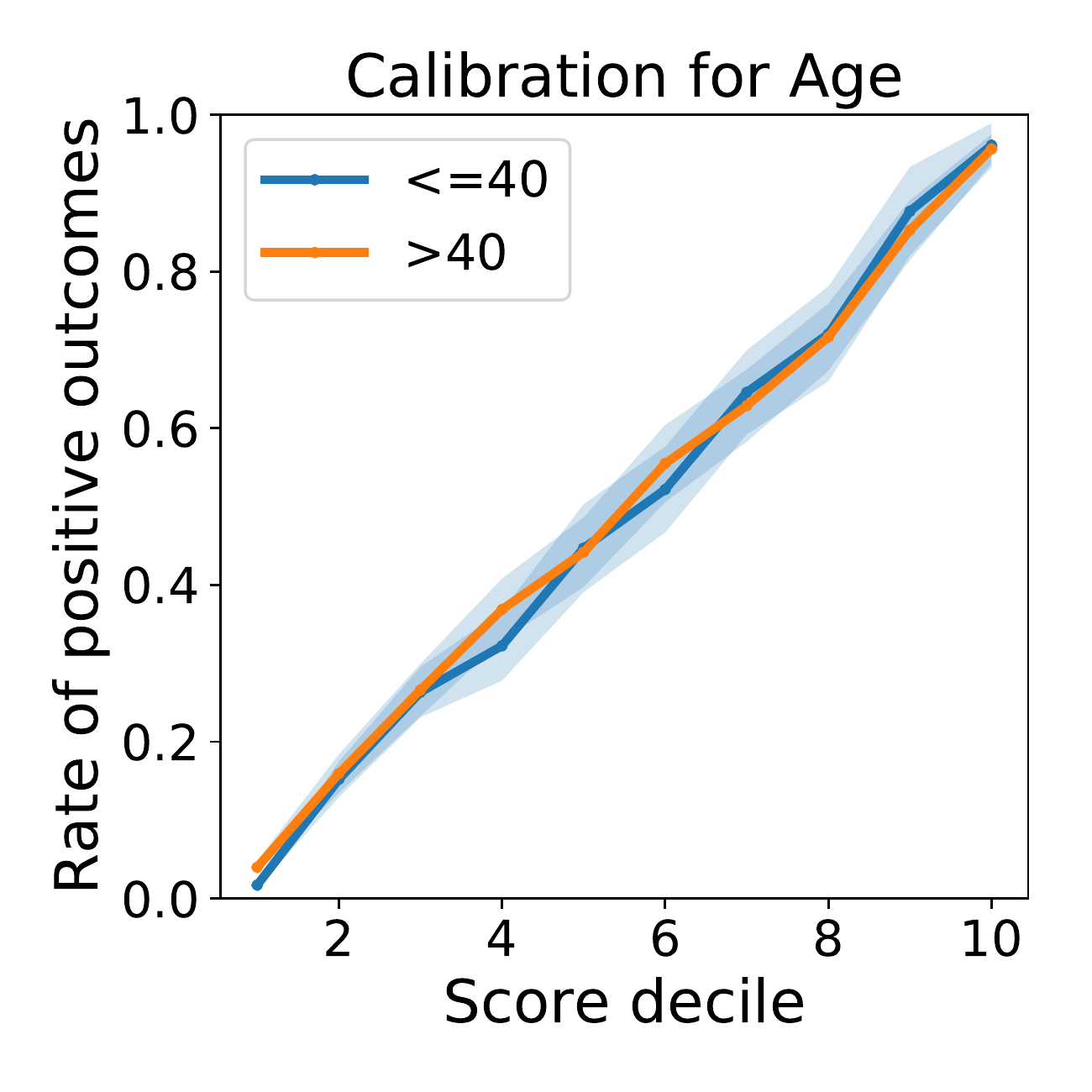}%
	\includegraphics[width=0.5\columnwidth, trim=0 10 0 10 , clip]{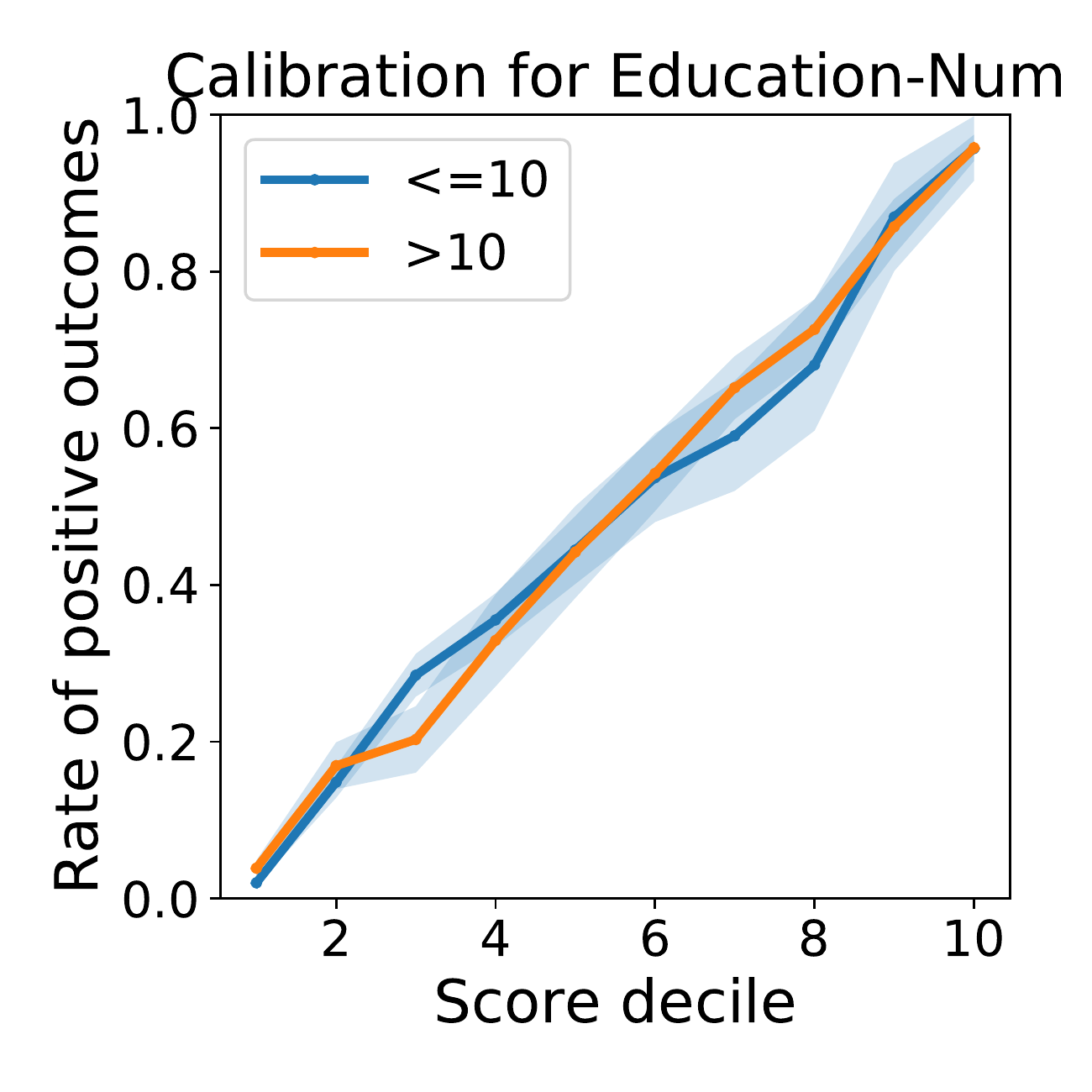}
	\includegraphics[width=0.5\columnwidth, trim=0 10 0 10 , clip]{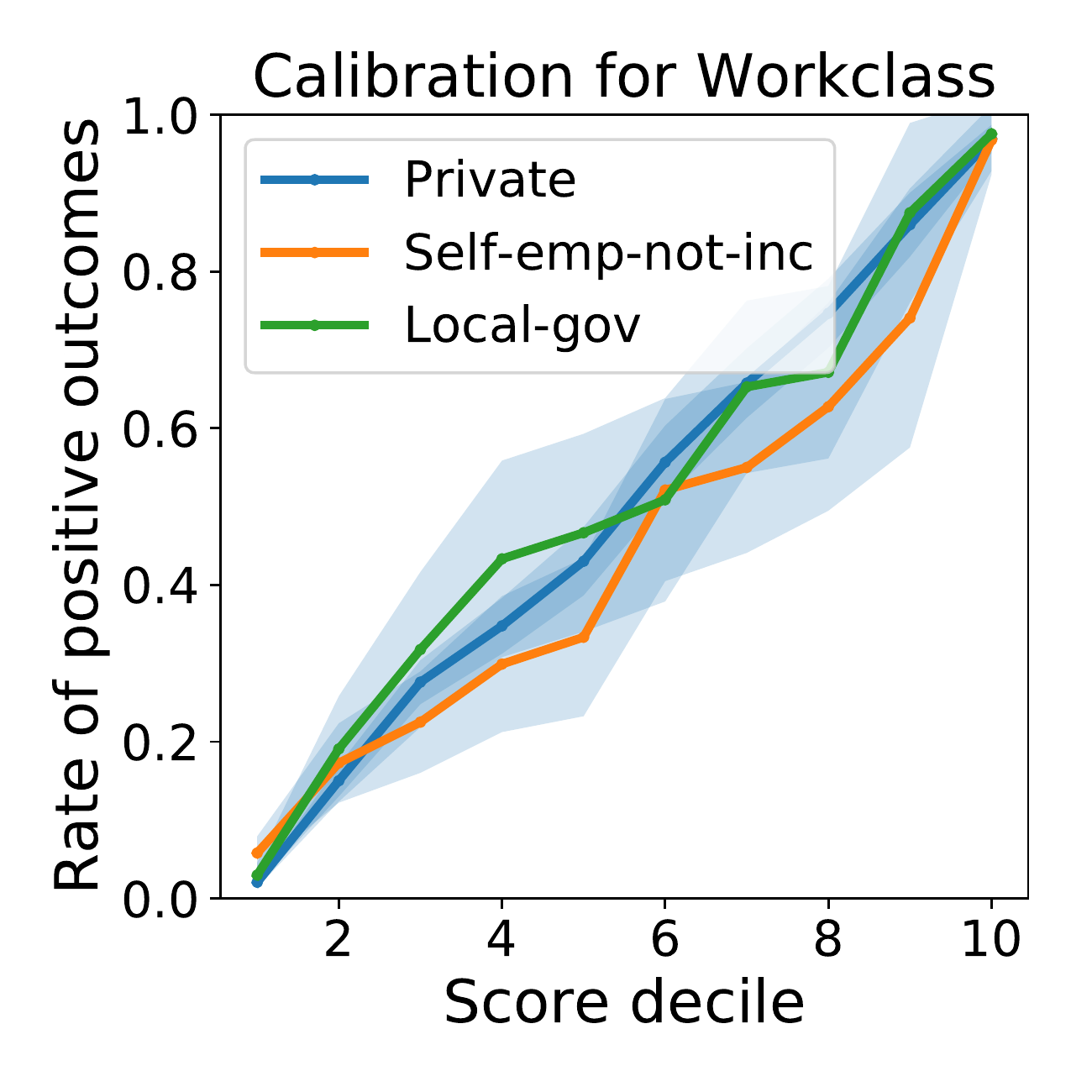}%
	\includegraphics[width=0.5\columnwidth, trim=0 10 0 10 , clip]{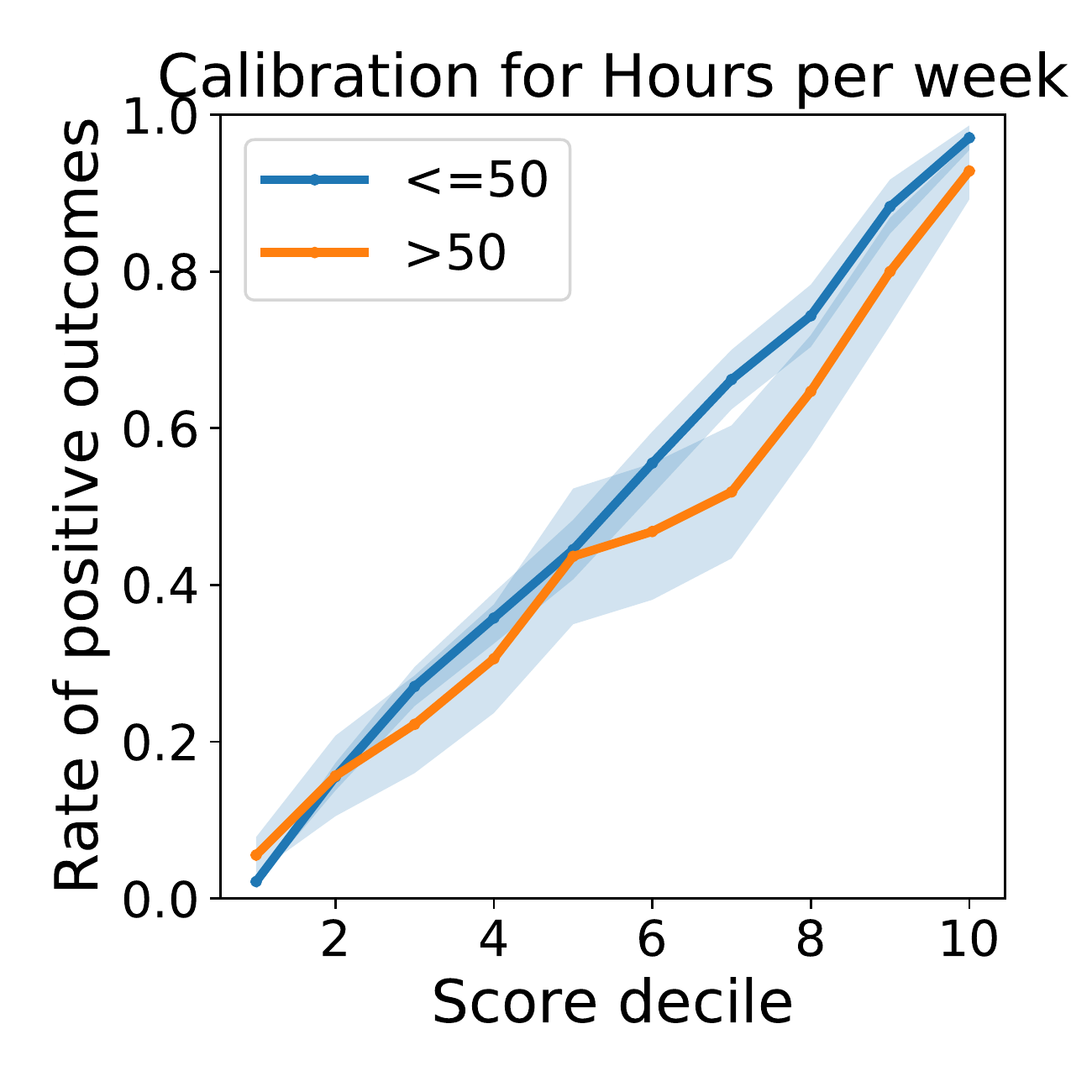}
	\caption{Calibration plot with respect to other group attributes}\label{fig:multi_calib}
\end{figure}

\begin{figure}[htbp]
	\includegraphics[width=0.5\columnwidth, trim=0 10 0 10 , clip]{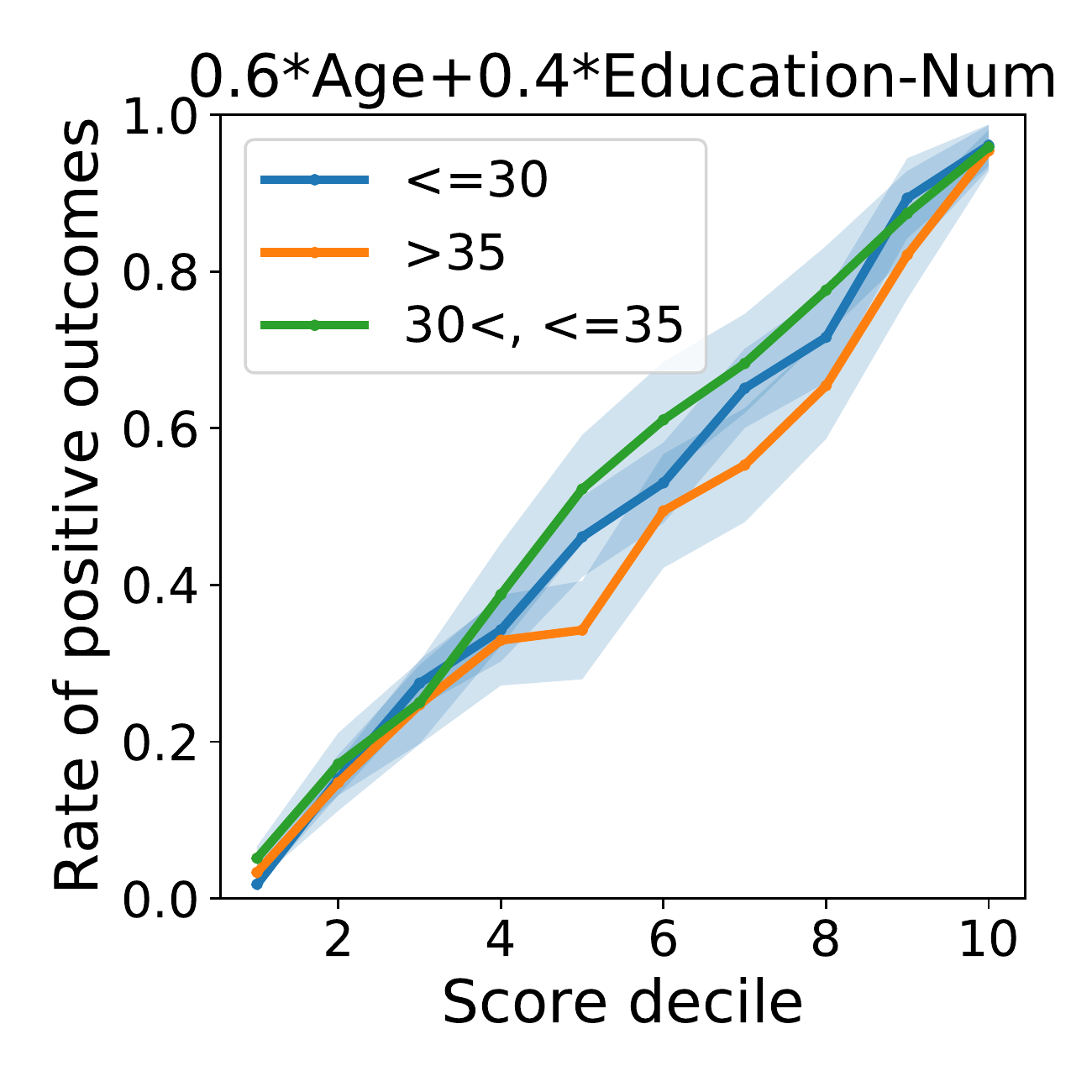}%
	\includegraphics[width=0.5\columnwidth, trim=0 10 0 10 , clip]{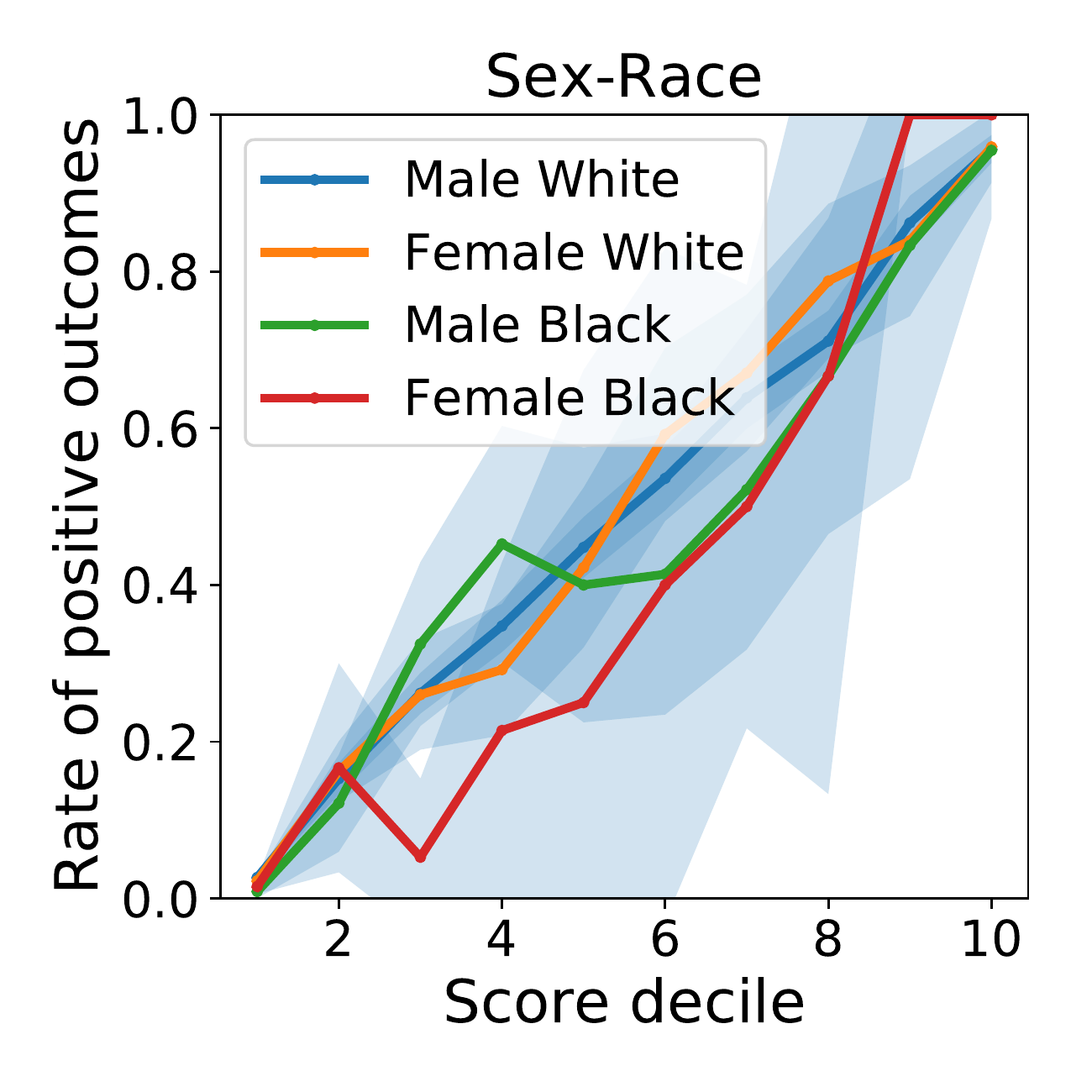}%
	\caption{Calibration plot with respect to combinations of features: linear combination (left), intersectional combination (right)}\label{fig:multi_calib2}
\end{figure}

Furthermore, we observe that empirical risk minimization with logistic regression also achieves approximate sufficiency with respect to any other group attribute defined on the basis of the given features, not only gender and race. In Figure \ref{fig:multi_calib}, we show the calibration plot for the ERM score with respect to Age, Education-Num, Workclass, and Hours per week; Figure \ref{fig:multi_calib2} considers combinations of two features. In each case, the confidence intervals for the rate of positive outcomes for all groups overlap at all, if not most, score deciles. In particular, Figure \ref{fig:multi_calib2} (right) shows that the ERM score is close to sufficient and calibrated even for a newly defined group attribute that is the intersectional combination of race and gender. The calibration plots for other features, as well as implementation details, can be found in Appendix~\ref{app:multicalib}.

\subsection{Sufficiency improves with model accuracy and model flexibility}\label{sec:expt_model}

\begin{figure*}[htbp]
	\centering
	\includegraphics[width=0.5\textwidth]{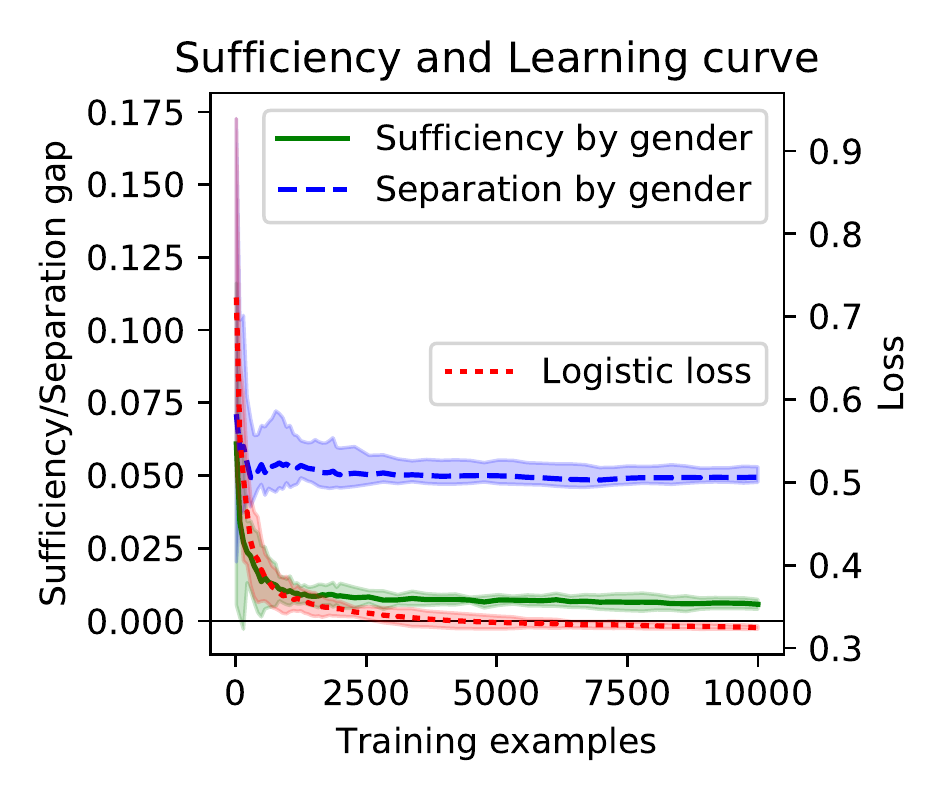}%
	\includegraphics[width=0.5\textwidth]{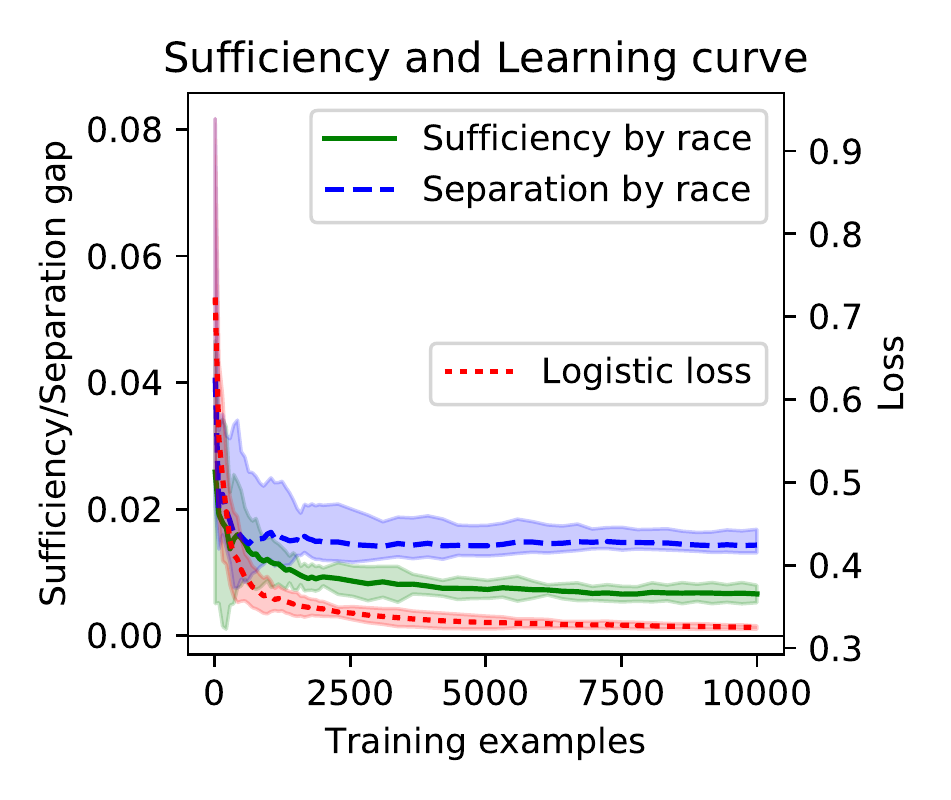}
	\caption{Sufficiency, Separation, and Logistic Loss vs. Number of training examples}\label{fig:learningcurve}
	
\end{figure*}

\begin{figure}[htbp]
	\centering
	\includegraphics[width=0.5\columnwidth, trim=0 0 10 0, clip]{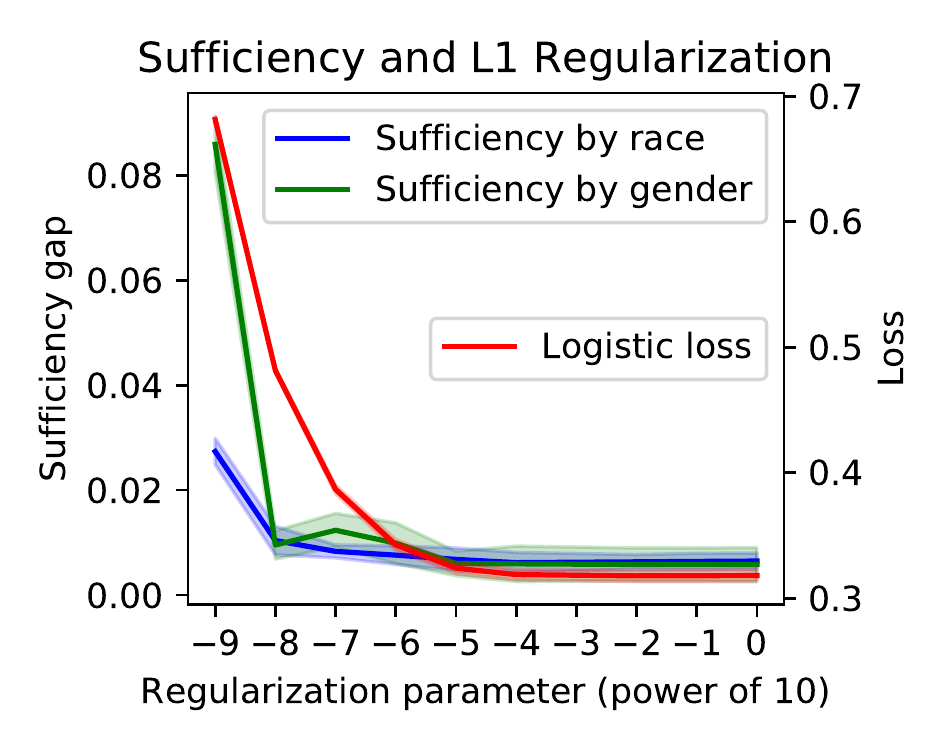}%
	\includegraphics[width=0.5\columnwidth, trim=0 0 10 0, clip]{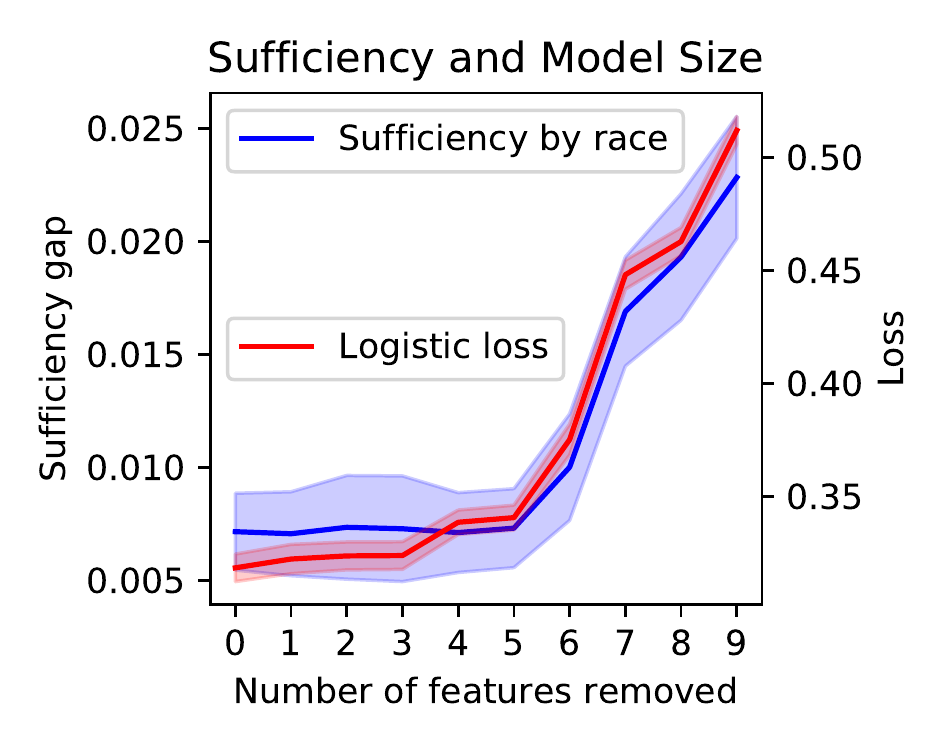}
	\caption{Sufficiency for models trained with different L1 regularization parameters (left) and with different number of features (right)} \label{fig:model-class}
\end{figure}

Our theoretical results suggest that the sufficiency gap of a score function is tightly related to its excess risk. In general, it is impossible to determine the excess risk of a given classifier with respect to the Bayes risk $\calL^*$ from experimental data. Instead we shall examine how the sufficiency gap of a score trained by logistic regression varies with the number of samples and the model class, both of which were chosen because of their impact on the excess risk of the score. 




Specifically, we explore the effects of decreased risk on sufficiency gap due to (a) increased number of training examples (Figure~\ref{fig:learningcurve}) and (b) increased expressiveness of the class $\F$ of score functions (Figure~\ref{fig:model-class}). As the number of training samples increases, the gap between the ERM and least-risk score function in a given class $\F$, $ \argmin_{f \in \F}\calL(f)$,  decreases. On the other hand, as the number of model parameters grows, the class $\F$ becomes more expressive, and $\min_{f \in \F}\calL(f)$ may become closer to the Bayes risk $\calL^*$.


Figures \ref{fig:learningcurve} and \ref{fig:model-class} display, for each experiment, the sufficiency gap and logistic loss on a test set averaged over 10 random trials, each using a randomly chosen training set. The shaded region in the figures indicates two standard deviations from the average value. In Figure \ref{fig:learningcurve}, as the number of training examples increase, the logistic loss of the score decreases, and so does the sufficiency gap. For the race group attribute, we even observe that the sufficiency gap is going to zero; this is predicted by Theorem \ref{thm:main_upper} as the risk of the score approaches the Bayes risk. Figure \ref{fig:learningcurve} also displays the separation gap of the scores. Indeed, the separation gap is bounded away from zero, as predicted by Corollary~\ref{cor:erm_sep}, and does not decrease with the number of training examples. This corroborates our finding that unconstrained machine learning cannot achieve the separation notion of fairness even with infinite data samples.

In Figure \ref{fig:model-class} (right), we gradually restrict the model class by reducing the number of features used in logistic regression. As the number of features decreases, the logistic loss increases and so does the sufficiency gap. In Figure \ref{fig:model-class} (left), we implicitly restrict the model class by varying the regularization parameter. 
In this case, a smaller regularization parameter corresponds to more severe regularization, which constrains the learned weights to be inside a smaller L1 ball. As we increase regularization, the logistic loss increases and so does the sufficiency gap. Both experiments show that the sufficiency gap is reduced when the model class is enlarged, again demonstrating its tight connection to the excess risk.


\paragraph{Conclusion}

In summary, our results show that group calibration follows from closeness to the risk of the calibrated Bayes optimal score function. 
Consequently, empirical risk minimization is a simple and efficient recipe for achieving group calibration, provided that (1) the function class is sufficiently rich, (2) there are enough training samples, and (3) the group attribute can be approximately predicted from the available features. On the other hand, we show that group calibration does not and cannot solve fairness concerns that pertain to the Bayes optimal score function, such as the violation of separation and independence. 

More broadly, our findings suggest that group calibration is an appropriate notion of fairness \emph{only} when we expect unconstrained machine learning to be fair, given sufficient data. Stated otherwise, focusing on calibration alone is likely insufficient to mitigate the negative impacts of unconstrained machine learning.  








\bibliographystyle{plainnat}

\bibliography{grpcali}

\newpage 
\begin{appendix}

\section{Additional Proofs for Section~\ref{sec:upper_bound_calib}}

In this section we prove Proposition~\ref{prop:bayes_calibrated}, and Lemma~\ref{lem:logistic}.

\subsection{Proof of Proposition~\ref{prop:bayes_calibrated}\label{proof:prop_bayes_calib}}
	
	Recall that $\fbayes(X,A) = \E\left[Y \mid \classx, \attx \right] $.
	
	By the tower rule for conditional expectation,
	\begin{align*}
	\Pr\left[Y = 1 \mid \fbayes(\classx,\attx), \attx\right] &= \E[ Y\mid \fbayes(\classx,\attx),\attx] \\
	&=\E\left[ \E\left[Y\mid X,\attx\right] \mid \fbayes(\classx,\attx),\attx\right] \\
	&= \E[\fbayes(\classx,\attx)\mid \fbayes(\classx,\attx),\attx]~=~\fbayes(\classx,\attx)~,
	\end{align*}
and, 
\begin{align*}
\Pr[Y= 1\mid \fbayes(\classx,\attx)] &= \E[Y\mid \fbayes(\classx,\attx)]\\
&=\E[ \E[Y\mid X,A] \mid \fbayes(\classx,\attx)]\\
&= \E[ \fbayes(\classx,\attx) \mid \fbayes(\classx,\attx)] =\fbayes(\classx,\attx).
\end{align*}
Therefore, the \bayscore~$\fbayes(X)$ is sufficient and calibrated. 

More generally, conditional expectations of the form $f(X) = \E[Y\mid \Phi(X),A]$ are calibrated, where $\Phi: \calX \to \calX'$ can be any transformation of the features. This follows similarly from the tower rule.

\subsection{Proof of Lemma~\ref{lem:logistic}\label{proof:lem:logistic}}
To see that this is true, first note that $\ell(f,y)$ is a Bregman divergence. We can easily check that $\E[\left.\grad_f \ell(f(x,a),Y)\right|_{f^B}] = \E[\frac{Y}{f^B} - \frac{1-Y}{1-f^B}]=0$.
	Finally, $\kappa$-strong convexity follows from Pinkser's inequality for Bernoulli random variables:
	\begin{align*}
	 (f - f')^2 \leq \frac{\log 2}{2} \left( f' \log \frac{f'}{f} + (1-f') \log \frac{1-f'}{1-f}\right) = \frac{\log 2}{2} \ell(f,f')  .\end{align*}

\section{Proof of Theorem \ref{thm:main_upper}} \label{sec:main_ub_proofs}
Throughout, we consider a fixed distribution $\calD$ and attribute $\attx$. We shall also use the shorthand $f= f(X)$ and $\fbayes = \fbayes(X,A)$.
We begin by proving the following lemma, which establishes Theorem~\ref{thm:main_upper} in the case where $f$ is the squared loss:
\begin{lemma}\label{calibl2bound}Let $\fbayes$ be the Bayes classifier, and let $f$ denote any function. Then,
	\begin{align}
	\Delweak_f &\leq 4\sqrt{\E_{X,A}[(f- \fbayes)^2]}, \label{eq:lemma_delweak} \\ 
		\forall a \in \calA, \quad \Delweak_{f}(a;A) &\le 2\sqrt{\frac{\E_{X,A}[(f- \fbayes)^2]}{\Pr[A=a]}}. \label{eq:lemma_worst_case_weak}\\
	\Delstrong_f &\leq \sqrt{\E_{X,A}[(f- \fbayes)^2]}, \label{eq:lemma_delstrong}\\
		\forall a \in \calA, \quad \Delstrong_f(a;A)&\le 2\sqrt{\frac{\E_{X,A}[(f- \fbayes)^2]}{\Pr[A=a] }}. \label{eq:lemma_worst_case_strong}
	\end{align}
\end{lemma}
To conclude the proof of Theorem~\ref{thm:main_upper}, we first show that $\E[\ell(f, f^B)] = \E[\ell(f,Y)] - \E[\ell(f^B, Y)]$. Since $\ell$ is a Bregman divergence and calibrated at $\fbayes$ (Assumption~\ref{asm}), we have 
\iftoggle{icml}
{	\begin{align*}
	\E[\ell(f, f^B)] =& \E[\ell(f, Y)] - \E[\ell(f^B, Y)] \\
	&+ \E[(g'(Y) -g'(f^B))\cdot(f-f^B)] \\
	=&\E[\ell(f, Y)] - \E[\ell(f^B, Y)] \\
	&-\E[\E[\left.\grad_f \ell(f,Y)\right|_{f^B}\mid X,A]\cdot(f-f^B)] \\
	=&\E[\ell(f, Y)] - \E[\ell(f^B, Y)]
	\end{align*}
	}
{	\begin{align*}
	\E[\ell(f, f^B)] =& \E[\ell(f, Y)] - \E[\ell(f^B, Y)] + \E[(g'(Y) -g'(f^B))\cdot(f-f^B)] \\
	=&\E[\ell(f, Y)] - \E[\ell(f^B, Y)] -\E[\underbrace{\E[\left.\grad_f \ell(f,Y)\right|_{f^B}\mid X,A]}_{=0}\cdot(f-f^B)] \\
	=&\E[\ell(f, Y)] - \E[\ell(f^B, Y)].
	\end{align*}}
	Moreover, by strong convexity,  we have that 
	\iftoggle{icml}
	{$ \calL(f) $ is bounded from below by $\frac{1}{\kappa}\E[(f-Y)^2]$. Thus,
		\begin{align*}
		\kappa \E[(f - f^B)^2] \leq \E[\ell(f,f^B)] &= \E[\ell(f,Y)]- \E[\ell(f^B,Y)] \\
		&= \calL(f) -  \calL(f^B)~.
		\end{align*}
		}
	{$ \calL(f) \geq \frac{1}{\kappa}\E[(f-Y)^2]$. Thus,
		\begin{align*}
	\kappa \E[(f - f^B)^2] &\leq \E[\ell(f,f^B)] = \E[\ell(f,Y)]- \E[\ell(f^B,Y)] = \calL(f) -  L(f^B)~.
	\end{align*}}
	Applying Lemma~\ref{calibl2bound} concludes the proof.

\subsection{Proof of Lemma~\ref{calibl2bound} \eqref{eq:lemma_delweak} and \eqref{eq:lemma_worst_case_weak}}
	
First we bound the $\L_2$ difference of the conditional expectations. Note that since $f = f(X)$, 
	\iftoggle{icml}{\begin{align}\label{eq:fbayes_idone}
		\E[Y\mid f, A] &=  \E[\E[Y| X, A,f] \mid f, A] \nonumber \\
		&= \E[\E[Y| X, A] \mid f, A] = 
		\E[f^B\mid f,A].
		\end{align}
		}
		{			
	\begin{align}\label{eq:fbayes_idone}
	\E[Y\mid f, A] =  \E[\E[Y| X, A,f] \mid f, A] = \E[\E[Y| X, A] \mid f, A] = 
	\E[f^B\mid f,A].
	\end{align}
}
	Moreover, by the definition of $\fbayes$
	\iftoggle{icml}{
	 \begin{align}
	 \E[Y\mid \fbayes,A] &= \E[\E[Y\mid A,X,\fbayes],\fbayes,A] \nonumber\\
	 &= \E[\E[Y\mid A,X, \E[Y\mid  A,X]]\nonumber\\
	  &= \E[\E[Y\mid  A,X],A,X] = \E[Y \mid A,X] = \fbayes\label{eq:fbayes_idtwo},
	 \end{align}
	}
	{\begin{align}
		\E[Y\mid \fbayes,A] &= \E[\E[Y\mid A,X,\fbayes],\fbayes,A]~=~ \E[\E[Y\mid A,X, \E[Y\mid  A,X]]\nonumber\\
		&= \E[\E[Y\mid  A,X],A,X] = \E[Y \mid A,X] = \fbayes\label{eq:fbayes_idtwo},
		\end{align}
		}
	 and thus, by~\eqref{eq:fbayes_idone} and~\eqref{eq:fbayes_idtwo}, we have
	 \iftoggle{icml}
{
	 \begin{align}
 		& ~\E_{X,A}[( \E[Y\mid f, A] - \E[Y\mid f^B,A])^2] \nonumber \\
	 	=&~ \E_{X,A}[(\E[f^B\mid f, A] - f^B)^2 ] \quad \text{by }\eqref{eq:fbayes_idone}~\text{and}~\eqref{eq:fbayes_idtwo}\nonumber\\
	 	= &~\E_{X,A}[( \E[f^B-f\mid f, A] +f- f^B)^2 ]\nonumber\\
	 	\leq &~2\E_{X,A}[(\E[f^B-f\mid f, A] )^2+(f- f^B)^2 ] \nonumber\\
	 	\leq &~2\E_{X,A}[\E[(f^B-f)^2\mid f, A] +(f- f^B)^2 ] \label{eq:jensens}\\
	 	=&~4\E_{X,A}[(f- f^B)^2] \label{eq:double_star}~,
	 	\end{align}
	 	where \eqref{eq:jensens} follows from Jensen's inequality.
}
{
	\begin{align}
	\E_{X,A}[( \E[Y\mid f, A] - \E[Y\mid f^B,A])^2]&= \E_{X,A}[(\E[f^B\mid f, A] - f^B)^2 ] \quad \text{by }\eqref{eq:fbayes_idone}~\text{and}~\eqref{eq:fbayes_idtwo}\nonumber\\
	&= \E_{X,A}[( \E[f^B-f\mid f, A] +f- f^B)^2 ]\nonumber\\
	&\leq 2\E_{X,A}[(\E[f^B-f\mid f, A] )^2+(f- f^B)^2 ] \nonumber\\
	&\leq 2\E_{X,A}[\E[(f^B-f)^2\mid f, A] +(f- f^B)^2 ]  \label{eq:jensens}\\
	&=4\E_{X,A}[(f- f^B)^2] \label{eq:double_star}~,
	\end{align}
	where \eqref{eq:jensens} follows from Jensen's inequality.
}
	Similarly, one has 
	\begin{align}
	\E_{X,A}[( \E[Y\mid f] - \E[Y\mid f^B])^2] \le 4\E_{X,A}[(f- f^B)^2].\label{eq:double_star_prime}
	\end{align}

	We then find that 
	\begin{align}
	\Delta_f  &= \Delta_f - \Delta_{f^B} \nonumber \\
	&= \E_{X,A}[| \E[Y\mid f] - \E[Y\mid f,A]|-| \E[Y\mid f^B] - \E[Y\mid f^B,A]|]  \nonumber\\
	&=\E_{X,A}[| \E[Y\mid f] - \E[Y\mid f,A]|-| \E[Y\mid f^B] - \E[Y\mid f^B,A]|]  \nonumber\\
	&= \E_{X,A}[\sqrt{ (\E[Y\mid f] - \E[Y\mid f,A] - (\E[Y\mid f^B] - \E[Y\mid f^B,A]))^2}]  \nonumber\\
	&\leq \sqrt{\E_{X,A}[ (\E[Y\mid f] - \E[Y\mid f,A] - (\E[Y\mid f^B] - \E[Y\mid f^B,A]))^2]} \label{eq:jensens2}\\
	&\le \sqrt{2\E_{X,A}[(\E[Y\mid f] - \E[Y\mid f^B])^2]+ 2\E_{X,A}[ (\E[Y\mid f,A] - \E[Y\mid f^B,A])^2] } \nonumber\\
	&\le \sqrt{8\E_{X,A}[(f- f^B)^2] + 8\E_{X,A}[(f- f^B)^2]}~ = 4\sqrt{\E_{X,A}[(f-f^B)^2]} \label{concluding_line},
	\end{align}
	where we've applied Jensen's inequality in \eqref{eq:jensens2}, and the inequality in~\eqref{concluding_line} uses \eqref{eq:double_star}~and~\eqref{eq:double_star_prime}.

	Similarly, for a fixed group $A=a$, we have
	\begin{align*}
	\Delweak_{f}(a;A) &=  \E_{X}[| \E[Y\mid f] - \E[Y\mid f,A]|-| \E[Y\mid f^B] - \E[Y\mid f^B,A]| \mid A=a]  \\
	&\leq 4\sqrt{\E_{X}[(f-f^B)^2\mid A=a]} \\
	&\leq 4\sqrt{\frac{1}{\Pr[A=a]}\E_{X,A}[(f-f^B)^2]} \\
	\end{align*}

\subsection{Proof of Lemma~\ref{calibl2bound} \eqref{eq:lemma_delstrong} and \eqref{eq:lemma_worst_case_strong}}
By~\eqref{eq:fbayes_idone}, we have
\begin{equation}
\E_{X,A}[( \E[Y\mid f, A] - f)^2] =\E_{X,A}[( \E[\fbayes \mid f, A] - f)^2] \leq \E_{X,A}[(\fbayes -f)^2 ],
\end{equation}
where the inequality follows from Jensen's inequality and the tower property. We then find that, by Jensen's inequality,
\begin{align*}
\Delstrong_f &= \E_{X,A}[| \E[Y\mid f,A] -f|]\\
&=\E_{X,A}[ \sqrt{(\E[Y\mid f,A] -f)^2}] \\
&\leq \sqrt{ \E_{X,A}[ (\E[Y\mid f,A] -f)^2] } \\
&\leq \sqrt{\E_{X,A}[(\fbayes -f)^2 ]}.
\end{align*}

	Similarly, for an fixed group $A=a$, we have
	\begin{align*}
	\Delstrong_{f}(a;A) &=  \E_{X,A}[| \E[Y\mid f,A] -f| \mid A=a]  \\
	&\leq \sqrt{\E_{X}[(f-f^B)^2\mid A=a]} \\
	&\leq \sqrt{\frac{1}{\Pr[A=a]}\E_{X,A}[(f-f^B)^2]}.
	\end{align*}
\subsection{Further Examples for Calibration and Sufficiency Bounds\label{sec:further_examples}}\label{app:proof_examples}
We present two further examples under which one can meaningfully bound the excess risk, and consequently the sufficiency and calibration gaps, in the incomplete information setting.
In the next example, we examine sufficiency when $\fhat$ is precisely the uncalibrated Bayes score $\func$. The following lemma establishes an upper bound on the sufficiency gap of the uncalibrated Bayes score in terms of the conditional mutual information between $Y$ and $A$, conditioning on $\classx$. It is a simple consequence of Tao's inequality. 
\begin{example}[Sufficiency for uncalibrated Bayes score]\label{ex:uncalibbayes} Suppose $X$ and $A$ are discrete $\calD$-measurable random variables, and $\F$ is the set of all functions $f:\X \to [0,1]$. Denote $f^* = \argmin_{f\in \F} \calL(f)$. Then, under Assumption \ref{asm}, $f^* = \E[Y\mid X]$ and  
	\begin{align*} \sufgap_{f^*}(A) \leq \sqrt{2\log2 I(Y;A \mid X)}.\end{align*}
\end{example}
\begin{proof} For $f = \E[Y\mid X]$, we have the following identity for $\Delta_f(A)$ by the tower rule: 
\begin{align*} \Delta_f(A) = \E|\E[Y\mid f] - \E[Y\mid f,A]| = \E| \E[Y\mid X] - \E[Y \mid X,A] |
\end{align*}

By applying Tao's inequality \citep{Tao2006SzemerdisRL,ahlswede07}, we have that 
\begin{align*} \E| \E[Y\mid X] - \E[Y \mid X,A] | \leq   \sqrt{2\log2 I(Y;A\mid X)} 
\end{align*}

Note that  $I(Y;(X,A)\mid X) = I(Y;A\mid X)$ and the result follows.

\end{proof}

Lastly, we consider an example when the attribute $\attx$ is continuous, and there exists a function $g$ which approximately predicts $A$ from $X$.
%
%
\begin{lemma}[Calibration and sufficiency for continuous group attribute $\attx$]
	Suppose (1) $\ell$ is the logistic loss, and (2) there exists $g:\clspace \to\attspace$ such that $\E[|\attx - g(\classx)|] \leq \delta_1$. Let $\tilde{F} = \{f:\clspace \times \attspace \to [0,1] \textnormal{ s.t. } f(x,g(x)) \in \F \}$. Denote $f^* = \arg\inf_{f\in \F}L(f)$ and $\tilde{f} = \arg\inf_{f\in \tilde{F}}L(f)$. Further suppose (3) $\tilde{f}$ is $\beta$-Lipschitz in its second argument, that is $\forall x, |\tilde{f}(x,a) -\tilde{f}(x, a')| \leq \beta|a-a'|$ and (4) $\Pr\{\delta_2 < \tilde{f} < \delta_3\} = 1$ for some $\delta_2, \delta_3 \in (0,1)$. Then, 
	\[\max\{\Delweak_{f^*}(A),\Delstrong_{f^*}(A)  \} \leq C \sqrt{\min_{f\in \tilde{\F}}\calL(f) - \calL(\fbayes) +\beta\frac{\delta_1}{\min\{\delta_2, 1-\delta_3\}}} , \]
	where $C$ is a universal constant.
\end{lemma}

\begin{proof}
	By computation, we have
		\begin{align*}
		\calL(f^*) - \calL(\fbayes) &= \E[\ell(f^*(\classx), Y)] - \calL(\fbayes)\\
		& \leq \E[\ell(\tilde{f}(\classx, g(\classx)), Y)] - \calL(\fbayes)   \\
		&\leq \E[\ell(\tilde{f}(\classx, A), Y)] - \calL(\fbayes) +\beta\frac{\delta_1}{\min\{\delta_2, 1-\delta_3\}}  \\
		&= \min_{f\in \tilde{\F}}\calL(f) - \calL(\fbayes) +\beta\frac{\delta_1}{\min\{\delta_2, 1-\delta_3\}}.
		\end{align*}
		The last inequality follows from the fact that $\forall y, ~\ell(\cdot, y)$ is $\frac{1}{\min\{\delta_2, 1-\delta_3\}}$-Lipschitz on $[\delta_2, \delta_3]$, and that $\tilde{f}$ is only supported on $[\delta_2, \delta_3]$. Then the conclusion follows from Theorem~\ref{thm:main_upper}.
\end{proof}

The above result shows that in the incomplete information setting we may be able to bound the sufficiency gap of the population risk minimizer of class $\F$ by the approximation error for an auxiliary class $\Ftil$ up to an additional error term that accounts for how well $g(X)$ predicts $A$.
In other words, a score obtained by empirical risk minimization will have low calibration gap with respect to any group attribute $\attx$ that is sufficiently encoded in the features that are used by the score, $\classx$, up to the flexibility allowed by the chosen class $\F$. Thus, our theoretical results also suggest that ERM can achieve simultaneous approximate sufficiency and calibration with respect to all such group attributes.

\section{Supplementary Material for the Average Sufficiency and Calibration Lower Bound}\label{app:lb_proof}
Before stating the precise version of the lower bound Theorem~\ref{thm:main_lb_body}, we begin by giving an explicit construction of the hard instance. Let $\sphereone := \{u\in\R^2:\|u\|_2 = 1\}$ be the circle in $\R^2$. For each $u \in \R^2$, we consider the following affine score functions $f_{u}: \R^2 \to \R$ and attributes $A_w \in \{-1,0, 1\}$:
\begin{align*}
f_{u}(X):= \frac{1}{2} + \frac{\langle u, X\rangle}{4} \quad \text{and} \quad A_w  := \sign(\langle X, u \rangle).
\end{align*}
We note that $f_u(X) \in [\tfrac{1}{4},\tfrac{3}{4}]$ whenever $u, X \in \sphereone$, and that our attributes are functions of our features. Lastly, we let $\calF := \{f_{u} : u \in \R^2\}$, and for $\psi \in \sphereone$, we let $\Dwst$ denote the joint distribution where 
\begin{align*}
\Dwst:= X \unifsim \sphereone \text{ and } Y \mid X = \Bern(f_{\wst}(X)).
\end{align*} Observe that since $A_w = A_w(X)$, we see that the calibrated Bayes score under the distribution $\Dwst$ is just $\fbayes(X,A) = \E_{\Dwst}[Y\mid X,A] = \E_{\Dwst}[Y\mid X] = f_{\wst}(X)$, and thus $\fbayes \in \calF$. Lastly, we shall let $\sufgap_{f;\Dwst}(A_w)$ denote the sufficiency gap of $f$ with respect to $\Dwst$, and $\calgap_{f;\Dwst}(A_w)$ the calibration gap of $f$ with respect to $\Dwst$.
We can now state a more precise version of Theorem~\ref{thm:lower_bound_informal}:
\begin{theorem}[Precise Lower bound for Sufficiency and Calibration ]\label{thm:main_lb_constants} Let $f_w$,$\F$, $\Dwst$, $A_w$, and $\sufgap_{f;\Dwst}(A_w)$ and $\calgap_{f;\Dwst}$ be as above. Then,
\begin{itemize}
\item[(a)] For any classifier $\fhat \in \F$ trained on a sample $\samplen := \{(X_i,Y_i)\}_{i =1}^n$,  any $\delta_1 \in (0,1)$,
{
	\begin{align}\label{eq:delta_lb}
\E_{\wst,w \unifsim \sphereone}\Pr_{\samplen \sim \Dwst}\left[\sufgap_{\fhat;\Dwst}(A_w) \le  \frac{1}{4\pi}\min\left\{1, \sqrt{\frac{3\log(1/\delta_1)}{2n}}\right\} \right] \le 1 - \frac{\delta_1}{4}.
\end{align}
}
Moreover, $\calgap_{\fhat;\Dwst}(A_w) \ge \sufgap_{\fhat;\Dwst}(A_w)$ almost surely. 

\item[(b)] Let $\fls$ denote the ERM under the square loss
\begin{align*}
\fls := \arg\min_{f \in \calF}\sum_{(X_i,Y_i) \in \samplen}(f(X_i) - Y_i)^2.
\end{align*}
Then~\eqref{eq:delta_lb} holds even when $\E_{\wst,\wA \unifsim \sphereone}$ is replaced by a supremum $\sup_{\wst,\wA \in \sphereone}$. Moreover, for any $\delta_2 \in (0,1)$ and $\wst \in \sphereone$,
	\begin{align*}
\Pr_{\samplen \sim \Dwst}\left[\calL_{\Dwst}(\fls) - \calL^*  \le \frac{8 + 6\log(1/\delta_2)}{n}\right] \ge 1 - \delta_2 - 2e^{ - \frac{n}{8 + 4/3}}.
\end{align*} 
where $\calL_{\Dwst}(f) = \E_{(X,Y) \sim \Dwst}[(f(X)-Y)^2]$ denotes population risk under the square loss, and where $\calL^* = \calL_{\Dwst}(f_{\wst}) $ denotes the (calibrated) Bayes risk.
\end{itemize}
\end{theorem}
The remainder of the section is organized as follows. In Section~\ref{sec:part_a_strat}, we given an overview of the proof strategy for part (a). In Section~\ref{proof:fls_performance}, we use standard concentration arguments to establish part (b) of the theorem. Sections~\ref{proof:prop:info} and~\ref{lem:error_sphere} are devoted to proving the major results whose proofs are omitted in the overview of Section~\ref{sec:part_a_strat}.

\subsection{ Proof Strategy for Theorem~\ref{thm:main_lb_constants}, Part (a)\label{sec:part_a_strat}}

In this section, we sketch the proof of Theorem~\ref{thm:main_lb_constants}, Part (a).
%
Since $\fhat \in \F$, we shall write $\fhat = f_{\thetahat}$ for some $\thetahat \in \R^2$.
 We begin by given a precise characterization of the calibration error:
\begin{lemma}\label{lem:error_sphere} Let $\wst,\wA \in \calS^1$, and suppose that either $\thetahat = 0$, or $\mathrm{span}(\thetahat,\wA) = \R^2$. Then, for $(X,Y) \sim \Dwst$, 
\begin{align*}
\Delstrong_{f_{\thetahat},\Dwst}(\Awa) \ge \Delweak_{f_{\thetahat},\Dwst}(\Awa) = \frac{\sqrt{\Phi(\thetahat;\wst)}}{2\pi}, \quad \text{where}~ \Phi(\thetahat;\wst) = \begin{cases}  1 - \langle \wst, \frac{\thetahat}{\|\thetahat\|}\rangle^2 & \thetahat \ne 0 \\
1 & \thetahat = 0 \end{cases}~.
\end{align*}
\end{lemma}
At the heart of Lemma~\ref{lem:error_sphere} is noting that when $\thetahat$ and $w$ are linearly independent, then $f_{\thetahat}(X)$ and $A_w(X) = \sign(|\langle X, w \rangle|)$ uniquely determine $X \in \sphereone$. Hence, $\E[Y|f_{\thetahat}(X), A_w(X)] = \E[Y|X] = f_{\wst}(X)$. In the proof of Lemma~\ref{lem:error_sphere}, we show that a similar simplification occurs in the case that $\thetahat = 0$. Because the attribute $\Awa$ is independent of the distribution $\Dwst$, and because $\wA \unifsim \sphereone$, we have that, for any $\wst$,
\begin{align}
\Pr_{\wA, S^n \sim \Dwst}[\{\mathrm{span}(\wA,\thetahat) = \R^2\} \cup \{\thetahat = 0\}] = 1 \label{eq:full_span},
\end{align}
so that the conditions of Lemma~\ref{lem:error_sphere} hold with probability one.

Next, we observe that $\Phi(\thetahat;\wst)$ corresponds to the square norm of the projection of $\wst$ onto a direction perpendicular to $\thetahat$, or equivalently, the square of the sign of the angle between $\thetahat$ and $\wst$. Note that calibration can occur when the angle between $\thetahat$ and $\wst$ is either close to zero, or close to $\pi$-radians; this is in contrast to prediction, where a small loss implies that the angle between $\thetahat$ and $\wst$ is necessarily close to zero. Nevertheless, we can still prove an information theoretic lower bound on the probability that $\Phi(\thetahat;\wst)$ is small by a reduction to binary hypothesis testing. This is achieved in the next proposition:
\begin{proposition}\label{prop:info} For any $n \ge 1$, $\delta \in (0,1)$, and any estimator $\thetahat$,
\iftoggle{sigconf}
{
\begin{align*}
\E_{\wst \unifsim \calS^1}\Pr_{S^n \sim \Dwst}\left[\Phi(\thetahat;\wst) \le \min\left\{\frac{1}{2}, \frac{3\log (\sfrac{1}{\delta})}{n}\right\} \right] \le 1- \frac{\delta}{4}.
\end{align*}
}
{
	\begin{align*}
\E_{\wst \unifsim \calS^1}\Pr_{S^n \sim \Dwst}\left[\Phi(\thetahat;\wst) \le \min\left\{\frac{1}{2}, \frac{3\log (1/\delta)}{n}\right\} \right] \le 1- \frac{\delta}{4}.
\end{align*}
}
\end{proposition}
The first part of Theorem~\ref{thm:main_lb_constants}, Equation~\eqref{eq:delta_lb}, now follows immediately from combining the bound in Proposition~\ref{prop:info},~\eqref{eq:full_span}, and the computation of $\sufgap_{f_{\thetahat},\Dwst}(\Awa)$ and $\calgap_{f_{\thetahat},\Dwst}(\Awa)$ in Lemma~\ref{lem:error_sphere}. The proof of Lemma~\ref{lem:error_sphere} and Proposition~\ref{prop:info} are deferred to Sections~\ref{proof:lem:error_sphere} and~\ref{proof:prop:info}, respectively. 

%

%
%


\subsection{Proof of Theorem~\ref{thm:main_lb_constants}, Part (b): Analysis of $\fls$\label{proof:fls_performance}} 

We can write $\fls = f_{\wls}$, where 
\begin{align*}
\wls := \arg\min_{w} \sum_{(X_i,Y_i) \in \samplen} (f_{w}(X_i) - Y_i)^2
\end{align*}
We readily see that the distribution of $\wls$ marginalized over $\wst \unifsim \sphereone$ is radially symmetric. Hence, the conclusion of~\eqref{eq:full_span} holds for any fixed $\wA \in \sphereone$.

Moreover, since $\sufgap_{f_{\wls},\Dwst}( \Awa) = \frac{\sqrt{\Phi(\wls;\wst)}}{2\pi}$, and both the least-squares algorithm and the error $\Phi(\cdot,\cdot)$ are radially symmetric, we see that for any $t$, $\Pr_{S^n \sim \Dwst}[\sufgap_{f_{\wls},\Dwst}(\Awa)  \le t]$ does not depend on $\wst \in \sphereone$ either.

It now suffices to prove an upper bound for least squares. We have that
\begin{align*}
\calL(\fls) - \calL^* &= \E\left[\left(f_{\wls}(A)(X) - f_{\wst}(X)\right)^2\right] \\
&= \E[\langle X, \frac{\wls - \wst}{4} \rangle^2 ]  \\
&= \|\wls - \wst\|_{\frac{1}{16}\E[XX]^{\top}}^2,
\end{align*} where we let $\|x\|_{\Sigma}^2 := x^\top \Sigma x$. Now, conditioning on $\{X_1,\dots,X_n\}$, and let $\widehat{\Sigma} := \frac{1}{n}\sum_{i=1}^n X_iX_i^\top$. 
Observe that $\E[Y\mid X_i] = \langle \wst, \frac{X_i}{4} \rangle$, and $Y_i - \E[Y\mid X_i]$ are independent random variables in $[0,2]$, so are $1$-subgaussian by Hoeffding's inequality. Hence, \citet[Proposition 1]{hsu2011analysis} with $\sigma^2 = 1$ and $d = 2$ implies that
\begin{align*}
&\Pr\left[\|\wls - \wls\|_{\frac{1}{16}\widehat{\Sigma}}^2  \le \frac{4 + 3\log(1/\delta)}{n} \mid \{X_1,\dots,X_n\} \right]\\ 
\overset{(i)}{\ge}& \Pr\left[\|\wls - \wls\|_{\frac{1}{16}\widehat{\Sigma}}^2  \le \frac{2 + 2\sqrt{2\log(1/\delta)} + 2\log(1/\delta)}{n} \mid \{X_1,\dots,X_n\} \right] \ge 1 - \delta.
\end{align*}
where $(i)$ uses the elementary inequality $ab \le \frac{a^2+b^2}{2}$. Lastly, we note that $\E[XX^\top] = \frac{1}{2}I$, so on the event $\lambda_{\min}(\widehat{\Sigma}) \ge \frac{1}{4}$, we have  $\left\|\wls - \wst\right\|_{\frac{1}{16}\E[XX^{\top}]}^2 \le \frac{1}{2}\|\wls - \wst\|_{\frac{1}{16}\widehat{\Sigma}}^2$. To this end, define $M_i = \E[XX^\top] - X_iX_i^\top = \frac{1}{2} I - X_iX_i^\top$. Note that $\lambda_{\max}(M_i) \le \frac{1}{2}$ and $\E[M_i^2] = \frac{1}{4} I$. Hence, by the Matrix Bernstein inequality~\citet[Theorem 6.6.1]{tropp2015introduction}, we have 
\begin{align*}
\Pr\left[ \lambda_{\min}(\widehat{\Sigma}) \le \frac{1}{4}\right] &= \Pr\left[ \lambda_{\max}(\sum_{i=1}^nM_i) \ge \frac{n}{4}\right] \le 2e^{ - \frac{t^2/2}{(n/4) + (t/6)}}\big{|}_{t = \tfrac{n}{4}} = 2e^{ - \frac{n}{8 + 4/3}}.
\end{align*} 
Putting pieces together, we conclude that
\begin{align*}
\Pr\left[\E\left[\left(f_{\wls}(A)(X) - f_{\wst}(X)\right)^2\right] \le \frac{8 + 6\log(1/\delta)}{n}\right] \ge 1 - \delta - 2e^{ - \frac{n}{8 + 4/3}}.
\end{align*}

\subsection{Proof of Information Theoretic Bound, Proposition~\ref{prop:info} \label{proof:prop:info}}

Let $R_{\psi}: \R^2 \to \R^2$ denote the linear operator corresponding to rotation by an angle $\psi \in [0,2\pi]$. 
Our strategy will be to show that for any $w \in \sphereone$, for 
\begin{align}\label{eq:eps_of_n_def}
\epsilon(n) = \sqrt{\min\left\{\frac{1}{2},\frac{3\log (1/\delta)}{n}\right\}},
\end{align} 
and some angle $\psi= \psi(n)$ depending on $n$, we have
\begin{align}\label{eq:Phi_bound}
\frac{1}{2}\left(\Pr[\Phi(\thetahat;\psi) \le \epsilon(n)^2] + \Pr[\Phi(\thetahat;R_{\psi}\psi)] \le \epsilon(n)^2\right) \le 1- \frac{\delta}{4}
\end{align}
Indeed, if~\eqref{eq:Phi_bound} holds, then we may express can express $\psi = R_{\phi} e_1$ where $e_1 = (1,0)$ for some $\phi \in [0,2\pi]$. Thus, taking an expectation over $\phi \unifsim [0,2\pi]$, we  observe that
\begin{align*}
\E_{\phi\unifsim [0,2\pi]}\Pr\left[\Phi(\thetahat;R_{\phi}e_1) \le \epsilon(n)^2\right] = \E_{\wst \unifsim \sphereone}\Pr\left[\Phi(\thetahat;\wst) \le \epsilon(n)^2\right],
\end{align*}
and similarly
\begin{align*}
\E_{\phi\unifsim [0,2\pi]}\Pr\left[\Phi(\thetahat;R_{\psi} R_{\phi}e_1) \le \epsilon(n)\right] &=\E_{\phi\unifsim [0,2\pi]}\Pr\left[\Phi(\thetahat;R_{\phi + \psi}e_1) \le \epsilon(n)\right]\\
&=\E_{\widetilde{\phi}\unifsim [0,2\pi]}\Pr\left[\Phi(\thetahat;R_{\widetilde{\phi}}w) \le \epsilon(n)^2\right]. 
\end{align*} 
Hence, 
\begin{align*}
1 - \frac{\delta}{4} &\ge \frac{1}{2}\E_{\phi\unifsim [0,2\pi]}\left[\Pr\left[\Phi(\thetahat;R_{\phi}w) \le \epsilon(n)^2\right] + \Pr\left[\Phi(\thetahat;R_{\psi}R_{\phi}e_1)\right] \le \epsilon(n)^2 \right]\\
&= \frac{1}{2} \cdot 2\E_{\wst \unifsim \calS_1}\Pr\left[\Phi(\thetahat;\wst) \le \epsilon(n)^2\right], \quad \text{as needed}.
\end{align*}

We now turn to proving~\eqref{eq:Phi_bound}. By rotation invariance argument, it suffices to prove the inequality for $w = e_1 = (1,0)$. We now fix an $\epsilon = \epsilon(n)$ as in~\eqref{eq:eps_of_n_def} to be chosen later, and choose $\psi = \arccos (1 - 2\epsilon^2)$. Note that $\epsilon \in (0,\frac{1}{2}]$ implies $\psi \in (0,\pi/2]$.

We construct two alternative instances $\theta^{(1)} = e_1$, and let $\theta^{(2)} = e_1 \cos \psi + e_2 \sin \psi$. We will establish a lower bound on the problem of testing between $\theta = \theta^{(1)}$ and $\theta = \theta^{(2)}$, and then translate this into a bound on $\Phi(\thetahat;\cdot)$. The first step is a $\KL$-divergence computation established in Section~\ref{proof:lem:kl_lem}:
\begin{lemma}\label{lem:kl_lem} There exists a constant $K > 0$ such that, if $(\calD_{\theta})^{\otimes n}$ denote the distribution of $n$ i.i.d. samples from $\calD_{\theta}$, then
\begin{align*}
\KL((\calD_{\thetaone})^{\otimes n},(\calD_{\thetatwo})^{\otimes n}) \le \frac{n}{12}\|\thetaone - \thetatwo\|_2^2.
\end{align*}
\end{lemma}
In our setting, we see that 
\begin{align*}
\|\thetaone - \thetatwo\|_2^2 = (1- \cos \psi)^2 + \sin^2 \psi = 1 + \cos^2\psi + \sign^2 \psi - 2\cos \psi = 2(1 - \cos \psi) = 4\epsilon^2.
\end{align*}
Hence, we have that $\KL((\calD_{\thetaone})^{\otimes n},(\calD_{\thetatwo})^{\otimes n}) \le n\frac{\epsilon^2}{3}$. Therefore, given any estimator $\ihat$ of $i$, the proof of [Theorem 2.2.iii in~\cite{Tsybakov_2008}] reveals that
\begin{align*}
 \frac{1}{2}\sum_{i \in \{1,2\}}\Pr_{(\calD_{\thetai})^{\otimes n}}\left[\{\ihat \ne i\}\right] \ge \frac{1}{4}e^{- \frac{n \epsilon^2}{3} }~.
 \end{align*}
 In particular, since $\epsilon = \epsilon(n)^2 \le \frac{3 \log (1/\delta)}{n}$, as in~\eqref{eq:eps_of_n_def}, and considering the complement of $\{\ihat \ne i\}$, we have 
 \begin{align}\label{eq:Pinsker}
 \frac{1}{2}\sum_{i \in \{1,2\}}\Pr_{(\calD_{\thetai})^{\otimes n}}\left[\{\ihat = i\}\right] \le 1 - \frac{1}{4} e^{-\log (1/\delta)} = 1 - \frac{\delta}{4}
 \end{align}
 Lastly, we show how a small value of $\Phi(\thetahat;\thetai)$ yields an accurate estimator of $\ihat$. Given an estimator $\thetahat$, we define the estimator of $\thetai$ give $\theta_{\ihat}$, where $\ihat$ is given by:
\begin{align*}
\ihat \in \arg\min_{i \in \{1,2\}} \Phi(\thetai;\thetahat),
\end{align*}
where we arbitrarily choose $\ihat  =1 $ if both values of $i$ attain the same value in the display above. The following lemma, proved in Section~\ref{proof:lem_R_lem}, shows that gives a reduction from estimating $i$ to obtaining a small value of $\Phi(\thetai,\thetahat)$:
\begin{lemma}\label{lem:R_lem} For $\psi \in [0,\frac{\pi}{2}]$, $\Phi(\thetai,\thetahat) < \sin^2 \frac{\psi}{2}$ implies that $\ihat = i$.
\end{lemma}
Hence, combining Lemma~\ref{lem:R_lem} with~\eqref{eq:Pinsker}, we have
\begin{align*}
\frac{1}{2}\sum_{i \in \{1,2\}} \Pr[ \Phi(\thetai,\thetahat) < \sin \frac{\psi}{2}] \le 1 - \frac{\delta}{4}
\end{align*}
Lastly, we find that as $\psi \in (0,\frac{\pi}{2})$, $\sin \frac{\psi}{2} = \sqrt{ \frac{1 - \cos \psi}{2}} = \sqrt{ \frac{2\epsilon^2}{2}} = \epsilon$, thereby concluding the proof.

\subsubsection{Proof of Lemma~\ref{lem:kl_lem}\label{proof:lem:kl_lem}}
By the tensorization of the $\KL$-divergence,
\begin{align*}
\KL\left((\calD_{\thetaone})^{\otimes n},(\calD_{\thetatwo})^{\otimes n}\right) &= n \KL\left(\calD_{\thetaone}, \calD_{\thetatwo}\right)\\
&= n \E_{X \unifsim \sphereone}\KL\left(\Bern(f_{\thetaone}(X)),\Bern(f_{\thetatwo}(X))\right),
\end{align*}
Now we use a standard Taylor-expansion upper bound on the $\KL$-divergence between two Bernoulli random variables (see, e.g. Lemma E.1 in \cite{simchowitz2016best}):
\begin{lemma} Let $p,q \in (0,1)$. Then,
\begin{align*}
\KL\left(\Bern(p),\Bern(q)\right) \le \frac{(p-q)^2}{2\min\{p(1-p),q(1-q)\}}.
\end{align*}
\end{lemma}

In our setting,
\begin{align*}
f_{\thetai}(X) = \frac{1}{2} + \frac{\langle \thetai, X \rangle }{4} \in \left[\frac{1}{4},\frac{3}{4}\right] \quad \text{ because } |\langle X, \thetai \rangle | \le \frac{\|\thetai\|\|x\|}{4} = \frac{1}{4}.
\end{align*}
Hence, since $2\min_{p \in [1/4,3/4]} p(1-p) = 2\cdot \frac{1}{4} \cdot \frac{3}{4} = 3/8$, 
\begin{align*}
\KL\left(\Bern(f_{\thetaone}(X)),\Bern(f_{\thetatwo}(X))\right) \le \frac{8\left\|f_{w}(X) - f_{w'}(X))\right\|_2^2}{3}
\end{align*}
Hence, 
\begin{align*}
KL\left((\calD_{\thetaone})^{\otimes n},(\calD_{\thetatwo})^{\otimes n}\right) &\le n \cdot \frac{8}{3} \E_{X \unifsim \sphereone}\left\|f_{\thetaone}(X) - f_{\thetatwo}(X))\right\|_2^2\\
&= \frac{8n}{3} \E_{X \unifsim \sphereone} \left\| \left\langle \frac{\thetaone - \thetatwo}{4}, X\right\rangle\right\|_2^2 \\
&= \frac{n}{6} (\thetaone - \thetatwo)^{\top}\E_{X \sim \sphereone}[XX^\top](\thetaone - \thetatwo) = \frac{n}{12}\|\thetaone - \thetatwo\|_2^2.
\end{align*}

\subsubsection{Proof of Lemma~\ref{lem:R_lem}\label{proof:lem_R_lem}}
By assumption, we have that $\Phi(\thetai,\thetahat) < \sin^2 \frac{\psi}{2}$ for some $\psi \in [0,\frac{\pi}{2}]$. Since $\psi \le \frac{\pi}{2}$, $\sin^2 \frac{\psi}{2} < 1$, so $\Phi(\thetai;\thetahat) < 1$, ruling out the case $\thetahat = 0$.
Thus, we may write can write 
\begin{align*}
\frac{\thetahat}{\|\thetahat\|} = e_1 \cos \phi + e_2 \sin \phi
\end{align*} for some $\phi \in [-\pi,\pi]$. Since $\thetahat \ne 0$, $\Phi(\thetai;\thetahat)$ corresponds to the square norm of the projection of $\thetai$ onto a direction perpendicular to $\thetahat$. Therefore,
\begin{align*}
\Phi(\thetaone;\thetahat) = \sin^2 \phi \text{ and } \Phi(\thetatwo;\thetahat) = \sin^2 (\phi - \psi).
\end{align*}

We shall show that if $\Phi(\thetaone;\thetahat) < \sin^2 \frac{\psi}{2}$, then $\ihat = 1$; proving that $\Phi(\thetatwo;\thetahat) < \sin^2 \frac{\psi}{2}$ implies $\ihat = 2$ is analogous.
To this end, suppose that $\sin^2 \phi = \Phi(\thetaone;\thetahat) < \sin^2 \frac{\psi}{2}$. We consider three cases, and show that in each case, $\sin^2 (\phi - \psi) \ge \sin^2 \frac{\psi}{2}$.
\begin{enumerate}
\item  Case (a): $|\phi| < \frac{\psi}{2}$. 
Then, we have that $\psi - \phi \in (\frac{\psi}{2},\frac{3\psi}{2})$. Since $\psi \le \frac{\pi}{2}$, $\min_{\varphi \in [\frac{\psi}{2},\frac{3\psi}{2}]} \sin^2 \varphi =  \sin^2 \frac{\psi}{2}$, 
whence $\sin^2 \frac{\psi}{2} < \sin^2 (\phi - \psi)$.
\item Case (b.1): $ \phi \in ( \pi - \frac{\psi}{2}, \pi]$. 
Then, we have that $\phi - \psi \in (\pi - \frac{3}{2}\psi, \pi - \psi]$.  
Now, if $\phi  - \psi \in [\frac{\pi}{2},\pi - \psi)$, we have that $\sin^2 (\phi - \psi) \in [\sin^2 \psi, 1]$, so that $\sin^2 (\phi - \psi) \ge \sin^2 \frac{\psi}{2}$. 
On the other hand, if $\phi  - \psi \in [\pi - \frac{3 \psi}{2},\frac{\pi}{2}]$, then since $\psi \le \frac{\pi}{2}$, we have $\sin^2 (\phi - \psi] \in [\frac{\pi}{4},\pi/2]$.
Thus, $\sin^2 (\phi - \psi) \ge \sin^2 \frac{\pi}{4} \ge \sin^2 \frac{\psi}{2}$. 
\item Case (b.2):  $\phi \in [-\pi,-\pi + \frac{\psi}{2})$. Then, we have that $\phi - \psi \in [-\pi - \psi, \pi - \frac{\psi}{2})$, and the rest is similar to (b.1).
\end{enumerate}

\subsection{Proof of Calibration Error Computation, Lemma~\ref{lem:error_sphere}\label{proof:lem:error_sphere}}
Recall that our goal is to establish that
\begin{align*}
\Delstrong_{f_{\thetahat},\Dwst}(\Awa) \ge \Delweak_{f_{\thetahat},\Dwst}(\Awa) = \frac{\sqrt{\Phi(\thetahat;\wst)}}{2\pi}, \quad \text{where}~ \Phi(\thetahat;\wst) = \begin{cases}  1 - \langle \wst, \frac{\thetahat}{\|\thetahat\|}\rangle^2 & \thetahat \ne 0 \\
1 & \thetahat = 0 \end{cases}~.
\end{align*}
We begin with a more involved calculation to compute $\Delweak_f(A)$; we shall lower bound $\Delstrong_f(A)$ at the end of the section.

\textbf{Bound for $\sufgap_f(A)$}: We shall show that if either $\mathrm{span}(\{\wA,\thetahat\}) = \R^2$ or $\thetahat = 0$, then there is a unit vector $v \in \sphereone$ for which
\begin{align*}
\sufgap_{f_{\thetahat}, \Dwst}(\Awa) = \frac{\sqrt{\Phi(\thetahat;\wst)}}{4} \cdot \E_{X \unifsim \sphereone}[|\langle v, X \rangle|].
\end{align*}
This is enough to conclude the proof, since
\begin{align*}
\E_{X \unifsim \sphereone}[|\langle v, X \rangle|] &= \frac{1}{2\pi}\int_{0}^{2\pi}|\sin \psi|d\psi~=~ \frac{1}{\pi}\int_{0}^{\pi}\sin \psi d\psi = \frac{2}{\pi}.
\end{align*} 
First, suppose that $\thetahat \ne 0$. Choose an orthonormal basis $\{e_1,e_2\}$ so that $\thetahat = \|\thetahat\| e_1$. Then, we can write
\begin{align*}
X = X_1 e_1 + X_2 e_2~,
\end{align*} 
where $X_i = \langle X_i, e_i \rangle$. Then, letting $\wst = \theta[1] e_1 + \theta[2] e_2$, we see that 
\begin{align*}
\theta[2] = \sqrt{\Phi(\thetahat;\wst)},
\end{align*}
and we have
\begin{align*}
f_{\theta}(X) &= \langle \wst,X \rangle + \frac{1}{2} = \frac{1}{4}X_1 \theta[1] + \frac{1}{4}X_{2}\theta[2] + \frac{1}{2}.
\end{align*}
First, suppose that $\thetahat \ne 0$. Then, since $f_{\thetahat}(X) = \frac{1}{2} + \|\thetahat\| \cdot X_1$ is in bijection with $X_1$, and since 
\begin{align*}
\E[X_2 | X_1 ] =0 \text{ for } (X_1,X_2) \unifsim \sphereone,
\end{align*}
we have
\begin{align*}
\E[f_{\theta}(X) \mid f_{\thetahat}(X)] &=  \E[f_{\theta}(X) \mid X_1] ~=~ \frac{1}{2} + \frac{1}{4}X_1 \theta[1] + \frac{1}{4}\theta[2]\E[X_2 \mid X_1] ~= \frac{1}{2} + \frac{1}{4}X_1 \theta[1].
\end{align*}
Moreover, if $\wA$ and $w = \|w\| e_1$ are linearly independent, then since $X \in \calS_1$, $\Awa = \sign(\langle \wA, X \rangle)$ and $X_1 = \langle e_1, X\rangle$ uniquely determine $X$. Hence, 
\begin{align*}
\E[f_{\theta}(X) \mid f_{\thetahat}, \Awa] &=  \E[f_{\theta}(X) \mid X] ~=~ f_{\theta}(X)  = \frac{1}{2} + \frac{1}{4}X_1 \theta[1] + \frac{1}{4}\theta[2]X_2.
\end{align*}
Thus, 
\begin{align*}
\E[f_{\theta}(X) \mid f_w(X), A] - \E[f_{\theta}(X) \mid f_w(X)] = \frac{\theta[2]X_2}{4}
\end{align*}
Hence, we conclude that
\begin{align*}
\sufgap_f(A) &= \frac{|\theta[2]|}{4} \cdot \E[|X_2|] = \frac{\sqrt{\Phi(\thetahat;\wst)}}{4}\E[|\langle e_2, X \rangle|].
\end{align*}

We now address the edge-case $\thetahat = 0$.  Since $f_{\thetahat}(X) = 0$ for all $X$, $\E[Y \mid f_{\thetahat}] = \E[Y] = \frac{1}{2}$, and $\E[Y \mid f_{\thetahat},\Awa] = \E[Y \mid \Awa]$. To compute $\E[Y \mid \Awa]$, let $e_1 = \wA$, and let $e_2$ be such that $\{e_1,e_2\}$ form an orthonormal basis, and write $X = X_1 e_1 + X_2 e_2$. Then, 
\begin{align*}
\E[X \mid A] &= \E[X_1 \mid \sign(X_1) ] + \E[X_2 \mid \sign(X_1)] \\
&= \E[X_1 \mid \sign(X_1) ]  \tag*{(since $X_2 \perp \sign(X_1)$)}\\
&= \sign(X_1)\E[\sign(X_1) X_1 \mid \sign(X_1)] \\
&= \sign(X_1)\E[|X_1| \mid \sign(X_1)] ~=~ \sign(X_1)\E[|X_1|]. 
\end{align*}
Hence, $\E[Y \mid A] = \frac{1}{2} + \frac{1}{4}\langle \wst, \E[X\mid A] \rangle = \frac{\sign(X_1)\E[|X_1|]}{4}$. Therefore, 
\begin{align}
\sufgap_{f_{\thetahat},\Dwst}(\Awa) &= \E\left|\E[f_{\theta}(X) \mid f_{\thetahat}(X), \Awa] - \E[f_{\theta}(X) \mid f_{\thetahat}(X)]\right| \nonumber \\
&= \E\left|\frac{1}{2} + \frac{\sign(X_1)\E[|X_1|]}{4} - \frac{1}{2}\right|  = \frac{1}{4}\E[|X_1|] = \frac{\sqrt{\Phi(\thetahat;\wst)}}{4}\E[|\langle \wA, X \rangle|] \label{eq:thetahat_zero}, 
\end{align}
where we recall the convention $\Phi(\thetahat;\wst) = 1$ if $\thetahat = 0$.

\textbf{Bound for $\calgap$}: 
First consider a function $f = f_{\thetahat}$ for some $\thetahat \in \R^2$. Suppose first that $\thetahat \ne 0$. As established above, $\E[Y | A_w, f_{\thetahat} ] = f_{\wst}$ almost surely, so we have we have
\begin{align*}
\Delstrong_{f_{\thetahat},\Dwst}(A_w) &= \E_{X \sim \sphereone}|f_{\thetahat}(X) - \E[Y|f_{\thetahat}(X),A_w]|\\
&= \E_{X \sim \sphereone}|f_{\thetahat}(X) - f_{\wst}(X)| \\
&= \E_{X \sim \sphereone}|\frac{1}{4}\langle \thetahat- \wst, X \rangle|\\
&\overset{(i)}{=} \frac{\|\thetahat - \wst\|_2}{4} \E_{X \sim \sphereone} |\langle e_1, X \rangle|\\
&\overset{(ii)}{=} \frac{1}{2\pi}\|\thetahat - \wst\|_2,
\end{align*}
where $(i)$ uses the fact that $X$ has a rotation invariant distribution, and $(ii)$ uses the computation $\E_{X \sim \sphereone} |\langle e_1, X \rangle| = 2/\pi$ performed above. To let $(e_1,e_2)$ be an orthonormal basis for $\R^2$ for which $\thetahat = \|\thetahat\|e_1$, and let $w_* = w_{1,*}e_1 + w_{2,*}e_2$ as above. Then
\begin{align*}
\|\thetahat - \wst\|_2 = \sqrt{(\|w\|_2 - w_{1,*})^2 + w_{2,*}^2  } \ge w_{2,*} = \sqrt{\Phi(\thetahat,\wst)},
\end{align*}
as needed. 

If $\thetahat = 0$, then as show above $\E[Y \mid f_{\thetahat},\Awa] = \E[Y \mid \Awa] = \frac{\sign(X_1)\E[|X_1|]}{4}$, and $f_{\thetahat}(X) = \frac{1}{2}$. Hence, 
\begin{align*}
\Delstrong_{f_{\thetahat},\Dwst}(A_w) &=  \E_{X \sim \sphereone}|f_{\thetahat}(X) - \E[Y|f_{\thetahat}(X),A_w]|\\
&=  \E_{X \sim \sphereone}|\frac{1}{2} - \frac{\sign(X_1)\E[|X_1|]}{4}|\\
&= \frac{\sqrt{\Phi(\thetahat;\wst)}}{4}\E[|\langle \wA, X \rangle|] \quad (\text{by}~\eqref{eq:thetahat_zero}),\\
\end{align*}
which is equal to $\frac{\sqrt{\Phi(\thetahat;\wst)}}{2\pi}$ by the computation of $\E[|\langle \wA, X \rangle|]$ above.

\section{Supplementary Material for the Per-group Sufficiency and Calibration Lower Bound\label{app:lower_bound_imbalance}}

\newcommand{\Echer}{\mathcal{E}_{\mathsf{Cher}}}

\newcommand{\na}{n_{\alpha}}
\newcommand{\nb}{n_{\beta}}
\newcommand{\calSa}{\calS_{\alpha}}
\newcommand{\calSb}{\calS_{\beta}}
\newcommand{\calAbar}{\overline{\calA}}

\newcommand{\Xbar}{\overline{X}}
\newcommand{\Xbara}{\overline{X}^{\alpha}}
\newcommand{\Xbarb}{\overline{X}^{\beta}}
\newcommand{\wbarst}{\overline{\theta}}
\newcommand{\wbarhat}{\widehat{\overline{\theta}}}
\newcommand{\wbarhata}{\wbarhat^{\alpha}}
\newcommand{\wbarhatb}{\wbarhat^{\beta}}
\newcommand{\calLbar}{\overline{\calL}}
\newcommand{\calFbar}{\overline{\calF}}

\newcommand{\wbarsta}{\theta^{\alpha}}
\newcommand{\wbarstb}{\theta^{\beta}}
\newcommand{\Dx}{\overline{\mathcal{D}}_{\Xbar}}
\newcommand{\Dwbarst}{\overline{\mathcal{D}}_{\wbarst}}
\newcommand{\Dwbarsta}{\mathcal{D}_{\wbarsta}}

\newcommand{\wbara}{w^{\alpha}}
\newcommand{\wbarb}{w^{\beta}}
\newcommand{\wbarls}{\widehat{\wbar}_{\mathrm{LS}}}
\newcommand{\wbarlsa}{\theta^{\alpha}}
\newcommand{\wbarlsb}{\theta^{\beta}}
\newcommand{\fbar}{\overline{f}}
\newcommand{\Xa}{X^{\alpha}}
\newcommand{\Xb}{X^{\beta}}

This section contains a formal analogue of Theorem~\ref{thm:imbalance_lb_body}. Again, we begin with a construction of the joint distribution over features, attributes and labels  before stating the precise result.
\subsection{Construction: } We construct a ``product'' of two independent instances of  Theorem~\ref{thm:main_lb_constants}. We will put ``bars'' over quantities related to the product distribution, function class, etc.., and use $\alpha$ and $\beta$ to denote each component of the product. 

Let $\calD_{\cdot}$ denote the distribution from the construction in Theorem~\ref{thm:main_lb_constants}. We will draw $\wbarsta,\wbarstb$ independently from $\sphereone$. We let $\wbarst = (\wbarsta,\wbarstb)$, and define $(\Xbar,Y,Z) \sim \Dwbarst$ as follows:
\begin{enumerate}
	\item Let $Z \sim \Bern(p)$.
	\item If $Z = 1$, draw $(\Xa,Y) \sim \calD_{\wbarsta}$, otherwise draw $(\Xb,Y) \sim \calD_{\wbarstb}$.
	\item Let $\Xbar = (\Xa,0) \in \R^{4}$ if $Z = 1$; otherwise set $\Xbar = (0,\Xb) \in \R^{4}$.
\end{enumerate}
Note that $Z$ can be determined from $\Xbar$ by looking at which coordinate is $0$. Further we define:
\begin{enumerate}
\item $\fbar_{\wbarhat}(\Xbar) = \frac{1}{2} + \frac{1}{4}\langle \wbarhat, \Xbar \rangle$ for $\Xbar,\wbarhat \in \R^4$, and $\calFbar := \{\fbar_{\wbarhat} : \wbarhat \in \R^4\}$. 
\item The loss function $\calLbar_{\wbarst}(\fbar) = \E_{(\Xbar,Y) \sim \Dwbarst)}((\fbar(X) - \E[Y | \Xbar])^2 = \E_{(\Xbar,Y) \sim \Dwbarst)}((\fbar(X) - \fbar_{\wbarst})^2$. 
\item We let $\calL_*$ denote the Bayes loss $\calLbar_{\wbarst}(\fbar_{\wbarst})$, which we note does not depend on $\wbarst$.
\end{enumerate} 
Lastly for $\wbar = (\wbara,\wbarb) \in \sphereone\times\sphereone$, we define our discrete attribute $\calAbar_{\wbar}(\Xbar)$ which takes $4$ values. 
\begin{align*}
\calAbar_{\wbar}(\Xbar) = (Z, \sign(\langle \wbar,\Xbar \rangle)) \in \{0,1\}\times \{-1,1\}.
\end{align*}
Here, we we note that with probability $1$, $\sign(\langle \wbar,\Xbar \rangle) = Z\sign(\langle \wbara,\Xbara \rangle) + (1-Z) \sign(\langle \wbarb,\Xbarb \rangle) \ne 0$. In the statement of the theorem, we map $\calAbar_{\wbar}(\Xbar) \to \{1,\dots,4\}$ via the bijection $(Z,\sign(\langle \wbar,\Xbar \rangle)) \mapsto \frac{1 + \sign(\langle \wbar,\Xbar \rangle))}{2} + 2(1-Z)$, which maps the $Z = 1$ attributes to $\{1,2\}$.

For now, using the $\{0,1\}\times \{-1,1\}$ attributes will be more transparent. Lastly, we see that for attributes $a \in \{(0,-1),(0,1)\}$, and any classifier $\fbar_{\wbarhat}$ that
\begin{align*}
&\frac{1}{2}\left(\sufgap_{\fbar_{\wbarhat};\Dwbarst}((0,-1);\calAbar_{\wbar} ) + \sufgap_{\fbar_{\wbarhat};\Dwbarst}(\A_{\wbar}((0,-1);\Xbar) \right) \\
&= \E[\left|\E[Y | \fbar_{\wbarhat} ] - \E[Y | \fbar_{\wbarhat}, \sign(\langle \wbar,\Xbar \rangle) \right| \mid Z = 1]\\
&= \E[\left|\E[Y | f_{\wbarhata} ] - \E[Y | f_{\wbarhata}, \sign(\langle \wbara,\Xa \rangle) \right| \mid Z = 1]\\
&= \sufgap_{f_{\wbarhata};\Dwbarsta}(A_{\wbara}),
\end{align*}
that is, the calibration term from Theorem~\ref{thm:main_lb_constants}. Hence, by Lemma~\ref{lem:error_sphere} in the proof of Theorem~\ref{thm:main_lb_constants}, we find that 
\begin{align}\label{eq:suf_Phi_two_group}
 \max_{a \in \{(0,-1),(0,1)\}} \sufgap_{\fbar_{\wbarhat};\Dwbarst}(a;\calAbar_{\wbar} ) \ge \sufgap_{f_{\wbarhata};\Dwbarsta}(A_{\wbara}) = \frac{\sqrt{\Phi(\wbarhata;\wbarsta)}}{2\pi}.
\end{align}
A similar argument shows that
\begin{align}\label{eq:cal_Phi_two_group}
 \max_{a \in \{(0,-1),(0,1)\}} \calgap_{\fbar_{\wbarhat};\Dwbarst}(a;\calAbar_{\wbar} ) \ge \calgap_{f_{\wbarhata};\Dwbarsta}(A_{\wbara}) = \frac{\sqrt{\Phi(\wbarhata;\wbarsta)}}{2\pi}.
\end{align}

\subsection{Statement of the Exact Theorem}
We now state the technical version of Theorem~\ref{thm:lower_bound_imbalance}; we remark that the $p$ in the theorem as stated previously corresponds to $p/2$ in the following statement:
\begin{theorem}\label{thm:lower_bound_imbalance} Fix any $p \in (0,1/2)$, and let $\calFbar$, $A_{\wbar}$, $\Dwbarst$ be as above, indexed by $\wbar, \wbarst \in \sphereone \times \sphereone$. Further,  write $\wbar, \wbarst \sim \sphereone \times \sphereone$ to denote that $\wbar$ and $\wbarst$ are drawn independently from the uniform distribution on $\sphereone \times \sphereone$. 
Then, for any $\wbar \in \sphereone \times \sphereone$, $\Pr[A_{\wbar} = 1] = \Pr[A_{\wbar} = 2] = p/2$ and

\begin{itemize}
\item[(a)] For any classifier $\fhat \in \calFbar$ trained on a sample $\samplen := \{(X_i,Y_i)\}_{i =1}^n$,  any $\delta \in (0,1)$,
{
	\begin{multline}\label{eq:delta_lb_app}
\E_{\wbarst,\wbar \sim \sphereone \times \sphereone}\Pr_{\samplen \sim \Dwbarst}\left[\max_{a \in \{1,2\}}\min\{\calgap_{\fhat;\Dwbarst}(a;A_{\wbar}),\sufgap_{\fhat;\Dwbarst}(a;W_{\wbar}) \} 
\le  c_1\min\left\{1, \sqrt{\frac{\log(1/\delta_1)}{ p n}}\right\} \right]\\  \le 1 - c_0 \delta.
\end{multline}
}
\item[(b)] Let $\fls$ denote the ERM under the square loss
\begin{align*}
\fls := \arg\min_{f \in \calF}\sum_{(X_i,Y_i) \in \samplen}(f(X_i) - Y_i)^2.
\end{align*}
Then~\eqref{eq:delta_lb_app} holds even when $\E_{\wbarst,\wbar \unifsim \sphereone\times \sphereone}$ is replaced by a supremum $\sup_{\wbarst,\wbar \in \sphereone \times \sphereone}$. Moreover, for any $\delta_2 \in (0,1/4)$ and $\wbarst \in \sphereone \times \sphereone$,
	\begin{align*}
\Pr_{\samplen \sim \Dwbarst}\left[\calL_{\Dwbarst}(\fls) - \calL^*  \le c_2\frac{\log(1/\delta_2)}{n}\right] \ge 1 - \delta_2 - 4e^{ -  c_3 p n},
\end{align*} 
where $\calL_{\Dwbarst}(f) = \E_{(X,Y) \sim \Dwbarst}[(f(X)-Y)^2]$ denotes population risk under the square loss, and where $\calL^* = \calL_{\Dwbarst}(\overline{f}_{\wbarst}) $ denotes the (calibrated) Bayes risk.
\end{itemize}
\end{theorem}

\subsection{Proof of Theorem~\ref{thm:lower_bound_imbalance}}
\textbf{Learning Setup:} 
Let $\calS = \{(\Xbar_i,Y_i,Z_i)\}_{i=1}^n$ be a sample drawn $i.i.d$ from $\Dwst$, and define the subsamples $\calSa := \{(\Xbara_i,Y_i) : Z_i = 1\}$ and $\calSb = \{(\Xbarb_i,Y_i) : Z_i = 0\}$, and sample numbers $\na := |\calSa|$ and $\nb:= |\calS_b|$. Then conditional on $\na$ (or equivalently, on $\nb$), $\calS_a$ has the distribution of a sample of size $\na$ from $\calD_{\wbarsta}$, where $\calD_{w}$ is the distribution from the $2$-group lower bound, Theorem~\ref{thm:main_lb_constants}, and is independent of $\calSb$. Lastly, we define the ``Chernoff' 'event 
\begin{align*}
\Echer := \{\na \ge \frac{1}{2}pn~\text{and}~\nb \ge \frac{1}{2}(1-p)n\}.
\end{align*}
And we note that $\Pr[\Echer] \ge 1 - 2e^{ - \Omega(\min\{pn,(1-pn)\})} \ge 1 - 2e^{-c_1 p n}$ for some constant $c_1$ by combining Chernoff bounds for $\na$ and $\nb$. We now can prove part 1 and part 2 of the proposition separately.

\textbf{Part 1: Information Theoretic Lower Bound:}
Let's condition on $\na$. Note then that $\calSb$ contains no information about $\wbarsta$, since the prior on $\wbarsta$ and $\wbarstb$ are independent. Thus, the information theoretic lower bound, Proposition~\ref{prop:info}, implies we have that for the $\alpha$-component of our estimator, $\wbarhata$, must satisfy the bound:
\begin{align*}
\E_{\wbarsta \unifsim \sphereone}\Pr_{\calS_a \sim \Dwst}\left[\Phi(\wbarhata;\wbarsta) \le \min\left\{\frac{1}{2}, \frac{3\log (1/\delta)}{\na}\right\} \mid \na \right] \le 1- \frac{\delta}{4}.
\end{align*}
Thus, using that  $\Pr[\Echer] \ge 1 - 2e^{-c_1 p n}$,
\begin{align*}
\E_{\wbarsta \unifsim \sphereone}\Pr_{\calS_a \sim \Dwst}\left[\Phi(\wbarhata;\wbarsta) \le \min\left\{\frac{1}{2}, \frac{3\log (1/\delta)}{2pn}\right\} \right] \le 1- \frac{\delta}{4} -  2e^{- c_1 p n}.
\end{align*}
By considering the cases $\delta \le c_1 pn$ and $d > c_1 pn$ separately, some algebraic manipulations reveal that there exists a constant $c,c'$ such that
\begin{align*}
\E_{\wbarsta \unifsim \sphereone}\Pr_{\calS_a \sim \Dwst}\left[\Phi(\wbarhata;\wbarsta) \le c \min\left\{1, \frac{\log (1/\delta)}{pn}\right\} \right] \le 1- c' \delta.
\end{align*}
Thus, by Equations~\eqref{eq:cal_Phi_two_group} and~\eqref{eq:suf_Phi_two_group}, we have
\begin{align*}
\Pr\left[\max_{a \in \{(0,-1),(0,1)\}} \min\{\sufgap_{\fbar_{\wbarhat};\Dwbarst}(a;\calAbar_{\wbar}),  \calgap_{\fbar_{\wbarhat};\Dwbarst}(a;\calAbar_{\wbar})\}  \le  c \min\left\{1, \frac{\log (1/\delta)}{pn}\right\} \right].
\end{align*}

\textbf{Part 2: Analysis of Least Squares Estimator }
As in the proof of Theorem~\ref{thm:main_lb_constants}, we see that the least squares estimator $\fls$ takes the form $\fbar_{\wbarls}$. 
\begin{align*}
\calLbar_{\wbarst}(\fbar_{\wbarls}) - \calL^* &= \E\|\wbarls - \wbarst\|_{\frac{1}{16}\E[\Xbar \Xbar^{\top}]}^2,
\end{align*}
where again let $\|x\|_{\Sigma}^2 := x^\top \Sigma x$. Let $X$ denote a random variable distributed uniformly on the sphere. By breaking the least squares estimator into components $\wbarls = (\wbarlsa,\wbarlsb)$, we see that 
\begin{enumerate}
	\item Since $\Xbar$ is either supported on the $\alpha$ component or the $\beta$ component, $\wbarlsa$ and $\wbarlsb$ are the least-squares estimates on $\calSa,\calSb$, respectively 
	\item Computing $
\E[\Xbar \Xbar^{\top}] = \begin{bmatrix} p \E[XX^\top] & 0 \\
0 & (1-p)\E[XX^\top] \end{bmatrix}$,  we see that
\begin{align*}
\calLbar_{\wbarst}(\fbar_{\wbarls}) - \calL^* = p\E\|\wbarlsa - \wbarsta\|_{\frac{1}{16}\E[XX^\top]}^2  + (1-p)\E\|\wbarlsb - \wbarstb\|_{\frac{1}{16}\E[XX^\top]}^2
\end{align*}
\end{enumerate}
With these two points in hand, we can use the analysis of the least squares estimator from Theorem~\ref{thm:main_lb_constants}, conditioning on $\na$ and $\nb$, to find that with probability $1 - 2 \delta - 2\exp( - c_3\min\{\na,\nb\})$ that
\begin{align*}
\calLbar_{\wbarst}(\fbar_{\wbarls}) - \calL^*  \le c_2\left(\frac{p}{\na} + \frac{(1-p)}{\nb}\right) \log(1/\delta).
\end{align*} 
In particular, when $\Echer$ holds, we have we have that with probability at least $1 - 2 \delta - 2\exp( - \frac{c_3}{2} \min\{pn,(1-p)n\}) = $, we have 
\begin{align*}
\calLbar_{\wbarst}(\fbar_{\wbarls}) - \calL^*  \le c_2\left(\frac{p}{(pn/2)} + \frac{(1-p)}{((1-p)n/2)}\right) \log(1/\delta) = \frac{4 c_2}{n}\log(1/\delta). 
\end{align*} 
Finally, using the $\Pr[\Echer] \ge 1 - 2e^{-\Omega(pn)}$, we conclude that for a new constant $c_2 \leftarrow 4c_2$, and new constant $c_3$ that 
\begin{align*}
\calLbar_{\wbarst}(\fbar_{\wbarls}) - \calL^*  \le \frac{c_2 \log(1/\delta)}{n} \text{ with probability} 1 - 2\delta - 4e^{-c_3 pn}.
\end{align*}


\section{Lower Bounds for Separation gap\label{sec:sep_lb_sec}
}
 	Central to the proof is the following lemma, proved in Section~\ref{sec:sep_cond_exp} below: 
 	\begin{lemma}\label{lem:separation_cond_exp_comp} Define the quantities
 	\begin{align*} 	Z_A := \E[(f^B)^2\mid A], \quad q_A := \Pr[Y=1\mid A], \quad \Zbar := \E[(f^B)^2], \text{ and} \quad  \qbar := \Pr[Y = 1]
 	\end{align*} Then, the following equalities hold.
 	\begin{align*}
 	\E[f^B\mid Y=1,A]  &= \frac{Z_A}{q_A}, \quad \E[f^B\mid Y=0,A] = \frac{q_A - Z_A}{1 - q_A}, \\
 	\E[f^B\mid Y=1]  &= \frac{\Zbar}{\qbar}, \quad \E[f^B\mid Y=0] = \frac{\qbar - \Zbar}{1 - \qbar}.\\
 	\end{align*}
 	\end{lemma}

 	We are now ready to finish the proof. By Lemma~\ref{lem:separation_cond_exp_comp} with $q_A = \Pr[Y = 1 \mid A]$,
 	\begin{align*}
 	\Delsep_{f^B}(A) &= \E_A  \left( q_A \cdot \left|\frac{\Zbar }{\qbar} -  \frac{Z_A}{q_A} \right|  + (1-q_A) \left| \frac{q_A - Z_A}{1 - q_A}  - \frac{\qbar - \Zbar }{1- \qbar}\right| \right)\\
 	&= \E_A  \left(\left|\frac{\Zbar q_A}{\qbar} -  Z_A \right|  + \left| q_A - Z_A   - \left(\frac{1- q_A}{1- \qbar}\right) (\qbar - \Zbar )\right| \right)\\
 	&\ge \E_A  \left|\frac{\Zbar q_A}{\qbar} -  Z_A - \left( q_A - Z_A   - \left(\frac{1- q_A}{1- \qbar}\right) (\qbar - \Zbar ) \right) \right| \quad \text{(Reverse Triangle Inequality)}\\
 	&= \E_A  \left|\frac{\Zbar q_A}{\qbar} + \left(\frac{1- q_A}{1- \qbar}\right) (\qbar - \Zbar )- q_A  \right| \quad \text{(Cancelling }Z_A \text{)}\\
 	&= \E_A  \left|\Zbar  \left(\frac{q_A}{\qbar} - \frac{1- q_A}{1- \qbar}\right) + \left(\frac{1- q_A}{1- \qbar}\right)\qbar - q_A\right| \quad \text{(Grouping Terms)}.
 	\end{align*}
 	We further unpack 
 	\begin{align*}
 	&\E_A  \left|\Zbar  \left(\frac{q_A}{\qbar} - \frac{1- q_A}{1- \qbar}\right) + \left(\frac{1- q_A}{1- \qbar}\right)\qbar - q_A\right|\\
 	=&\E_A  \left|(\Zbar - \qbar) \left(\frac{q_A}{\qbar} - \frac{1- q_A}{1- \qbar}\right) \right|\\
 	=&\E_A  \left|(\Zbar - \qbar) \left(\frac{q_A(1- \qbar) - (1- q_A)\qbar}{\qbar(1- \qbar) } \right) \right|\\
 	=&  \frac{|\Zbar - \qbar|}{\qbar(1-\qbar)} \cdot \E_A |q_A - \qbar|,
 	\end{align*}
 	Lastly, we find that 
 	\begin{align*}
 	|\Zbar - \qbar| &= |\E[(f^B)^2] - \Pr[Y = 1]]~= |\E[(f^B)^2] - \E[f^B]] \\
 	 &= |\E[(f^B)(f^B - 1)]| ~= \E[(f^B)(1-f^B)] \quad \text{ since } f^B \in [0,1]\\
 	 &= \E[(\E[Y|X,A])(1-\E[Y|X,A])]~= \E[\Var[Y|X,A]].
 	\end{align*}
 	Moreover, $\qbar(1-\qbar) = \E[Y](1 - \E[Y]) = \Var[Y]$. Thus 
 	\begin{align*}
 	\Delsep_{f^B}(A) \ge \frac{|\Zbar - \qbar|}{\qbar(1-\qbar)} \cdot \E_A |q_A - \qbar| \ge \frac{\E[\Var[Y|X,A]]}{\Var[Y]} \cdot \E_A |q_A - \qbar|.
 	\end{align*}

 	\subsection{Proof of Lemma~\ref{lem:separation_cond_exp_comp}}\label{sec:sep_cond_exp}
	For ease of notation, we write $f = f^B$. Observe then that by the definition of the Bayes classifer,
	\begin{align*}
	\Pr[Y=1\mid f(X,A)=\tau, A] = \tau
	\end{align*}
	 First we compute the conditional densities by Bayes rule:
 	\begin{align*}
 	\Pr[f(X,A)=\tau \mid Y=1,A] &= \frac{\Pr[Y=1\mid f(X,A)=\tau, A]\Pr(f(X,A) = \tau\mid A)}{\Pr[Y=1\mid A]} \\
 		&=\frac{\tau\Pr(f(X,A) = \tau\mid A)}{\Pr[Y=1\mid A]} , \\
 		\Pr[f(X,A)=\tau \mid Y=0,A] &= \frac{\Pr[Y=0\mid f(X,A)=\tau, A]\Pr(f(X,A) = \tau\mid A)}{\Pr[Y=0\mid A]} \\
 		&=\frac{(1-\tau)\Pr(f(X,A) = \tau\mid A)}{\Pr[Y=0\mid A]}.  
 		\end{align*}
 Integrating these to compute the relevant expectations, we have
 		\begin{align*}
 	 	\E[f\mid Y=1,A] &= \int \tau \Pr[f(X,A)=\tau \mid Y=1,A] d\tau \\
 	 	&= \int \frac{\tau^2Y }{\Pr[Y=1\mid A]}  \Pr(f(X,A) = \tau\mid A) d\tau\\
 	 	&=\frac{\E[f^2\mid A] }{\Pr[Y=1\mid A]},\\
 		\E[f\mid Y=0,A] &= \int \tau \Pr[f(X,A)=\tau \mid Y=0,A] d\tau \\
 	&=\int \frac{\tau(1-\tau)}{\Pr[Y=0\mid A]} \Pr[f(X,A)=\tau \mid A] d\tau \\
 		&=\frac{\Pr[Y=1\mid A] -\E[f^2\mid A] }{ \Pr[Y=0\mid A]}, \tag*{since $\E[\fbayes \mid A]= \Pr[Y=1\mid A].$}
 	\end{align*}
 	
 	In summary, we have
 	\begin{align*}
 	\E[f\mid Y=1,A] &= \frac{\E[f^2 \mid A] }{\Pr[Y=1\mid A]} = \frac{Z_A}{q_A}.\\
 	\E[f\mid Y=0,A] &= \frac{(\Pr[Y=1\mid A] -\E[f^2\mid A] ) }{ \Pr[Y=0\mid A]} = \frac{q_A - Z_A}{1 - q_A}.
 	\end{align*}
 	Similarly,
 	\begin{align*}
 	\E[f\mid Y = 1] &= \frac{\E[f^2] }{\Pr[Y=1]} = \frac{\Zbar }{\qbar}.\\
 	\E[f\mid Y = 0] &= \frac{(\Pr[Y=1] -\E[f^2] ) }{ \Pr[Y=0]} = \frac{\qbar - \Zbar }{ 1- \qbar}.
 	\end{align*}

 	\subsection{Proof of Corollary~\ref{cor:erm_sep}\label{sec:cor:erm_sep}}
 	By the reverse triangle inequality, the definition of separation, and Jensen's inequality, we bound
 	\begin{align*}
 	\Delsep_f(A) &:= \E_{Y,A}[|\E[f(X,A)\mid Y,A] - \E[f(X,A)\mid Y] |]\\
 	&\ge  \E_{Y,A}[|\E[f^B(X)\mid Y,A] - \E[f^B(X)\mid Y] |] - \E_{Y,A}[|\E[f - f^B \mid Y,A] - \E[f - f^B(X)\mid Y] |] \\
 	&= \Delsep_{f^B}(A) - \E_{Y,A}[|\E[f^B(X)\mid Y,A] - \E[f^B(X)\mid Y] |] - \E_{Y,A}[|\E[f - f^B \mid Y,A]\\
 	&\ge  \Delsep_{f^B}(A) - 2\E_{Y,A}[|f - f^B|]. 
 	\end{align*}
 	By Proposition~\ref{proposition:sep_lb}, we have $\Delsep_{f^B} \ge C_{f_B} \E[|\overline{q} - q_A|]$. Moreover, we have 
 	\begin{align*}
 	\E_{Y,A}[|f - f^B|] \le \sqrt{\E_{Y,A}[|f - f^B|^2] } \le \sqrt{ \frac{\calL(f) - \calL(f^B)}{\kappa}},
 	\end{align*} 
 	where the first inequality is Jensen's inequality, and the second using $\kappa$-strong convexity of $\calL$ as in the proof of Theorem~\ref{thm:main_upper}.

\section{Lower Bound for Independence Gap\label{sec:sep_lb_sec}
}
	We define the independence gap of a score $f$ with respect to a group attribute $A$ as 
	\begin{equation}
	\pargap_{f}(A) = \E[|\E[f | A] - \E[f]|] 
	\end{equation}
	Note that $\pargap_{f}(A) = 0$ if and only if $f$ and $A$ are conditionally mean independent, which is implied by (but does not imply) statistical independence. 
	The following proposition shows that any score with a small excess risk must have a large independence gap if the base rate $\Pr[Y = 1]$ differs by group.
\begin{proposition}[Independence of Unconstrained Learning]  Let $\calL$ be the risk associated with a loss function $\ell(\cdot,\cdot)$ satisfying Assumption~\ref{asm} with parameter $\kappa > 0$. Denote $\qbar := \Pr[Y = 1]$, and  $q_A := \Pr[Y = 1 | A]$ for a group attribute $A$.  Then, for any score $\hat{f} \in \F$,
	
	\begin{align*}
	\pargap_{\hat{f}}(A) \geq \E[|q_A - \qbar|] - 2\sqrt{ \frac{\calL(\hat{f}) - \calL(f^B)}{\kappa}}
	\end{align*}

\end{proposition}

\begin{proof}

Observe that for the \bayscore~$f^B$, $\E[f^B |A ] = q_A := \Pr[Y = 1 | A]$, whereas $\E[f^B] = \qbar = \Pr[Y = 1]$. Hence
\begin{align}\label{eq:pargap_bayes}
\pargap_{f^B}(A) &= \E[|\E[f^B | A] - \E[f^B]|] = \E[|q_A - \qbar|].
\end{align}
We now lower bound the $\pargap_{f}(A)$ for arbitrary $f$.
The remainder of the proof follows along the lines of Corollary~\ref{cor:erm_sep}. 
By the reverse triangle inequality, the definition of separation, and Jensen's inequality, we bound
 	\begin{align*}
 	\pargap_f(A) &:= \E_{A}[|\E[f |A ] - \E[f] |]\\
 	&\ge  \E_{A}[|\E[f^B(X)\mid A] - \E[f^B(X)] |] - \E_{A}[|\E[f - f^B \mid A] - \E[f - f^B(X)] |] \\
 	&= \pargap_{f^B}(A) - \E_{A}[|\E[f - f^B \mid A] - \E[f - f^B(X)] |]\\
 	&\ge  \pargap_{f^B}(A) - 2\E_{Y,A}[|f - f^B|]. 
 	\end{align*}
 	By~\eqref{eq:pargap_bayes}, we have $\pargap_{f^B}(A) \ge \E[|\overline{q} - q_A|]$. Moreover, we have 
 	\begin{align*}
 	\E_{Y,A}[|f - f^B|] \le \sqrt{\E_{Y,A}[|f - f^B|^2] } \le \sqrt{ \frac{\calL(f) - \calL(f^B)}{\kappa}},
 	\end{align*} 
 	where the first inequality is Jensen's inequality, and the second using $\kappa$-strong convexity of $\calL$ as in the proof of Theorem~\ref{thm:main_upper}.
 	
\end{proof}

\section{Addendum to experiments}

\subsection{Empirical estimate of $\sufgap_f(A)$}\label{app:implementation_deets}

We estimate $\sufgap_f$ from the test set and the scores for this test set, that is $\{x_i, y_i, a_i, f(x_i)\}_{i=1}^n$. We divide the scores into deciles, that is, $B=10$ equally spaced intervals on $[0,1]$. For any score value $f \in [0,1]$, let $d(f)$ denote the corresponding decile for $f$. We estimate $\E[Y|f]$ as the average rate of positive outcomes in the corresponding decile of the score $f$, i.e. \[\hat{g}(f) = \frac{1}{N}\sum_{i=1}^n y_i \Ind{f(x_i)\in d(f)},\] where $N = \sum_{i=1}^n\Ind{f(x_i)\in d(f)}$. For any group $a$ and score $f$, we estimate $\approx \E[Y|f,A=a]$ as the average rate of positive outcomes in the corresponding decile of the score $f$ in group $a$, i.e. 
\begin{align*}
\hat{g}(f,a) = \frac{1}{M}\sum_{i=1}^n y_i \Ind{f(x_i)\in d(f) \cap a_i = a},
\end{align*}
where $M = \sum_{i=1}^n\Ind{f(x_i)\in d(f)\cap a_i = a}$.
$\widehat{\sufgap}_f$ is then computed as the sample average $\frac{1}{n} \sum_{i=1}^n |\hat{g}(f(x_i),a_i) - \hat{g}(f(x_i))|$. In general, the value of the estimate does vary with the chosen number of intervals $B$. We find that on the Adult dataset, for example, the choice $B=10$ results in bins that all have a suitable number of samples, and hence provides an adequate estimate of the expected sufficiency gap for experimental purposes. We leave the statistical properties of this estimator, such as the ramifications of different $B$, to future work.

\subsection{Broward dataset}\label{sec:compas_expt}

Compared to the Adult dataset, the Broward dataset contains about 6 times fewer training and testing examples; as expected, our estimates of the sufficiency gap are noisier. Figure \ref{fig:multi_calib_broward_eqbuckets} shows that the score obtained from empirical risk minimization with the logistic loss is largely calibrated with respect to gender, race and age across different scores, barring score buckets where there was grossly insufficient data to estimate the rate of positive outcomes.

In Figure \ref{fig:learningcurve_broward}, we show the calibration and separation error of the logistic regression model as we use more training examples. The average and standard deviation (indicated as confidence intervals) are computed over 20 random draws of training examples. To deal with insufficient data in certain score deciles, we instead estimated the sufficiency gap using 8 score buckets based on quantiles. For race, the sufficiency gap is decreasing with the number of samples, while the separation gap does not decrease, stabilizing at a value of 0.05. For gender, the sufficiency gap is relatively small to begin with (0.03) and appears to remain at the same level with more samples.


\begin{figure}[H]
		\includegraphics[width=0.33\columnwidth]{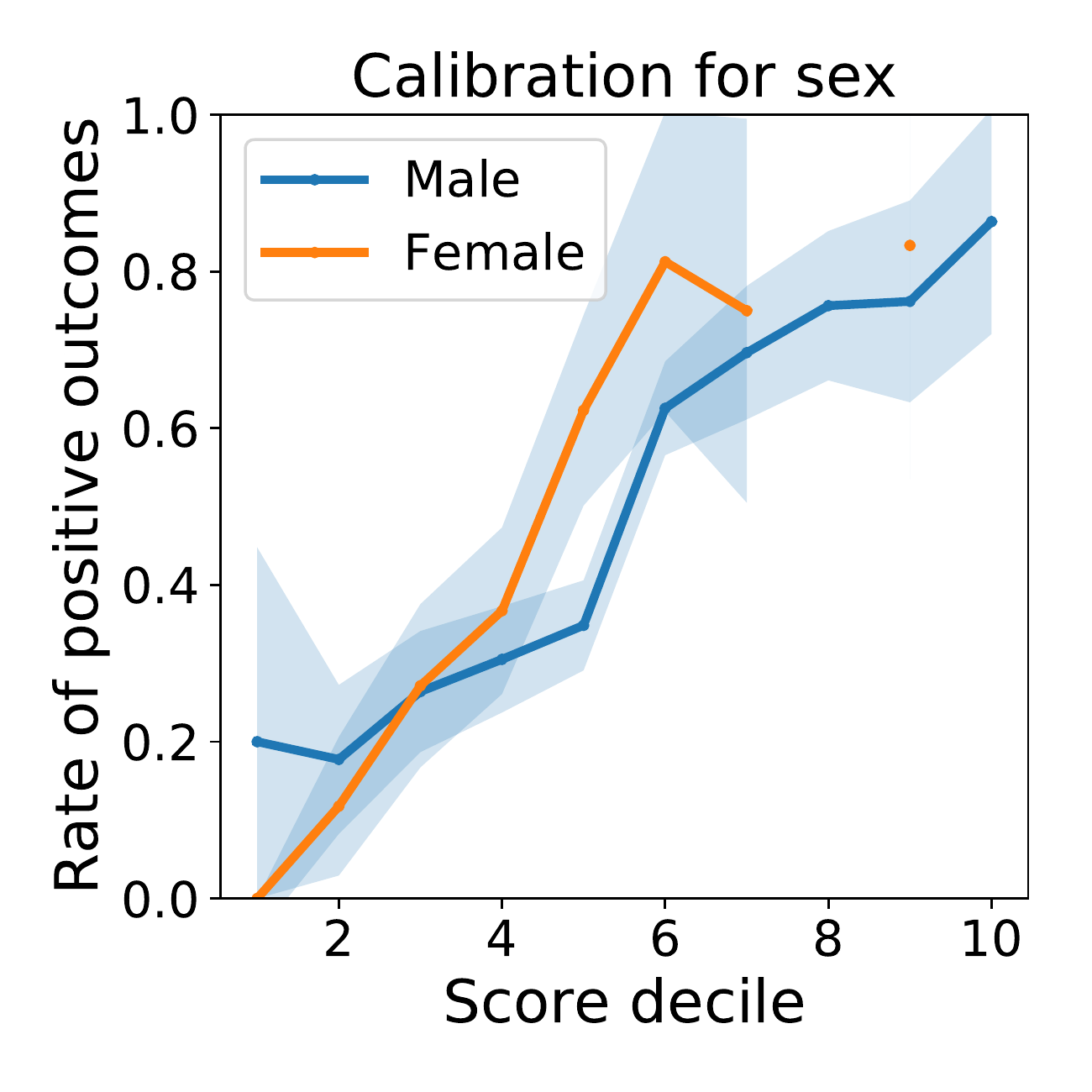}%
		\includegraphics[width=0.33\columnwidth]{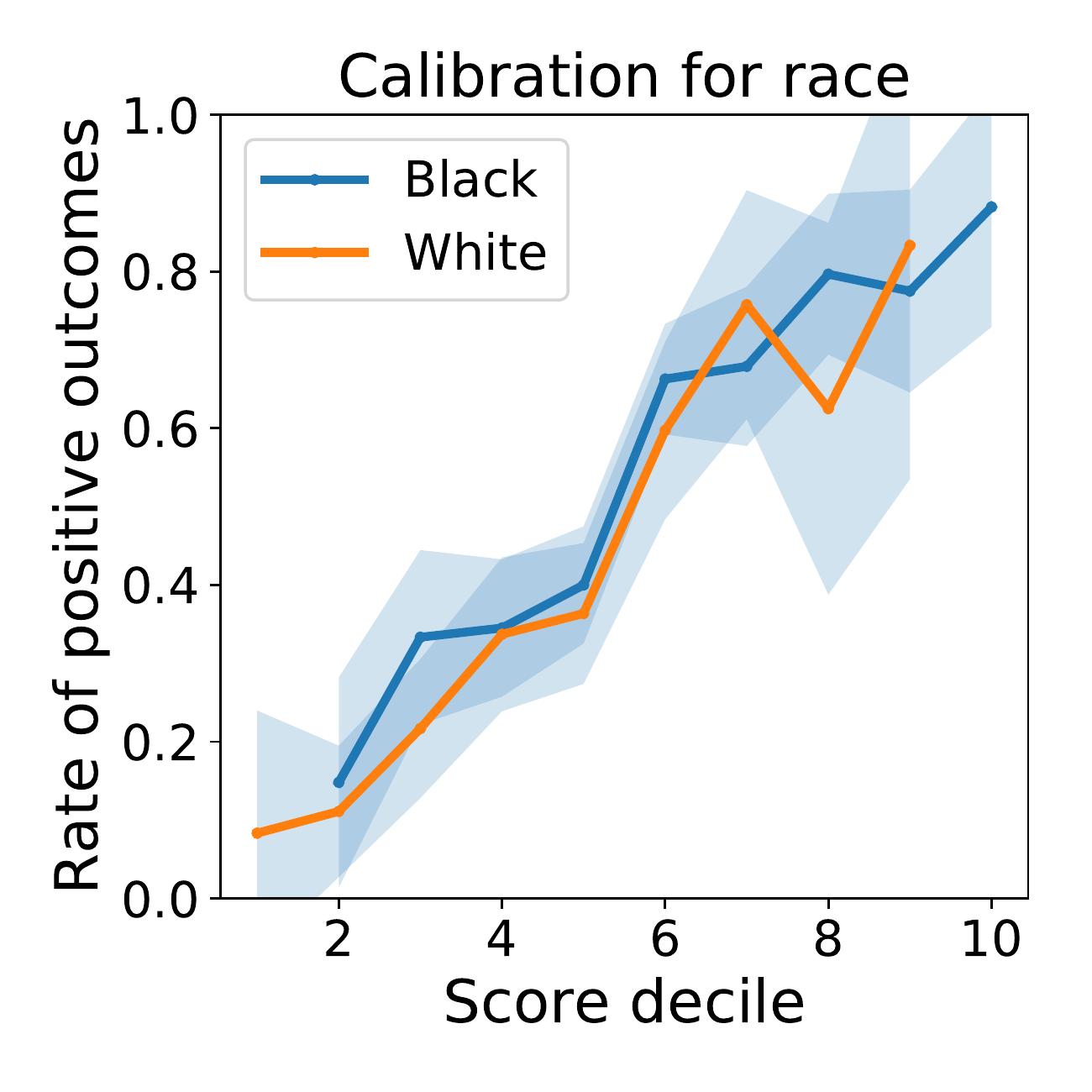}%
	\includegraphics[width=0.33\columnwidth]{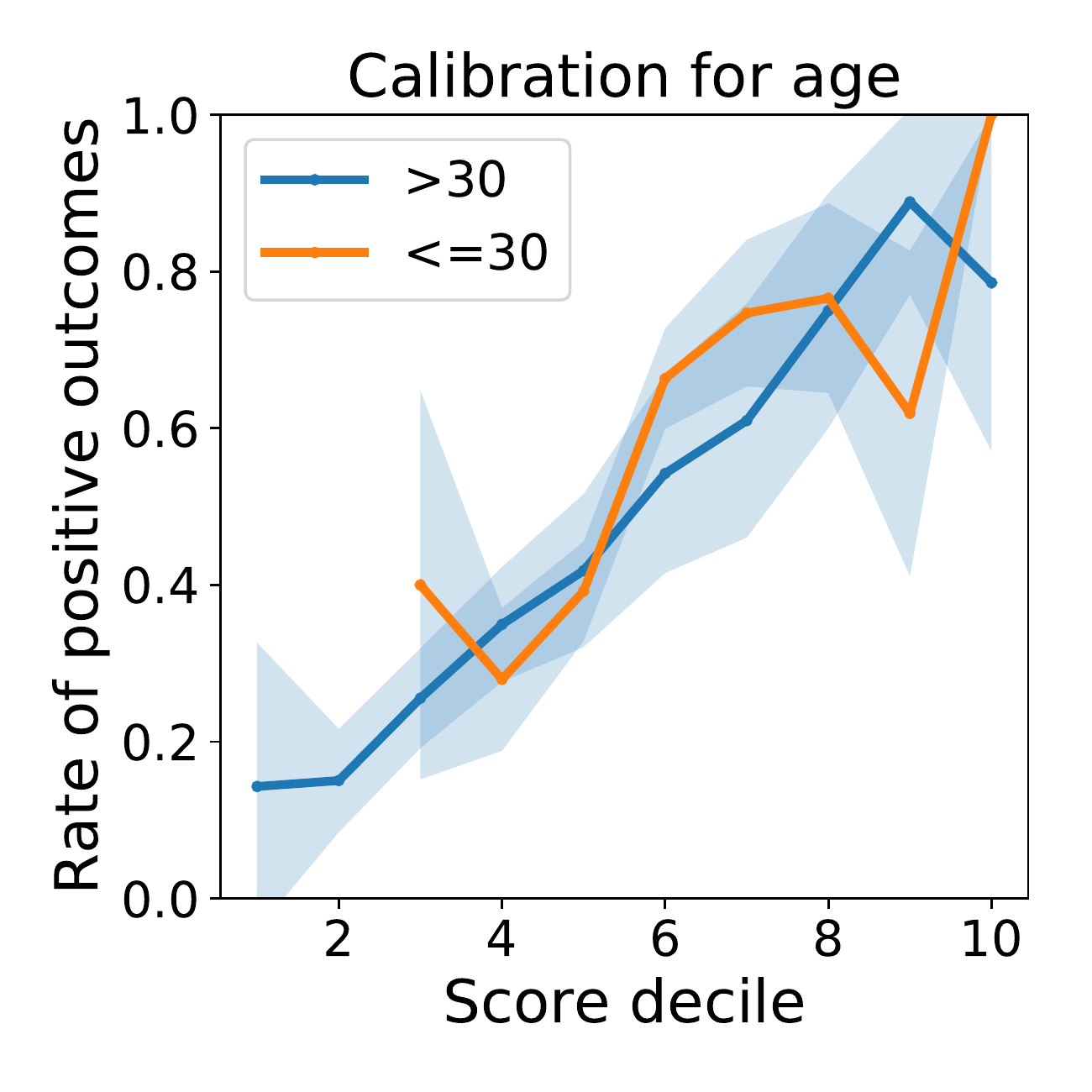}%
	\caption{Calibration plots with respect to group attributes for the Broward dataset. Missing datapoints are where the score decile bucket contains fewer than two individuals. }\label{fig:multi_calib_broward_eqbuckets}
\end{figure}


\begin{figure*}[htbp]
		\includegraphics[width=0.5\textwidth]{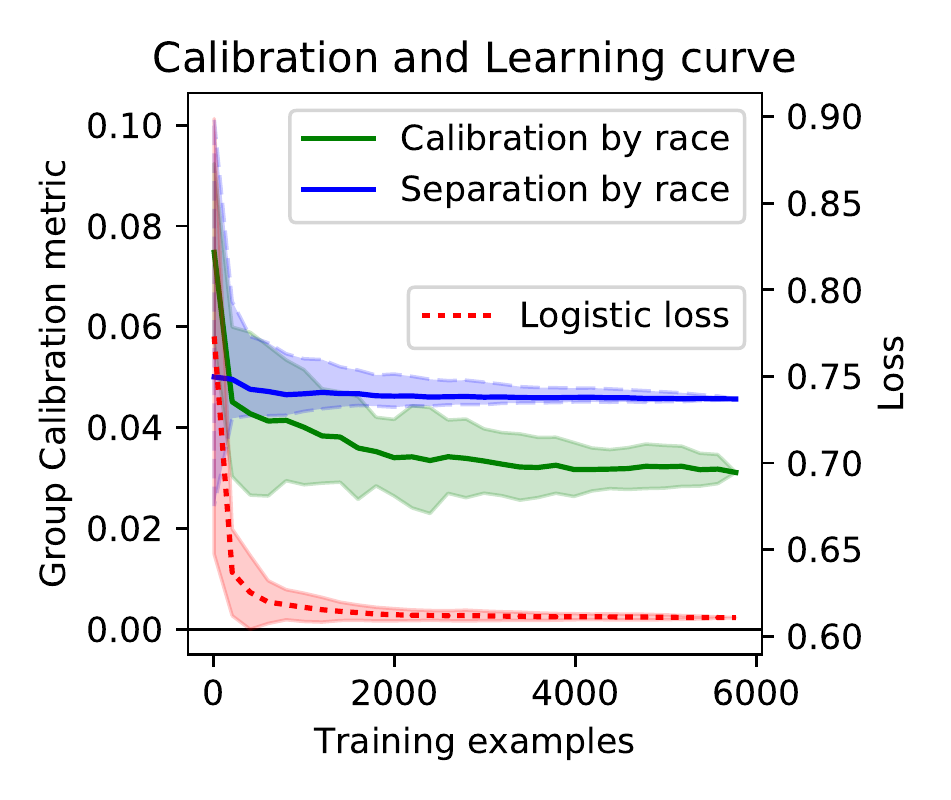}%
		\includegraphics[width=0.5\textwidth]{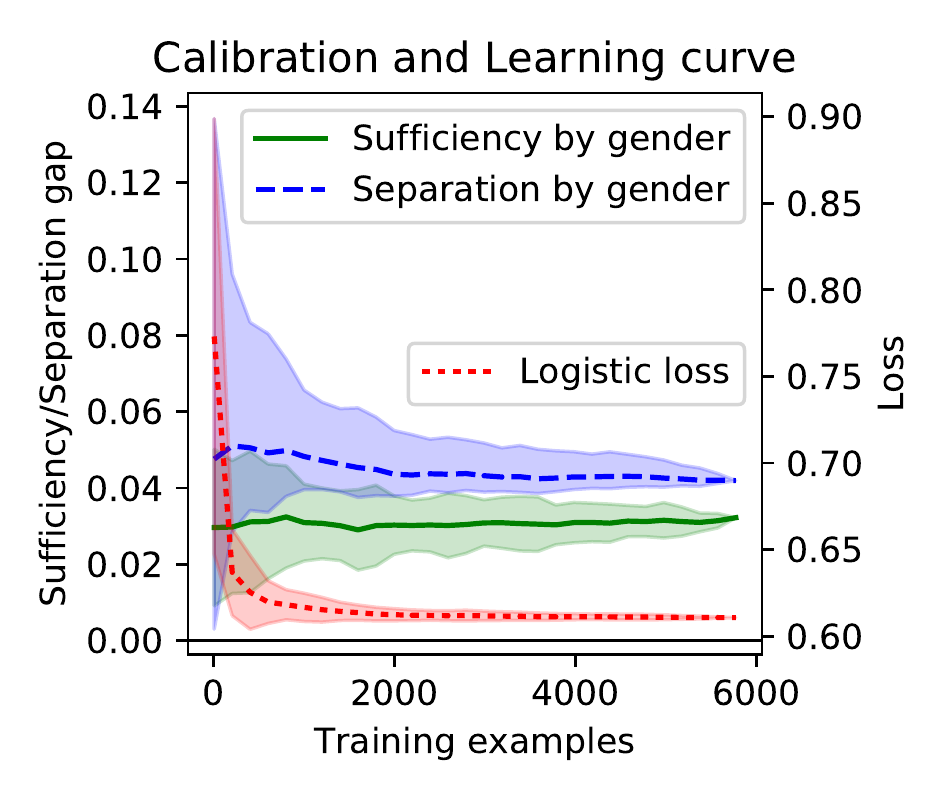}
	\caption{Sufficiency, Separation, and Loss vs. Number of training examples for the Broward dataset}\label{fig:learningcurve_broward}
\end{figure*}

\subsection{Simultaneous calibration with respect to multiple, rich group attributes}\label{app:multicalib}

In this section, we present additional results on the Adult dataset. Specifically, we show the calibration plots for the score obtained from logistic regression, with respect to a variety of group attributes, including `fabricated' group attributes that are a combination of two features. For numerical features (e.g. Age), we split the data into 2 groups according to an arbitrary threshold (e.g. above 40 years old and below 40 years old) and compute calibration with respect to those groups. For categorical features, we only visualize the calibration of the top three most populous groups, for clarity. This is shown in Figure \ref{fig:multi_calib_2}.
Note that by Theorem \ref{thm:main_upper}, groups that have small mass have a worse bound on calibration.

\begin{figure}[htpb]
	\includegraphics[width=0.33\columnwidth]{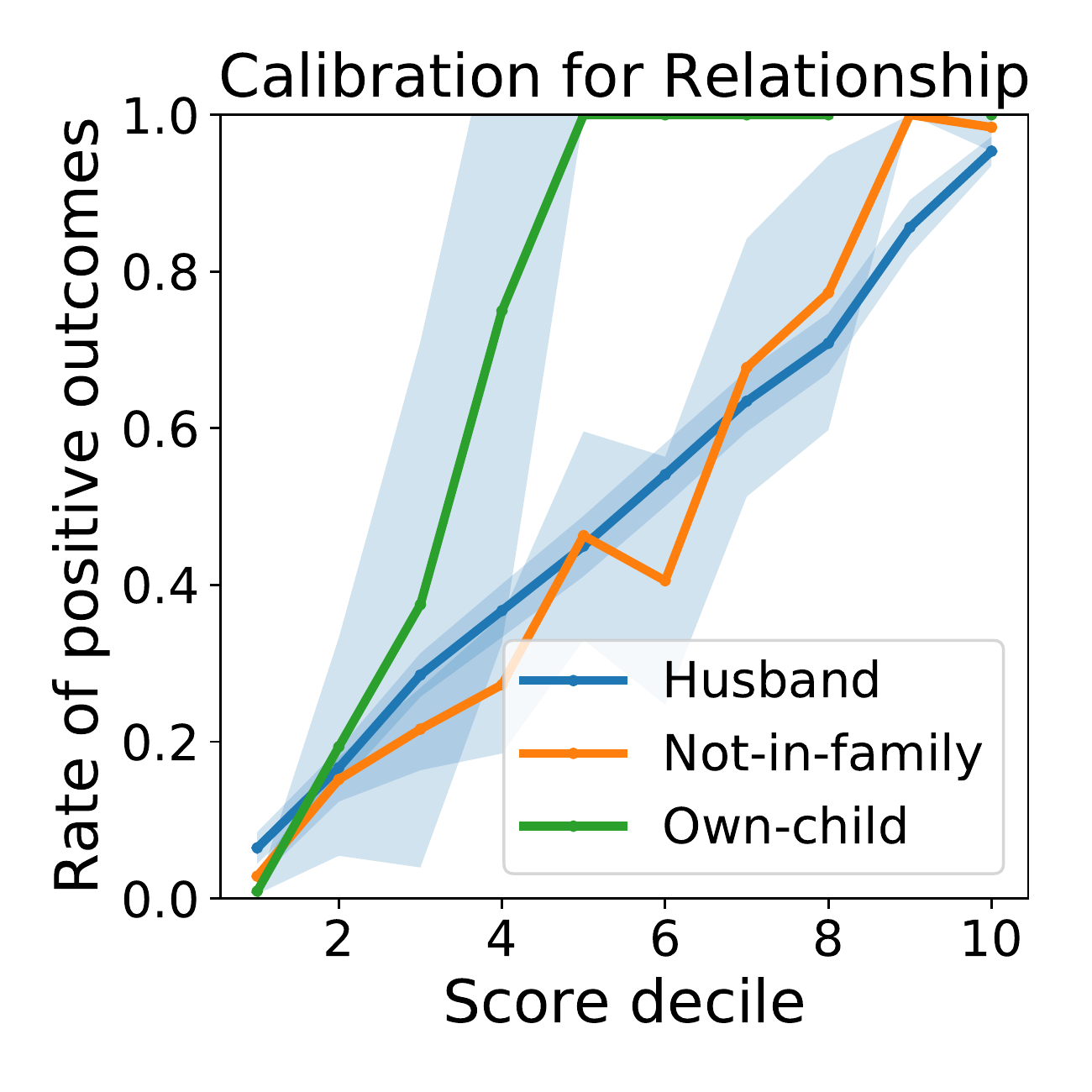}%
	\includegraphics[width=0.33\columnwidth]{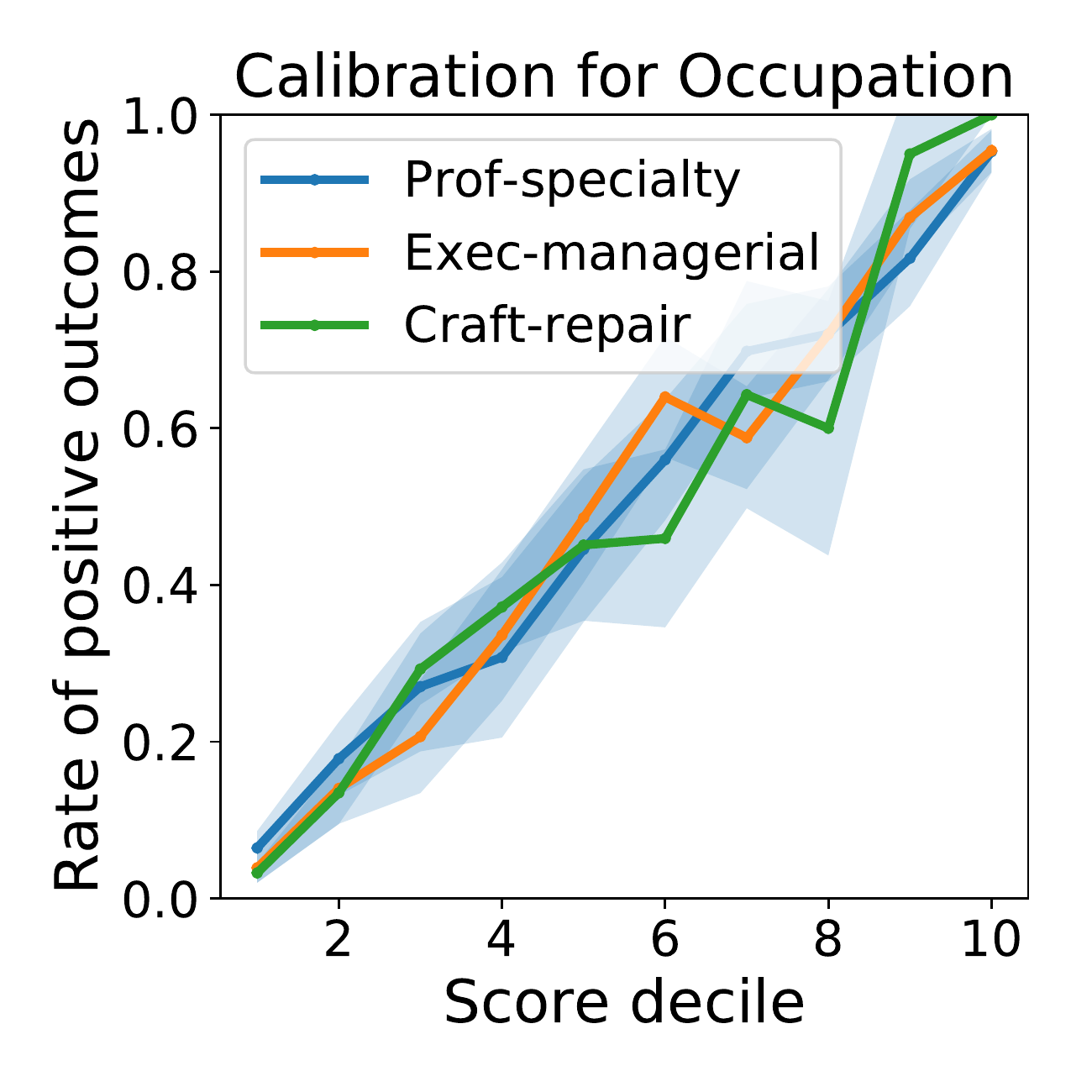}%
		\includegraphics[width=0.33\columnwidth]{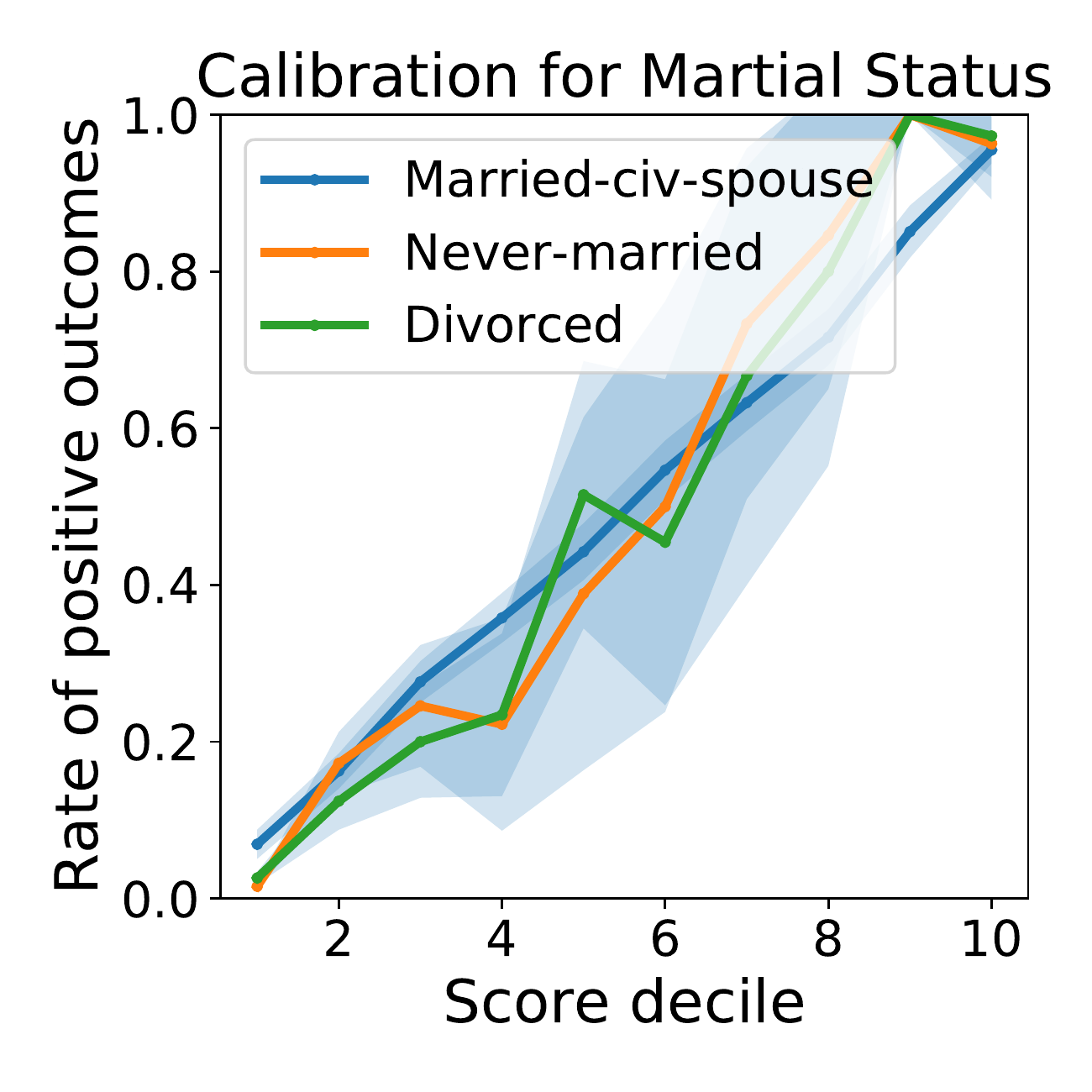}
		\includegraphics[width=0.33\columnwidth]{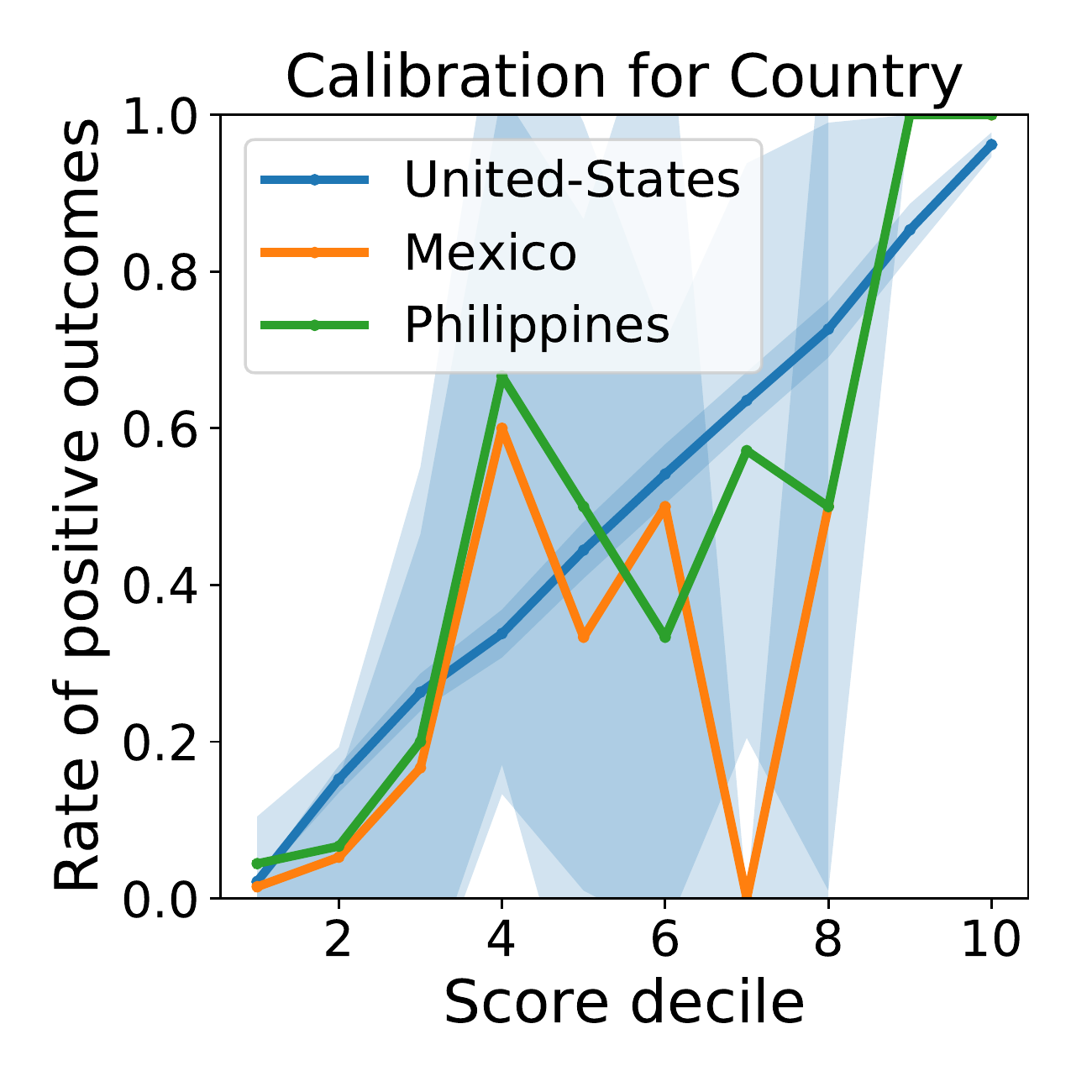}%
		\includegraphics[width=0.33\columnwidth]{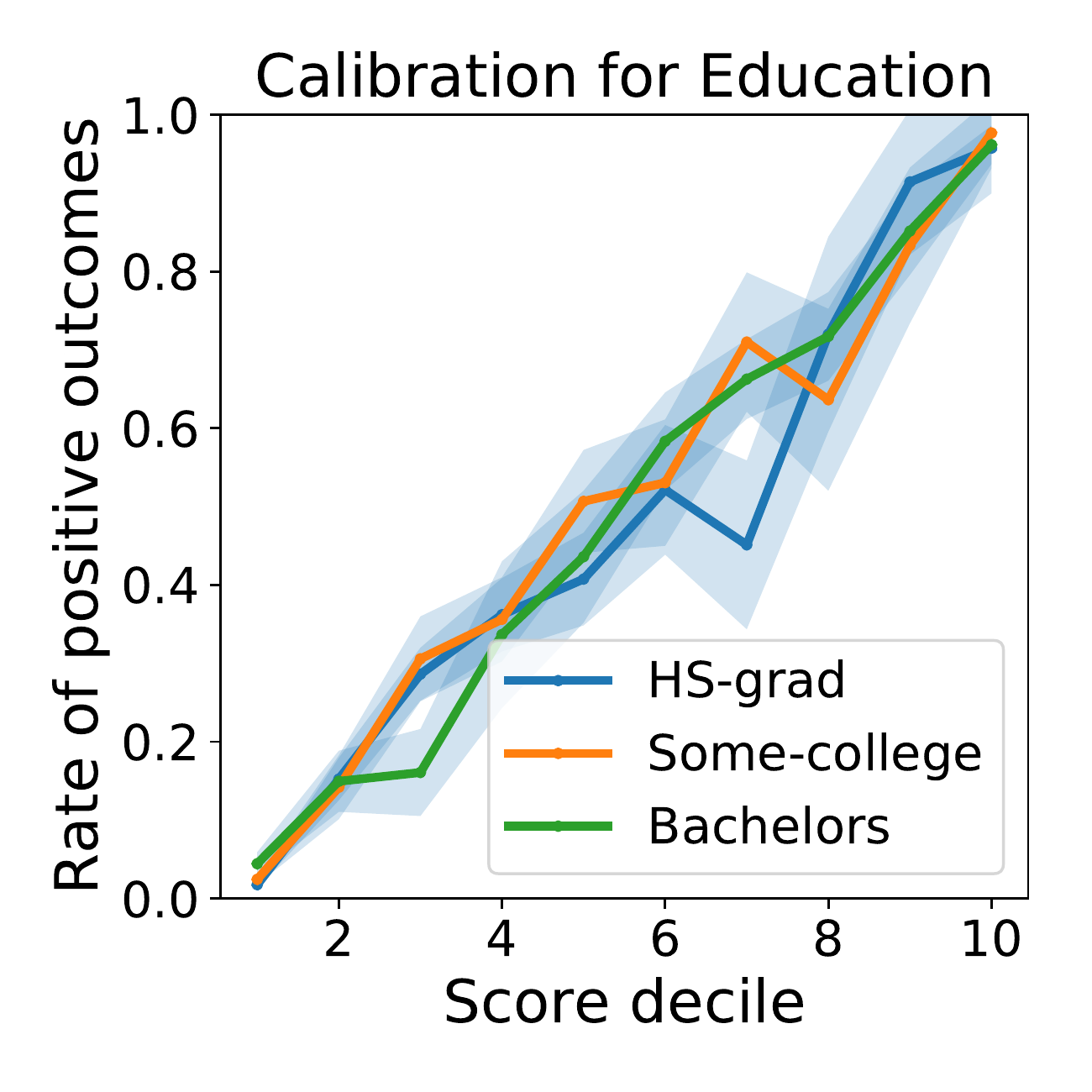}%
		\includegraphics[width=0.33\columnwidth]{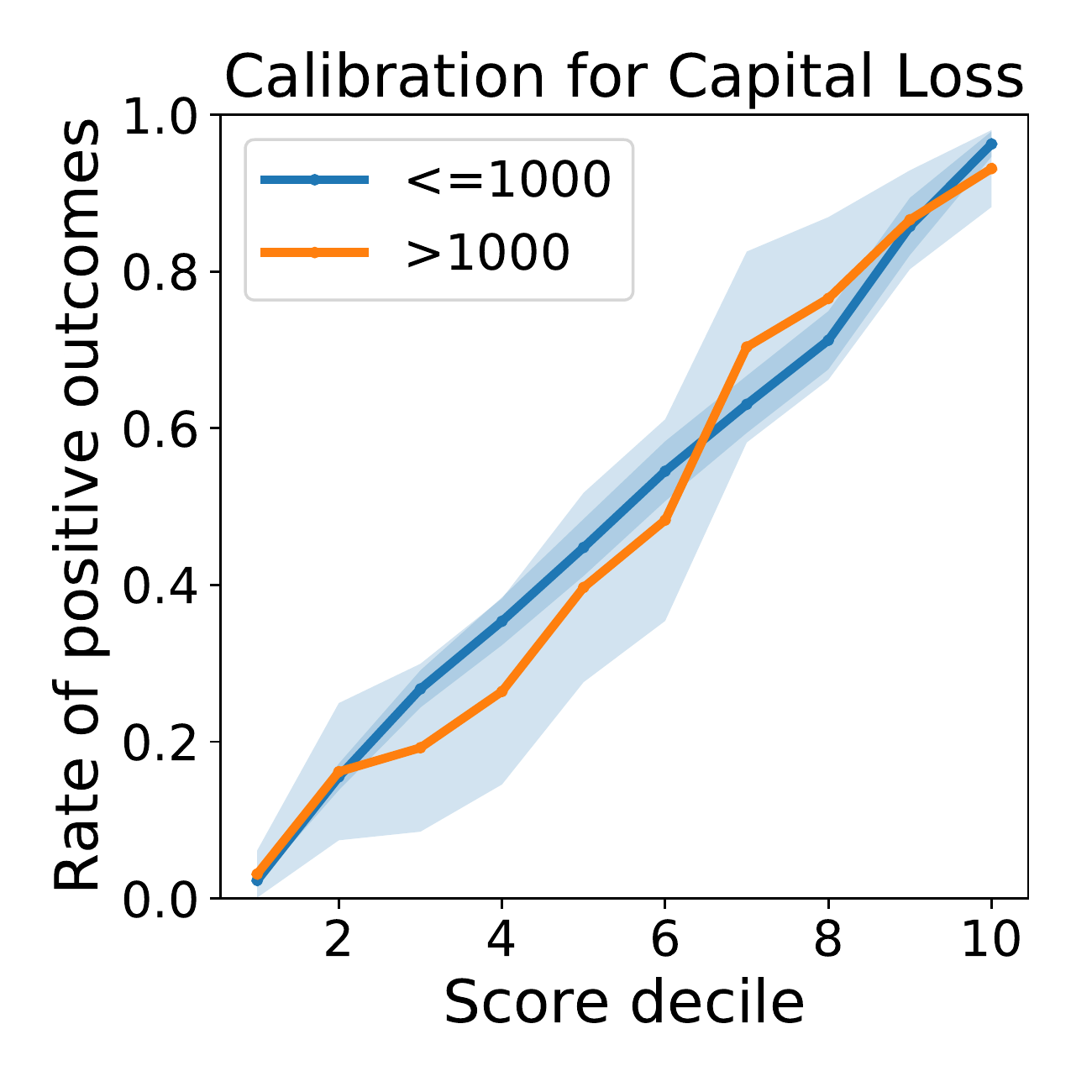}
		\includegraphics[width=0.33\columnwidth]{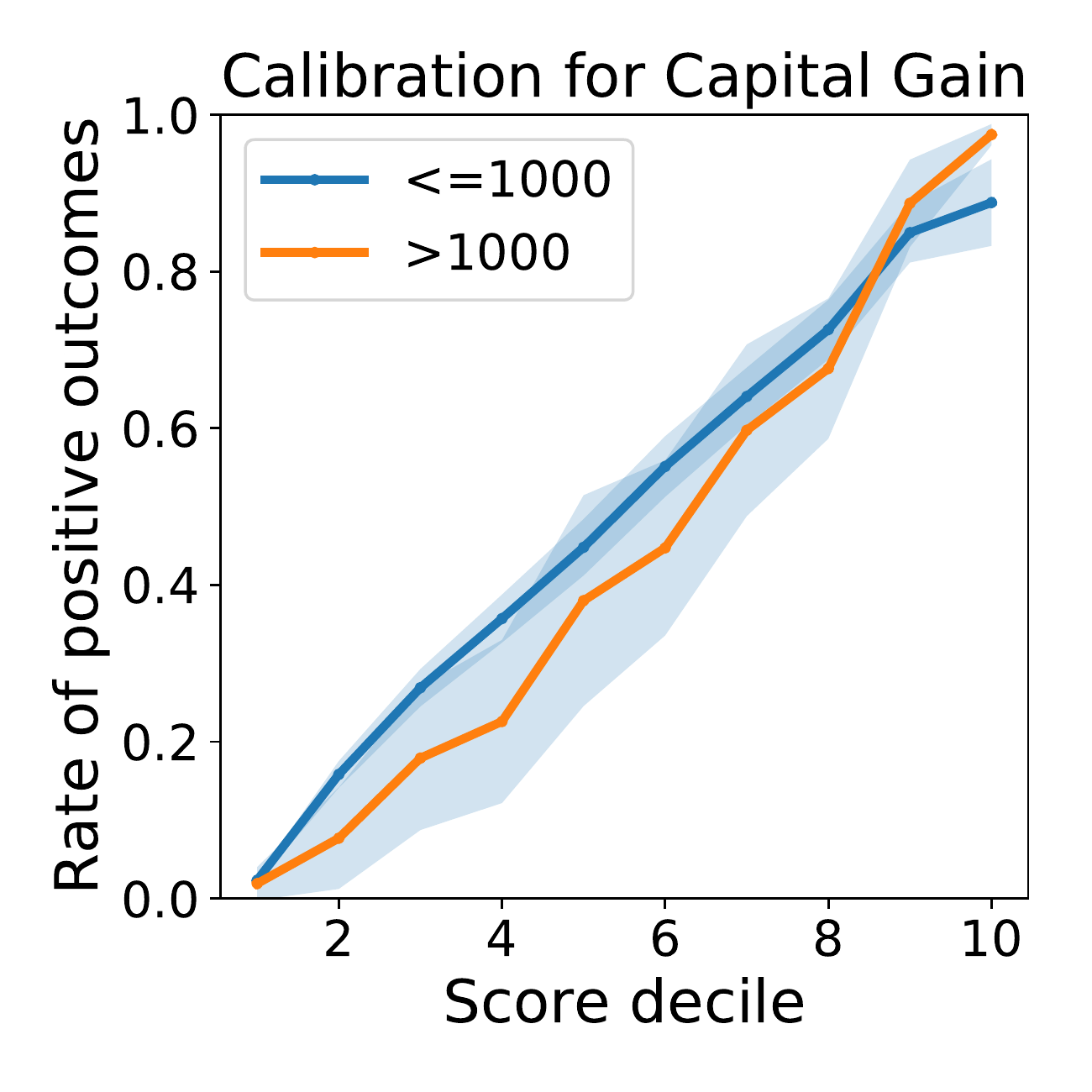}%
		\includegraphics[width=0.33\columnwidth]{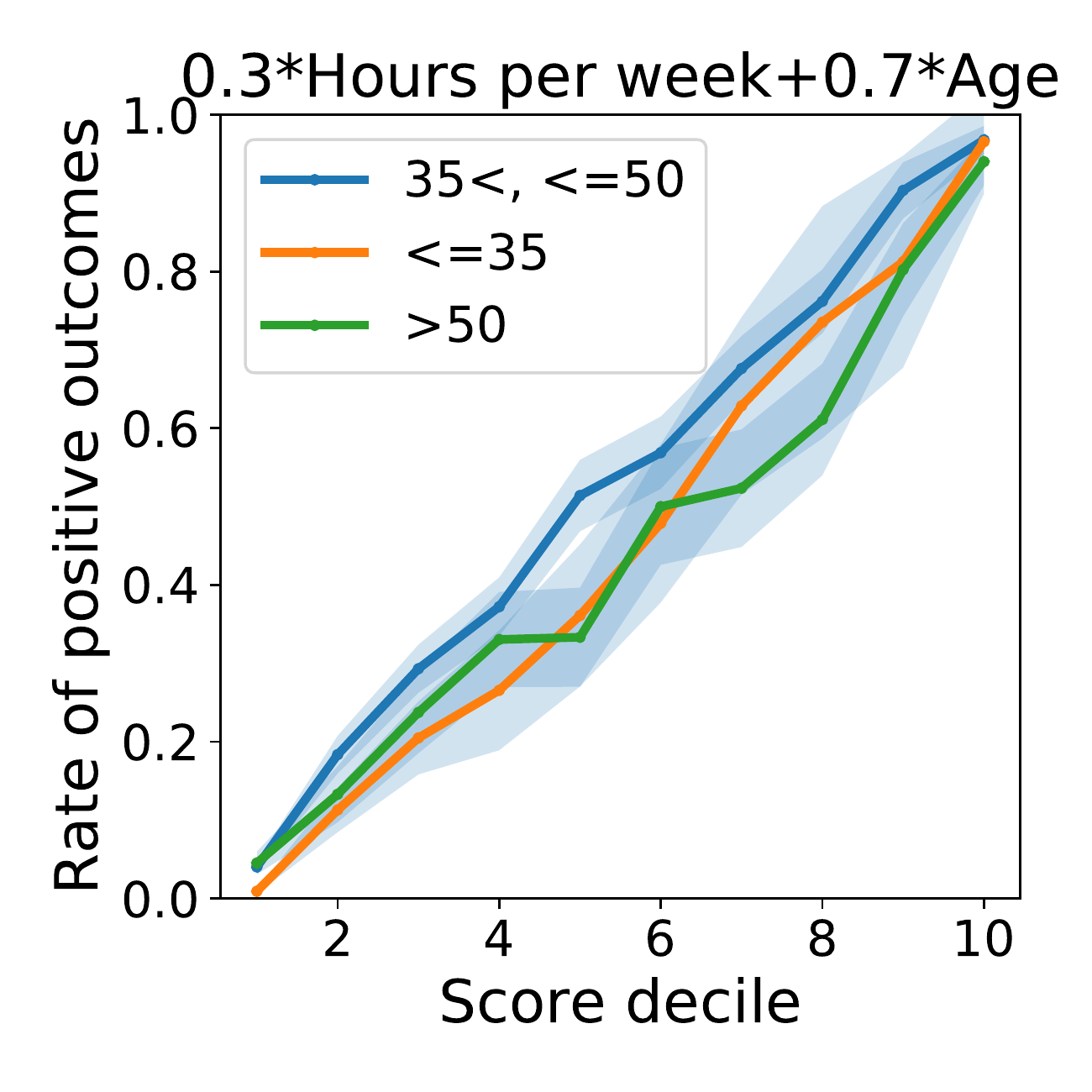}%
		\includegraphics[width=0.33\columnwidth]{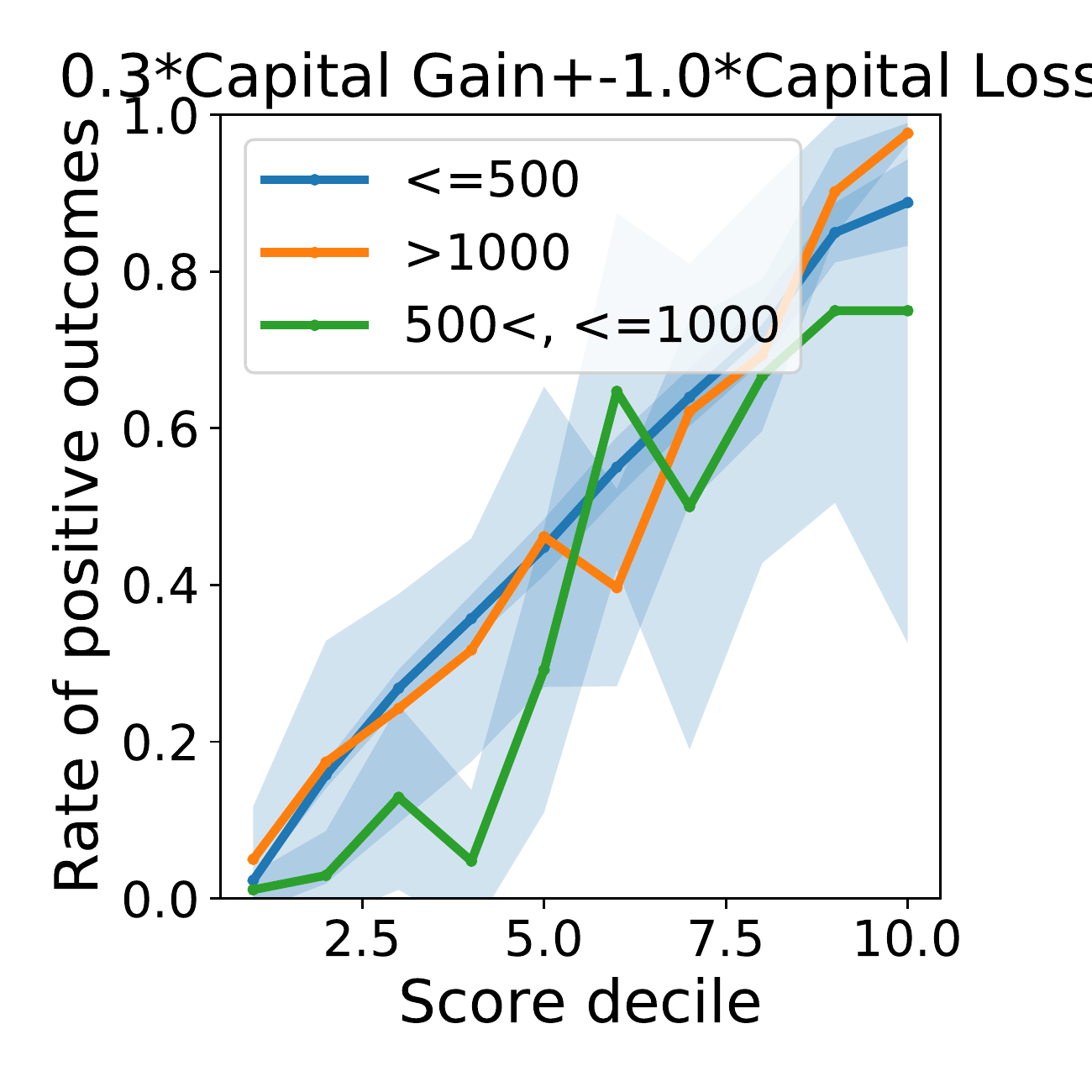}
	\caption{Calibration plots with respect to other features and combinations of features for the Adult dataset}\label{fig:multi_calib_2}
\end{figure}


%
%

%
%


\end{appendix}

\end{document}